%% file: main.tex
\documentclass[twoside]{article}

\usepackage[accepted]{aistats2020}
\usepackage[T1]{fontenc}

\usepackage{pgfplots}
\makeatletter
\def\BState{\State\hskip-\ALG@thistlm}
\makeatother

\usepackage{cite}
\usepackage[utf8]{inputenc}

\usepackage{graphicx}
\usepackage{subcaption}
\usepackage{pstricks}
\usepackage{color}
\usepackage{bbm}
\usepackage{tcolorbox}
\usepackage{tikz}
\usepackage{pgfplots}
\usepackage{wrapfig}
\usepackage{lipsum}  % generates 
\usetikzlibrary{positioning}
\usepackage{tcolorbox}
\usepackage{amssymb}
\usepackage{amsmath,amsfonts,amssymb,amsthm}
\usepackage{natbib}

\usepackage{mathtools}
\usepackage{commath}
\usepackage{relsize}
\usepackage{bbm}
\usepackage{bm}
\usepackage[font={small}]{caption}
\usepackage{comment,color,soul}
\usepackage{enumerate} 
\usetikzlibrary{arrows,automata}
\usepackage{amsfonts}

\usepackage{mathtools}
\usepackage{url}
\usepackage{algorithm}
\usepackage{algpseudocode}
\usepackage{lipsum}
\usepackage[thinlines]{easytable}
 \usepackage{nicefrac}

\makeatletter
\newcommand{\algmargin}{\the\ALG@thistlm}
\makeatother
\newlength{\whilewidth}
\algdef{SE}[parWHILE]{parWhile}{EndparWhile}[1]
  {\parbox[t]{\dimexpr\linewidth-\algmargin}{%
     \hangindent\whilewidth\strut\algorithmicwhile\ #1\ \algorithmicdo\strut}}{\algorithmicend\ \algorithmicwhile}%
\algnewcommand{\parState}[1]{\State%
  \parbox[t]{\dimexpr\linewidth-\algmargin}{\strut #1\strut}}
  
\usepackage{mysymbol}
\makeatletter
\def\thm@space@setup{\thm@preskip=2pt
        \thm@postskip=2pt \itshape}
        \def\Let@{\def\\{\notag\math@cr}}
\makeatother

\newtheoremstyle{newstyle}
{} %Aboveskip
{} %Below skip
{\mdseries} %Body font e.g.\mdseries,\bfseries,\scshape,\itshape
{} %Indent
{\bfseries} %Head font e.g.\bfseries,\scshape,\itshape
{.} %Punctuation afer theorem header
{ } %Space after theorem header
{} %Heading

\newtheorem{example}{Example} 
\newtheorem{theorem}{Theorem}
\newtheorem{lemma}{Lemma} 
 
\newtheorem{remark}{Remark}

\newtheorem{assumption}{Assumption}

  

\newcommand{\gr}{\nabla}

\newcommand{\bx}{\mathbf{x}}

\newcommand{\bE}{\mathbb{E}}

\newcommand{\xbar}{\overline{\mathbf{x}}}
\newcommand{\zbar}{\overline{\mathbf{z}}}

\newcommand{\tNab}{\widetilde{\nabla}}

\usepackage[hidelinks]{hyperref}
\definecolor{darkred}{RGB}{150,0,0}
\definecolor{darkgreen}{RGB}{0,150,0}
\definecolor{darkblue}{RGB}{0,0,150}
\hypersetup{colorlinks=true, linkcolor=red, citecolor=blue, urlcolor=darkblue}
\begin{document}

\runningtitle{\texttt{FedPAQ}: A Communication-Efficient Federated Learning Method}

\runningauthor{Reisizadeh, Mokhtari, Hassani, Jadbabaie, Pedarsani}

\twocolumn[

\aistatstitle{\texttt{FedPAQ}: A Communication-Efficient Federated Learning Method with Periodic Averaging and Quantization}

\vspace{-4mm}

%\aistatsauthor{Amirhossein Reisizadeh\thanks{Dept. of Electrical $\&$ Computer Engineering, University of California, Santa Barbara, \{reisizadeh@ucsb.edu\}.} , Aryan Mokhtari\thanks{Dept. of Electrical $\&$ Computer Engineering, The University of Texas at Austin, \{mokhtari@austin.utexas.edu\}.} , Hamed Hassani\thanks{Dept. of Electrical $\&$ Systems Engineering, University of Pennsylvania, \{hassani@seas.upenn.edu\}.} ,\vspace{2mm}\\
%Ali Jadbabaie\thanks{Laboratory for Information $\&$ Decision Systems, Massachusetts Institute of Technology, \{jadbabai@mit.edu\}.} , Ramtin Pedarsani\thanks{Dept. of Electrical $\&$ Computer Engineering, University of California, Santa Barbara, \{ramtin@ece.ucsb.edu\}.}}

\aistatsauthor{Amirhossein Reisizadeh \And Aryan Mokhtari \And Hamed Hassani}
\aistatsaddress{UC Santa Barbara \And UT Austin \And UPenn}
\vspace{-4mm}
\aistatsauthor{Ali Jadbabaie \And Ramtin Pedarsani}
\aistatsaddress{MIT \And UC Santa Barbara}
]

\begin{abstract}
\input{0-abstract.tex}
\end{abstract}
\vspace{-2mm}

\section{Introduction}\label{sec:intro}
\input{1-intro.tex}

\vspace{-2mm}
\section{Federated Learning Setup}\label{sec:setup}
\input{1-setup.tex}

\vspace{-1mm}
\section{Proposed \texttt{FedPAQ} Method}\label{sec:qfl}
\input{2-method.tex}

\vspace{-1mm}
\section{Convergence Analysis}\label{sec:convg}
\input{3-convergence.tex}

\vspace{-1mm}
\section{Numerical Results and Discussions}\label{sec:numerical}
\input{7-numerical.tex}

\vspace{-3mm}
\section{Conclusion}\label{sec:conclusion}
\input{6-conclusion.tex}

%\newpage

%{\small{
%\newpage
\bibliography{ref}
\bibliographystyle{plainnat}
%
%\bibliographystyle{apalike}
%\bibliographystyle{abbrvnat}
%bibliographystyle{unsrt}
%\bibliographystyle{neurips_2019}
%\bibliographystyle{ieeetr}
%}}

\newpage
\onecolumn
\begin{center}
    \large{\textbf{Supplementary Materials}}
\end{center}

Here, we provide the proofs of the main two theorems of this paper in Sections \ref{sec:thm1-proof} and \ref{sec:thm2-proof} along with the necessary lemmas and discussions. Moreover, we provide more numerical results over more complicated datasets and model parameters in Section \ref{sec:additional-numerical}.

\section{Proof of Theorem \ref{thm:1}}\label{sec:thm1-proof}
\input{5-proofs.tex}

\section{Proof of Theorem \ref{thm:2}}\label{sec:thm2-proof}
\input{6-nonconvex.tex}

\section{Additional Numerical Results}\label{sec:additional-numerical}
\input{8-additional-numerical.tex}

\end{document}

%% file: 0-abstract.tex
\vspace{-2mm}
Federated learning is a distributed framework according to which  a model is trained over a set of devices, while keeping data localized. This framework  faces several systems-oriented challenges which include (i) \emph{communication bottleneck} since a large number of devices upload their local updates to a parameter server, and (ii) \emph{scalability} as the federated network consists of millions of devices. Due to these systems challenges as well as issues related to statistical heterogeneity of data and privacy concerns, designing a provably efficient federated learning method is of significant importance yet it remains challenging. In this paper, we present \texttt{FedPAQ}, a communication-efficient \textbf{Fed}erated Learning method with \textbf{P}eriodic \textbf{A}veraging and \textbf{Q}uantization. \texttt{FedPAQ} relies on three key features: (1) periodic averaging where models are updated locally at devices and only periodically averaged at the server; (2) partial device participation where only a fraction of devices participate in each round of the training; and (3) quantized message-passing where the edge nodes quantize their updates before uploading to the parameter server. These features address the communications and scalability challenges in federated learning. We also show that \texttt{FedPAQ} achieves near-optimal theoretical guarantees for strongly convex and non-convex loss functions and empirically demonstrate the communication-computation tradeoff provided by our method.
\vspace{-2mm}

%% file: 1-intro.tex
In many large-scale machine learning applications, data is acquired and processed at the edge nodes of the network such as mobile devices, users' devices, and IoT sensors. \emph{Federated Learning} is a novel paradigm that aims to train a statistical model at the ``edge'' nodes as opposed to the traditional distributed computing systems such as data centers \citep{konevcny2016federated,li2019federated}. The main objective of federated learning is to fit a model to data generated from network devices without continuous transfer of the massive amount of collected data from edge of the network to back-end servers for processing.

Federated learning has been deployed by major technology companies with the goal of providing privacy-preserving services using users' data \citep{bonawitz2019towards}. Examples of such applications are learning from wearable devices \citep{huang2018loadaboost}, learning sentiment \citep{smith2017federated}, and location-based services \citep{samarakoon2018distributed}. While federated learning is a promising paradigm for such applications, there are several challenges that remain to be resolved. In this paper, we focus on two significant challenges of federated learning, and propose a novel federated learning algorithm that addresses the following  two challenges:\vspace{2mm}\\
(1) \textbf{Communication bottleneck.} Communication bandwidth is a major bottleneck in federated learning as a large number of devices attempt to communicate their local updates to a central parameter server. Thus, for a communication-efficient federated learning algorithm, it is crucial that such updates are sent in a compressed manner and infrequently.\vspace{2mm}\\
(2) \textbf{Scale.} A federated network typically consists of thousands to millions of devices that may be active, slow, or completely inactive during the training procedure. Thus, a proposed federated learning algorithm should be able to operate efficiently with partial device participation or random sampling of devices. 
% \item[(iii)] \textbf{Optimality guarantees.} Given the mentioned systems bottlenecks, achieving provable accuracy guarantees in federated learning is significantly challenging. Establishing such guarantees becomes even harder when the users' data is statistically heterogeneous. 

The goal of this paper is to develop a provably efficient federated learning algorithm that addresses the above-mentioned systems challenges. More precisely, we consider the task of training a model in a federated learning setup where we aim to find an \emph{accurate} model over a collection of $n$ distributed nodes. In this setting, each node contains $m$ independent and identically distributed samples from an unknown probability distribution and a parameter server helps coordination between the nodes. We focus on solving the empirical risk minimization problem for a federated architecture while addressing the challenges mentioned above. In particular, we consider both strongly convex and non-convex settings and provide sharp guarantees on the performance of our proposed algorithm.

\vspace{2mm}
\noindent\textbf{Contributions.} In this work, we propose {\texttt{FedPAQ}}, a communication-efficient \textbf{Fed}erated learning algorithm with \textbf{P}eriodic \textbf{A}veraging and \textbf{Q}uantization, which addresses federated learning systems' bottlenecks. In particular, \texttt{FedPAQ} has three key features that enable efficient federated learning implementation:\vspace{2mm}\\
%\begin{enumerate}[(1)]
(1) \texttt{FedPAQ} allows the nodes (users) of the network to run local training before synchronizing with the parameter server. In particular, each node iteratively updates its local model for a period of iterations using the stochastic gradient descent (SGD) method and then uploads its model to the parameter server where all the received models are averaged periodically. By tuning the parameter which corresponds to the number of local iterations before communicating to the server, periodic averaging results in slashing the number of communication rounds and hence the total communication cost of the training process.\vspace{2mm}\\
(2) \texttt{FedPAQ} captures the constraint on availability of active edge nodes by allowing a partial node participation. That is, in each round of the  method, only a fraction of the total devices--which are the active ones--contribute to train the model. This procedure not only addresses the scalability challenge, but also leads to smaller communication load compared to the case that all nodes participate in training the learning model.\vspace{2mm}\\
(3) In \texttt{FedPAQ}, nodes only send a quantized version of their local information to the server at each round of communication. As the training models are of large sizes, quantization significantly helps reducing the communication overhead on the network.

%\end{enumerate}
While these features have been proposed in the literature, to the best of our knowledge, \texttt{FedPAQ} is the first federated learning algorithm that simultaneously incorporates these features and provides near-optimal theoretical guarantees on its statistical accuracy, while being communication-efficient via periodic averaging, partial node participation and quantization. 

In particular, we analyze our proposed \texttt{FedPAQ} method for two general class of loss functions: strongly-convex and non-convex. For the strongly-convex setting, we show that after $T$ iterations the squared norm of the distance between the solution of our method and the optimal solution is of  $\ccalO(1/T)$ in expectation. We also show that \texttt{FedPAQ} approaches a first-order stationary point for non-convex losses at a rate of $\ccalO(1/\sqrt{T})$. This demonstrates that our method significantly improves the communication-efficiency of federated learning while preserving the optimality and convergence guarantees of the baseline methods. In addition, we would like to highlight that our theoretical analysis is based on few relaxed and customary assumptions which yield more technical challenges compared to the existing works with stronger assumptions and hence acquires novel analytical techniques. More explanations will be provided in Section \ref{sec:convg}.

\vspace{2mm}
\noindent
\textbf{Related Work.} The main premise of  federated learning has been collective learning using a network of common devices such as phones and tablets. This framework  potentially allows for smarter models, lower latency, and less power consumption, all while ensuring privacy. Successfully achieving these goals in practice requires addressing key challenges of federated learning such as communication complexity, systems heterogeneity, privacy, robustness, and heterogeneity of the users.  Recently, many federated methods have been considered in the literature which mostly aim at reducing the communication cost.  \cite{mcmahan2016communication} proposed the \texttt{FedAvg} algorithm, where the global model is updated by averaging local SGD updates. \cite{guha2019one} proposed one-shot federated learning in which the master node learns the model after a single round of communication. 

Optimization methods for federated learning are naturally tied with tools from stochastic  and distributed optimization. Minibatch stochastic gradient descent distributed optimization methods have been largely studied in the literature without considering the communication bottleneck. Addressing the communication bottleneck via quantization and compression in distributed learning has recently gained considerable attention for both master-worker \citep{alistarh2017qsgd,seide20141,bernstein2018signsgd,smith2016cocoa} and masterless topologies \citep{reisizadeh2019exact,zhang2018compressed,koloskova2019decentralized,wang2019matcha}. Moreover, \cite{wang2019matcha} reduces the communication delay by decomposing the graph.

Local updates, as another approach to reduce the communication load in distributed learning has been studied in the literature, where each learning node carries out multiple local updates before sharing with the master or its neighboring nodes. \cite{stich2018local} considered a master-worker topology and provides theoretical analysis for the convergence of local-SGD method. \cite{lin2018don} introduced a variant of local-SGD namely post-local-SGD which demonstrates empirical improvements over local-SGD. \cite{wang2018cooperative} provided a general analysis of such cooperative method for decentralized settings as well.

Statistical heterogeneity of users' data points is another major challenge in federated learning. To address this heterogeneity, other methods such as multitask learning and meta learning have been proposed to train multiple local models \citep{smith2017federated,nichol2018first,li2019convergence}. Many methods have been proposed to address systems heterogeneity and in particular stragglers in distributed learning using coding theory, e.g., \citep{lee2018speeding,yu2017polynomial,dutta2016short,tandon2016gradient,reisizadeh2019codedreduce}. 
Another important challenge in federated learning is to preserve privacy in learning \citep{duchi2014privacy}. \cite{mcmahan2017learning,agarwal2018cpsgd} proposed privacy-preserving methods for distributed and federated learning using differential privacy techniques. Federated heavy hitters discovery with differential privacy was proposed in \citep{zhu2019federated}. %We note that privacy is not the focus of our paper, as FL is somewhat inherently private. 

Robustness against adversarial devices is another challenge in federated learning and distributed learning that has been studied in \citep{chen2017distributed,yin2018byzantine,ghosh2019robust}. Finally, several works have considered communication-efficient collaborative learning where there is no master node, and the computing nodes learn a model collaboratively in a decentralized manner \citep{reisizadeh2019exact,zhang2018compressed,doan2018accelerating,koloskova2019decentralized,lalitha2019peer}.  While such techniques are related to federated learning, the network topology in master-less collaborative learning is fundamentally different.

%% file: 1-setup.tex
\vspace{-1mm}
In this paper, we focus on a federated architecture where a parameter server (or server) aims at finding a model that performs well with respect to the data points that are available at different nodes (users) of the network, while nodes exchange their local information with the server. We further assume that the data points for all nodes in the network are generated from a common probability distribution. In particular, we consider the following stochastic learning problem
%%%$nm$ realizations in $\ccalD = \ccalD^1 \cup \cdots \cup \ccalD^n$.
\vspace{-1mm}
\begin{equation}\label{eq:ERM}
  \min_{\bbx}  f(\bbx) \coloneqq \min_{\bbx} \frac{1}{n}\sum_{i = 1}^n f_i(\bbx), \vspace{-1mm}
\end{equation}
where the local objective function of each node $i$ is defined as the expected loss of its local sample distributions
%%%
\vspace{-1mm}
\begin{equation}
    f_i(\bbx) \coloneqq \mathbb{E}_{\xi \sim \ccalP^i} \, [\ell(\bx,\xi)].\vspace{-1mm}
\end{equation}
%%%
%%%
%\begin{equation}\label{population_risk}
%  \min_{\bbx}  L(\bx) \coloneqq \min_{\bbx} \mathbb{E}_{\xi \sim \ccalP} \, [\ell(\bx,\xi)],
%\end{equation}
%%%
Here $\ell:\mathbb{R}^p\times\mathbb{R}^u  \to\mathbb{R}$ is a stochastic loss function, $\bbx\in\mathbb{R}^p$ is the model vector, and  $\xi\in \mathbb{R}^u$ is a random variable with unknown probability distribution $\ccalP^i$. Moreover, $ f:\mathbb{R}^p \to\mathbb{R}$ denotes the expected loss function also called population risk. 
In our considered federated setting, each of the $n$ distributed nodes generates a local loss function according to a distribution $\ccalP^i$ resulting in a local stochastic function $f_i(\bbx) \coloneqq \mathbb{E}_{\xi \sim \ccalP^i} \, [\ell(\bx,\xi)]$. A special case of this formulation is when each node $i$ maintains a collection of $m$ samples from distribution $\ccalP^i$ which we denote by $\ccalD^i = \{\xi^i_1, \cdots, \xi^i_m\}$ for $i \in [n]$. This results in the following empirical risk minimization problem over the collection of $nm$ samples in $\ccalD \coloneqq \ccalD^1 \cup \cdots \cup \ccalD^n$:
%We focus on solving the expected risk minimization problem over the collection of $n$ nodes.This problem can be written as  
%Note that the formulation in \eqref{eq:ERM} can be also interpreted as minimization of the average of the local functions $f_i$ for all nodes $i=1,\dots, n$ in the network, i.e., 
%%%
\vspace{-1mm}
\begin{equation} \label{eq:erm}
    \min_{\bbx} L(\bbx) = \min_{\bbx} \frac{1}{nm}\sum_{\xi \in \ccalD} \ell(\bbx, \xi),\vspace{-1mm}
\end{equation}
%%%

We denote the optimal model $\bbx^*$ as the solution to the expected risk minimization problem in \eqref{eq:ERM} and denote the minimum loss $f^* \coloneqq \min_{\bbx} f(\bbx) = f(\bbx^*)$ as the optimal objective function value of the expected risk minimization problem in \eqref{eq:ERM}. In this work, we focus on the case that the data over the $n$ nodes is independent and identically distributed (i.i.d.), which implies the local distributions are common.

As stated above, our goal is to minimize the expected loss $f(\bx)$. However, due to the fact that we do not have access to the underlying distribution $\ccalP$, there have been prior works that focus on minimizing the empirical risk $L(\bx)$ which can be viewed as an approximation of the expected loss $f(\bx)$. The accuracy of this approximation is determined by the number of samples $N=nm$. It has been shown that for convex losses $\ell$, the population risk $f$ is at most $\ccalO(1/\sqrt{nm})$ distant from the empirical risk $L$, uniformly and with high probability \citep{bottou2008tradeoffs}. That is, $\sup_{\bx} |f(\bx) - L(\bx)| \leq \ccalO(1/\sqrt{nm})$ with high probability. This result implies that if each of the $n$ nodes separately minimizes its local empirical loss function, the expected deviation from the local solution and the solution to the population risk minimization problem is of $\ccalO(1/\sqrt{m})$ (note that each node has access to $m$ data samples). However, if the nodes manage to somehow share or synchronize their solutions, then a more accurate solution can be achieved, that is a solution with accuracy of order $\ccalO(1/\sqrt{nm})$. Therefore, when all the $mn$ available samples are leveraged, one can obtain a solution $\hat{\bx}$ that satisfies $\bE [ L(\hat{\bx}) - L(\bx^*) ] \leq \ccalO(1/\sqrt{nm})$. This also implies that $\bE [ f(\hat{\bx}) - \min_{\bx}f(\bx) ] \leq \ccalO(1/\sqrt{nm})$.

For the case of non-convex loss function $\ell$, however, finding the solution to the expected risk minimization problem in \eqref{eq:ERM} is hard. Even further, finding (or testing) a local optimum is NP-hard in many cases \citep{murty1987some}. Therefore, for non-convex losses we relax our main goal and instead look for first-order optimal solutions (or stationary points) for \eqref{eq:ERM}. That is, we aim to find a model $\hat{\bx}$ that satisfies $\norm{\gr f(\hat{\bx})} \leq \epsilon$ for an arbitrarily small approximation error $\epsilon$. \cite{mei2018landscape} characterized the gap for the gradients of the two expected risk and empirical risk functions. That is, if the gradient of loss is sub-Gaussian, then with high probability $\sup_{\bx} \norm{\gr L(\bx) - \gr f(\bx)} \leq \ccalO(1/\sqrt{nm})$. This result further implies that having all the nodes contribute in minimizing the empirical risk results in better approximation for a first-order stationary point of the expected risk $L$. In summary, our goal in non-convex setting is to find $\hat{\bx}$ that satisfies $\norm{\gr f(\bx)} \leq \ccalO(1/\sqrt{nm})$ which also implies $\norm{\gr L(\bx)} \leq \ccalO(1/\sqrt{nm})$.

%% file: 2-method.tex
\vspace{-1mm}
In this section, we present our proposed communication-efficient federated learning method called \texttt{FedPAQ}, which consists of three main modules: (1) periodic averaging, (2) partial node participation, and (3) quantized message passing.

\vspace{-1mm}
\subsection{Periodic averaging}\label{sec:p_avg}
\vspace{-1mm}

As explained in Section \ref{sec:setup}, to leverage from all the available data samples on the nodes, any training method should incorporate synchronizing the intermediate models obtained at local devices. One approach is to let the participating nodes synchronize their models through the parameter server in each iteration of the training. This, however, implies many rounds of communication between the federated nodes and the parameter server which results in communication contention over the network. Instead, we let the participating nodes conduct a number of local updates and synchronize through the parameter server periodically. To be more specific, once nodes pull an updated model from the server, they update the model locally by running $\tau$ iterations of the SGD method and then send  proper information to the server for updating the aggregate model. Indeed, this periodic averaging scheme reduces the rounds of communication between server and the nodes and consequently the overall communication cost of training the model. In particular, for the case that we plan to run $T$ iterations of SGD at each node, nodes need to communicate with the server $K=T/\tau$ rounds, hence reducing the total communication cost by a factor of $1/\tau$. 

Choosing a larger value of $\tau$ indeed reduces the rounds of communication for a fixed number of iterations $T$. However, if our goal is to obtain a specific accuracy $\eps$, choosing a very large value for $\tau$ is not necessarily optimal as by increasing $\tau$ the noise of the system increases and the local models approach the local optimal solutions instead of the global optimal solution. Hence, we might end up running more iterations $T$ to achieve a specific accuracy $\eps$ comparing to a case that $\tau$ is small. Indeed, a crucial question that we need to address is finding the optimal choice of $\tau$ for minimizing the overall communication cost of the process. 

\vspace{-1mm}
\subsection{Partial node participation}
\vspace{-1mm}

In a federated network, often there is a large number of devices such as smart phones communicating through a base station. On one hand, base stations have limited download bandwidth and hence only a few of devices are able to simultaneously upload their messages to the base station. Due to this limitation the messages sent from the devices will be pipelined at the base station which results in a dramatically slow training. On the other hand, having all of the devices participate through the whole training process induces a large communication overhead on the network which is often costly. Moreover, in practice not all the devices contribute in each round of the training. Indeed, there are multiple factors that determine whether a device can participate in the training \citep{mcmahan2017GoogleFed}: a device should be available in the reachable range of the base station; a device should be idle, plugged in and connected to a free wireless network during the training; etc. 

Our proposed \texttt{FedPAQ} method captures the restrictions mentioned above. In particular, we assume that among the total of $n$ devices, only $r$ nodes 
($r\leq n$) are available in each round of the training. We can also assume that due to the availability criterion described before, such available devices are randomly and uniformly distributed over the network \citep{sahu2018convergence}. In summary, in each period $k=0,1,\cdots,K-1$ of the training algorithm, the parameter server sends its current model $\bx_k$ to all the $r$ nodes in subset $\ccalS_k$, which are distributed uniformly at random among the total $n$ nodes, i.e., $\Pr{\ccalS_k} = 1/{n \choose r}$.

\vspace{-1mm}
\subsection{Quantized message-passing}\label{sec:quan}
\vspace{-1mm}

Another aspect of the communication bottleneck in federated learning is the limited uplink bandwidth at the devices which makes the communication from devices to the parameter server slow and expensive. Hence, it is critical to reduce the size of the uploaded messages from the federated devices \citep{li2019federated}. Our proposal is to employ quantization operators on the transmitted massages. Depending on the accuracy of the quantizer, the network communication overhead is reduced by exchanging the quantized updates. 

In the proposed \texttt{FedPAQ}, each node $i \in \ccalS_k$ obtains the model $\bx^{(i)}_{k,\tau}$ after running $\tau$ local iterations of an optimization method (possibly SGD) on the most recent model $\bx_k$ that it has received form the server. Then each node $i$ applies a quantizer operator $Q(\cdot)$ on the difference between the received model and its updated model, i.e., $\bx^{(i)}_{k,\tau} - \bx_k$, and uploads the quantized vector $Q(\bx^{(i)}_{k,\tau} - \bx_k)$ to the parameter server. Once these quantized vectors are sent to the server, it decodes the quantized signals and combines them to come up with a new model $\bbx_{k+1}$.

Next, we describe a widely-used random quantizer.

\vspace{2mm}
\begin{example}[Low-precision quantizer \citep{alistarh2017qsgd}] \label{ex:lp}
For any variable $\bx \in \mathbb{R}^p$, the low precision quantizer $ Q^{\text{LP}}: \mathbb{R}^p \rightarrow \mathbb{R}^p$ is defined as below
\begin{align}\label{eq:low_per_def}
    % Q^{\text{LP}}(\bx) &= [Q^{\text{LP}}_1 (\bx),\cdots,Q^{\text{LP}}_p (\bx)]^\top,
    % \quad \text{ and } \quad \\
    Q^{\text{LP}}_i (\bx) &= \norm{\bx} \cdot {\rm{sign}}(x_i) \cdot \xi_i(\bx,s), \quad  i \in [p],
\end{align}
where $ \xi_i(\bx,s) $ is a random variable taking on value $\nicefrac{l+1}{ s}$ with probability $\frac{|x_i|}{\norm{\bx}} s - l$ and $\nicefrac{l}{ s}$ otherwise.
% \begin{equation}\label{eq:low_per_def_2}
%     \xi_i(\bx,s) =
% \left\{
% 	\begin{array}{ll}
% 	\frac{l}{ s}  &\quad \mbox{w.p.} \ \ 1- \left(\frac{|x_i|}{\norm{\bx}} s - l\right), \vspace{2mm} \\ 
% 	\frac{l+1}{s} & \quad \mbox{w.p.} \ \ \frac{|x_i|}{\norm{\bx}} s - l. 
% 	\end{array}
% \right.
% \end{equation}
Here, the tuning parameter $s$ corresponds to the number of quantization levels and $l \in [0,s)$ is an integer such that $\nicefrac{|x_i|}{\norm{\bx}} \in [\nicefrac{l}{s}, \nicefrac{l+1}{s})$.
\end{example}

\vspace{-1mm}
\subsection{Algorithm update}
\vspace{-1mm}

\begin{comment}
\begin{algorithm}[t!]
\caption{\texttt{FedPAQ}}\label{alg:update}
%\Require No. of periods $K$, period length $\tau$, stepsize $\eta_{k,t}$
\For{$k=0,1,\cdots,K-1$}{
     server picks $r$ nodes $\ccalS_k$ uniformly at random \\
     server sends $\bx_{k}$ to nodes in $\ccalS_k$ \\
    \For{\text{node }$i \in \ccalS_k$}{
     $\bx^{(i)}_{k,0} \gets \bx_{k}$\\
    \For{$t=0,1,\cdots,\tau-1$}{
     randomly pick a data point $\xi \in \ccalD^{i}$\\
     compute $\tNab f_i (\bx) = \gr \ell (\bx, \xi)$\\
     set $\bx^{(i)}_{k,t+1} \gets \bx^{(i)}_{k, t} - \eta_{k,t} \tNab f_i \left(\bx^{(i)}_{k, t}\right)$
    }\textbf{end for}\\
     send $Q \left(\bx^{(i)}_{k, \tau} - \bx_{k} \right)$ to the server
    }\textbf{end for}\\
     server updates
     $\bx_{k+1} \gets \bx_{k} + \frac{1}{r} \sum_{i \in \ccalS_k} Q \left(\bx^{(i)}_{k, \tau} - \bx_{k} \right)$
}\textbf{end for}
\end{algorithm}
\end{comment}

%%%%%

Now we use the building blocks developed in Sections \ref{sec:p_avg}-\ref{sec:quan} to precisely present \texttt{FedPAQ}. Our proposed method consists of $K$ periods, and during a period, each node performs $\tau$ local updates, which results in total number of  $T=K\tau$  iterations. In each period $k = 0,\cdots,K-1$ of the algorithm, the parameter server picks $r \leq n$ nodes uniformly at random which we denote by $\ccalS_k$. The parameter server then broadcasts its current model $\bx_k$ to all the nodes in $\ccalS_k$ and each node $i\in \ccalS_k$ performs $\tau$ local SGD updates using its local dataset. To be more specific, let $\bx^{(i)}_{k,t}$ denote the model at node $i$ at $t$-th iteration of the $k$-th period. At each local iteration $t=0,\cdots,\tau-1$, node $i$ updates its local model according to the following rule:
\vspace{-1mm}
\begin{equation}\label{eq:sgd}
    \bx^{(i)}_{k,t+1} 
    =
    \bx^{(i)}_{k, t} - \eta_{k,t} \tNab f_i \left(\bx^{(i)}_{k, t}\right),\vspace{-1mm}
\end{equation}
where the stochastic gradient $\tNab f_i$ is computed using a random sample\footnote{The method can be easily made compatible with using a mini-batch during each iteration.} picked from the local dataset $\ccalD^{i}$. Note that all the nodes begin with a common initialization $\bx^{(i)}_{k, 0} = \bx_k$. After $\tau$ local updates, each node computes the overall update in that period, that is  $\bx^{(i)}_{k, \tau} - \bx_k$, and uploads a quantized update $Q (\bx^{(i)}_{k, \tau} - \bx_k)$ to the parameter server.  
\begin{algorithm}[h!]
\caption{\texttt{FedPAQ}}\label{alg:update}
\begin{algorithmic}[1]
%\Require No. of periods $K$, period length $\tau$, stepsize $\eta_{k,t}$
\For{$k=0,1,\cdots,K-1$}
    \State server picks $r$ nodes $\ccalS_k$ uniformly at random 
    \State server sends $\bx_{k}$ to nodes in $\ccalS_k$
    \For{node $i \in \ccalS_k$}
    \State $\bx^{(i)}_{k,0} \gets \bx_{k}$
    \For{$t=0,1,\cdots,\tau-1$}
%    \State randomly pick a data point $\xi \in \ccalD^{i}$
    \State compute stochastic gradient 
    \State $\tNab f_i (\bx) = \gr \ell (\bx, \xi)$ for a $\xi \in \ccalP^i$
    \State set $\bx^{(i)}_{k,t+1} \gets \bx^{(i)}_{k, t} - \eta_{k,t} \tNab f_i (\bx^{(i)}_{k, t})$
    \EndFor
    \State send $Q (\bx^{(i)}_{k, \tau} - \bx_{k} )$ to the server
    \EndFor
    \State server finds
     $\bx_{k+1} \gets \bx_{k} + \frac{1}{r} \sum_{i \in \ccalS_k} Q (\bx^{(i)}_{k, \tau} - \bx_{k} )$
\EndFor
\end{algorithmic}
\end{algorithm}
The parameter server then aggregates the $r$ received quantized local updates and computes the next model according to 
\vspace{-1mm}
\begin{equation}
    \bx_{k+1} 
    =
    \bx_{k} + \frac{1}{r} \sum_{i \in \ccalS_k} Q \left(\bx^{(i)}_{k, \tau} - \bx_{k} \right),\vspace{-1mm}
\end{equation}
and the procedure is repeated for $K$ periods. The proposed method is formally summarized in Algorithm~\ref{alg:update}.

%% file: 3-convergence.tex
\vspace{-1mm}
In this section, we present our theoretical results on the guarantees of the \texttt{FedPAQ} method. We first consider the strongly convex setting and state the convergence guarantee of \texttt{FedPAQ} for such losses in Theorem~\ref{thm:1}. Then, in Theorem~\ref{thm:2}, we present the overall complexity of our method for finding a first-order stationary point of the aggregate objective function $f$, when the loss function $\ell$ is non-convex (All proofs are provided in  the supplementary material).  
 Before that, we first mention three customary assumptions required for both convex and non-convex settings.

\vspace{1mm}

\begin{assumption}\label{assump-Q}
The random quantizer $Q(\cdot)$ is unbiased and its variance grows with the squared of $l_2$-norm of its argument, i.e.,
%%%%%%
\vspace{-1mm}
\begin{equation}\label{quant_cond}
    \bE \left[Q(\bx)| \bx \right] = \bx, \quad
    \bE \left[ \norm{Q(\bx)-\bx}^2 | \bx \right] \leq q \norm{\bx}^2,\vspace{-1mm}
\end{equation}
%%%%%%
\vspace{-1mm}
for some positive real constant $q$ and any $\bx \in \reals^{p}$.
\end{assumption}

\vspace{.5mm}

\begin{assumption}\label{assump-smooth}
The loss functions $f_i$ are $L$-smooth with respect to $\bx$, i.e., for any $\bx, \hbx \in \reals^{p}$, we have
%%%%%
$ \norm{\nabla  f_i(\bx) - \gr f_i(\hbx)} \leq L \norm{\bx - \hbx}$.
%%%%%
\end{assumption}

\vspace{.5mm}

\begin{assumption}\label{assump-gr}
Stochastic gradients $\tNab f_i (\bx)$ are unbiased and variance bounded, i.e., 
%%%%%
$\mathbb{E}_{\xi} [ \tNab f_i (\bx)]  =  \gr f_i (\bx)$ and $ \mathbb{E}_{\xi} [ \| \tNab f_i (\bx) - \gr f_i (\bx) \|^2 ]\leq \sigma^2.$
%%%%%
\end{assumption}

\vspace{0mm}

The conditions in Assumption~\ref{assump-Q} ensure that output of quantization is an unbiased estimator of the input with a variance that is proportional to the norm-squared of the input. This condition is satisfied with most common quantization schemes including the low-precision quantizer introduced in Example~1. 
Assumption \ref{assump-smooth} implies that the gradients of local functions $\nabla f_i$ and the aggregated objective function $\nabla f$ are also $L$-Lipschitz continuous.
The conditions in Assumption \ref{assump-gr} on the bias and variance of stochastic gradients are also customary. Note that this is a much weaker assumption compared to the one that uniformly bounds the expected norm of the stochastic gradient.

\vspace{1mm}

\noindent\textbf{Challenges in analyzing the \texttt{FedPAQ} method.} 
Here, we highlight the main theoretical challenges in proving our main results. As outlined in the description of the proposed method, in the $k$-th round of \texttt{FedPAQ}, each participating node $i$ updates its local model for $\tau$ iterations via SGD method in~\eqref{eq:sgd}. Let us focus on a case that we use a constant stepsize for the purpose of this discussion. First consider the naive parallel SGD case which corresponds to $\tau=1$. The updated  local model after $\tau=1$ local update is
\begin{align}
    \bx^{(i)}_{k,\tau} 
    =
    \bx^{(i)}_{k, 0} - \eta \tNab f_i \left(\bx^{(i)}_{k, 0}\right).
\end{align}
Note that $\bx^{(i)}_{k, 0} = \bx_k$ is the parameter server's model sent to the nodes. Since we assume the stochastic gradients are unbiased estimators of the gradient, it yields that the local update $\bx^{(i)}_{k, \tau} - \bx_k$ is an unbiased estimator of $- \eta \gr f (\bx_k)$ for every participating node. Hence, the aggregated updates at the server and the updated model $ \bx_{k+1}$ can be simply related to the current model $ \bx_k$ as one step of parallel SGD. However, this is not the case when the period length $\tau$ is larger than $1$. For instance, in the case that $\tau=2$, the local updated model after $\tau=2$ iterations is
\begin{align}
    \bx^{(i)}_{k,\tau} 
    \!=\!
    \bx_k \!-\! \eta \tNab f_i \left(\bx_k\right) \!-\! \eta \tNab f_i \left(\bx_k \!-\! \eta \tNab f_i \left(\bx_k\right)\right).
\end{align}
Clearly,  $\bx^{(i)}_{k, \tau} - \bx_k$ is not an unbiased estimator of $- \eta \gr f (\bx_k)$ or $- \eta \gr f (\bx_k - \eta \gr f (\bx_k))$. This demonstrates that the aggregated model at server cannot be treated as $\tau$ iterations of parallel SGD, since each local update contains a bias. Indeed, this bias gets propagated when $\tau$ gets larger. For our running example  $\tau=2$, the variance of the bias, i.e. $\bE \| \eta \tNab f_i(\bx_k - \eta \tNab f_i (\bx_k)) \| ^2$ is not uniformly bounded either (Assumption \ref{assump-gr}), which makes the analysis even more challenging compared to the works with bounded gradient assumption (e.g. \citep{stich2018local,yu2019parallel}). % as we do not assume such a strong condition unlike \citep{stich2018local,yu2019parallel}.

\vspace{-1mm}
\subsection{Strongly convex setting}
\vspace{-1mm}

Now we proceed to establish the convergence rate of the proposed \texttt{FedPAQ} method for a federated setting with strongly convex and smooth loss function $\ell$. We first formally state the strong convexity assumption.

\vspace{2mm}
\begin{assumption}\label{assump-convex}
The loss functions $f_i$ are $\mu$-strongly convex, i.e., for any $\bx, \hbx \in \reals^{p}$ we have that  
%%%%%
$   \langle \gr f_i (\bx) - \gr f_i (\hbx), \bx - \hbx  \rangle \geq \mu \norm{\bx - \hbx}^2.
$
%%%%%
\end{assumption}

%This assumption implies that the local functions $f_i$ and the aggregated function $f$ are also $\mu$-strongly convex.

\vspace{2mm}

\begin{theorem} [Strongly convex loss] \label{thm:1}
Consider the sequence of iterates $\bx_k$ at the parameter server generated according to the \texttt{FedPAQ} method outlined in Algorithm~\ref{alg:update}. Suppose the conditions in Assumptions \ref{assump-Q}--\ref{assump-convex} are satisfied. Further, let us define the constant $B_1$ as
\vspace{-1mm}
\begin{equation}
    B_1
    =
    2  L^2 \left(\frac{q}{n}  +  \frac{n-r}{r(n-1)} 4(1+q)  \right),\vspace{-1mm}
\end{equation}
where $q$ is the quantization variance parameter defined in \eqref{quant_cond} and $r$ is the number of active nodes at each round of communication. If we set the stepsize in \texttt{FedPAQ} as $\eta_{k,t} = \eta_k = \nicefrac{4 \mu^{-1}}{k \tau + 1}$, then for any $k\geq k_0$ where $k_0$ is the smallest integer satisfying 
%%%%%%
%%%%%%
\vspace{-1mm}
\begin{equation}\label{eq:k0}
    k_0
    \geq
    4 \max \bigg\{ \frac{L}{\mu}, 4 \left( \frac{B_1}{\mu^2} + 1 \right), \frac{1}{\tau}, \frac{4n}{\mu^2 \tau} \bigg\}, \vspace{-1mm}
\end{equation}
%%%%%%
the expected error $\bE [\| \bx_{k} - \bx^*\|]^2$ is bounded above by
\begin{align}
    &\mathbb{E} \norm{ \bx_{k} - \bx^*}^2
    \leq
    \frac{(k_0 \tau + 1)^2}{(k \tau + 1)^2}  \norm{ \bx_{k_0} - \bx^*}^2 \\
    & \quad +
    C_1 \frac{\tau}{k \tau + 1}
    +
    C_2 \frac{(\tau - 1)^2}{k \tau + 1}
    +
    C_3 \frac{\tau - 1}{(k \tau + 1)^2}, \label{eq:thm1-c}
\end{align}
%%%%%%
%%%%%%
where the constants in \eqref{eq:thm1-c} are defined as
%%%%%%
\begin{align}
   \!\! C_1\!
     &= \!
    \frac{16\sigma^2}{\mu^2n} \!\left(\!1\! + \!2q +\!  8  (1\!+\!q)\frac{n(n\!-\!r)}{r(n\!-\!1)}\! \right)\!,\
    C_2\!
     =\! 
    \frac{16eL^2\sigma^2}{\mu^2n}  ,  \\
    \!C_3\!
    & =
    \frac{256 e L^2\sigma^2}{\mu^4 n} \left(n + 2q + 8(1+q)\frac{n(n-r)}{r(n-1)}   
    \right)  .
\end{align}
%%%%%%
\end{theorem}

%\begin{proof}
%See Section \ref{sec:thm1-proof}.
%\end{proof}

\begin{remark}
Under the same conditions as in Theorem~\ref{thm:1} and for a total number of iterations $T = K \tau \geq k_0 \tau$ we have the following convergence rate
%%%%%%
\vspace{-2mm}
\begin{align}
    &\bE \| \bx_{K} - \bx^*\|^2
    \leq
    \ccalO \bigg( \frac{\tau}{T } \bigg)
    +
    \ccalO \bigg( \frac{\tau^2}{T^2} \bigg) \\
    & \quad  +
    \ccalO \bigg( \frac{(\tau - 1)^2}{T } \bigg)
    +
    \ccalO \bigg( \frac{\tau - 1}{T^2}\bigg).
\end{align}
%%%%%%
As expected, the fastest convergence rate is attained when the contributing nodes synchronize with the parameter server in each iteration, i.e. when $\tau =1$. Theorem \ref{thm:1} however characterizes how large the period length $\tau$ can be picked. In particular, any pick of $\tau = o (\sqrt{T})$ ensures the convergence of the \texttt{FedPAQ} to the global optimal for strongly convex losses.
\end{remark}

\begin{remark}
By setting $\tau=1$, $q=0$ and $r=n$, Theorem \ref{thm:1} recovers the convergence rate of vanilla parallel SGD, i.e., $\ccalO(1/T)$ for strongly-convex losses. Our result is however more general since we remove the uniformly bounded assumption on the norm of stochastic gradient. For $\tau \geq 1$, Theorem \ref{thm:1} does not recover the result in \citep{stich2018local} due to our weaker condition in Assumption \ref{assump-gr}. Nevertheless, the same rate $\ccalO(1/T)$ is guaranteed by \texttt{FedPAQ} for constant values of $\tau$.
\end{remark}

\vspace{-1mm}
\subsection{Non-convex setting}
\vspace{-1mm}

We now present the convergence result of  \texttt{FedPAQ}  for smooth non-convex loss functions.

\begin{theorem} [Non-convex Losses] \label{thm:2}

Consider the sequence of iterates $\bx_k$ at the parameter server generated according to the \texttt{FedPAQ} method outlined in Algorithm~\ref{alg:update}. Suppose the conditions in Assumptions \ref{assump-Q}--\ref{assump-gr} are satisfied. Further, let us define the constant $B_2$ as
%%%
\vspace{-1mm}
\begin{equation}
    B_2 \coloneqq
    \frac{q}{n}
    +
 \frac{4(n-r)}{r(n-1)} (1+q),\vspace{-1mm}
\end{equation}
%%%%%%
where $q$ is the quantization variance parameter defined in \eqref{quant_cond} and $r$ is the number of active nodes at each round. If 
the total number of iterations $T$ and the period length $\tau$ satisfy the following conditions,  
%%%%%%
\vspace{-1mm}
\begin{equation}\label{eq:thm2-cond}
    T\geq 2, \qquad
    \tau
    \leq
    \frac{\sqrt{ B_2^2 + 0.8 } -B_2}{8} \sqrt{T}, \vspace{-1mm}
\end{equation}
%%%%%%
and  we set the stepsize as $\eta_{k,t} = \nicefrac{1}{L\sqrt{T}}$, then the following first-order stationary condition holds
%%%%%%
\begin{align}
    & \frac{1}{T} \sum_{k=0}^{K-1} \sum_{t=0}^{\tau-1} \bE \norm{\gr f (\xbar_{k,t})}^2 \\
    & \quad \leq
    \frac{2L(f (\bx_{0}) - f^*)}{\sqrt{T}}
  +
    N_1 \frac{1}{\sqrt{T}} 
    + 
    N_2 \frac{\tau-1}{T}, \label{eq:thm2}
\end{align}
%%%%%%
where the constants in \eqref{eq:thm2} are defined as 
%%%%%%
\begin{align*}
    N_1
\coloneqq
    (1+q) \frac{\sigma^2}{n} \left(
    1
    +
     \frac{n(n-r)}{r(n-1)}
    \right), \quad \!\!
    N_2
     \coloneqq
\frac{\sigma^2}{n} (n+1).
\end{align*}
%%%%%%

\end{theorem}

%\begin{proof}
%See Section \ref{sec:thm2-proof}.
%\end{proof}

\begin{remark}
The result in Theorem \ref{thm:2} implies the following order-wise rate 
%%%%%%
\begin{align*}
    \frac{1}{T} \sum_{k=0}^{K-1} \sum_{t=0}^{\tau-1} \bE \norm{\gr f (\xbar_{k,t})}^2
    \leq
    \ccalO \left( \frac{1}{\sqrt{T}} \right) 
    + 
    \ccalO \left(\frac{\tau\!-\!1}{T} \right).
\end{align*}
%%%%%%
Clearly, the fastest convergence rate is achieved for the smallest possible period length, i.e., $\tau=1$. This however implies that the edge nodes communicate with the parameter server in each iteration, i.e. $T$ rounds of communications which is costly. On the other hand, the conditions \eqref{eq:thm2-cond} in Theorem \ref{thm:2} allow the period length $\tau$ to grow up to $\ccalO(\sqrt{T})$ which results in an overall convergence rate of $\ccalO(1/\sqrt{T})$ in reaching an stationary point. This result shows that with only $\ccalO(\sqrt{T})$ rounds of communication \texttt{FedPAQ} can still ensure the convergence rate of $\ccalO(1/\sqrt{T})$ for non-convex losses.
\end{remark}

\begin{remark}
Theorem \ref{thm:2} recovers the convergence rate of the vanilla parallel SGD \citep{yu2019parallel} for non-convex losses as a special case of $\tau=1$, $q=0$ and $r=n$. Nevertheless, we remove the uniformly bounded assumption on the norm of the stochastic gradient in our theoretical analysis. We also recover the result in \citep{wang2018cooperative} when there is no quatization $q=0$ and we have a full device participation $r=n$.
\end{remark}

It is worth mentioning that for Theorems \ref{thm:1} and \ref{thm:2}, one can use a batch of size $m$ for each local SGD update and the same results hold by changing $\nicefrac{\sigma^2}{n}$ to $\nicefrac{\sigma^2}{mn}$.

%% file: 7-numerical.tex
\vspace{-1mm}

The proposed \texttt{FedPAQ} method reduces the communication load by employing three modules: periodic averaging, partial node participation, and quantization. This communication reduction however comes with a cost in reducing the convergence accuracy and hence requiring more iterations of the training, which we characterized in Theorems \ref{thm:1} and \ref{thm:2}. In this section, we empirically study this communication-computation trade-off and evaluate \texttt{FedPAQ} in comparison to other benchmarks. To evaluate the total cost of a method, we first need to specifically model such cost. We consider the total training time as the cost objective which consists of communication and computation time \citep{berahas2018balancing,reisizadeh2019robust}. Consider $T$ iterations of training with \texttt{FedPAQ} that consists of $K=T/\tau$ rounds of communication. In each round, $r$ workers compute $\tau$ iterations of SGD with batchsize $B$ and send a quantized vector of size $p$ to the server. 

\begin{figure*}[t!]
\centering
    \includegraphics[width=0.245\linewidth]{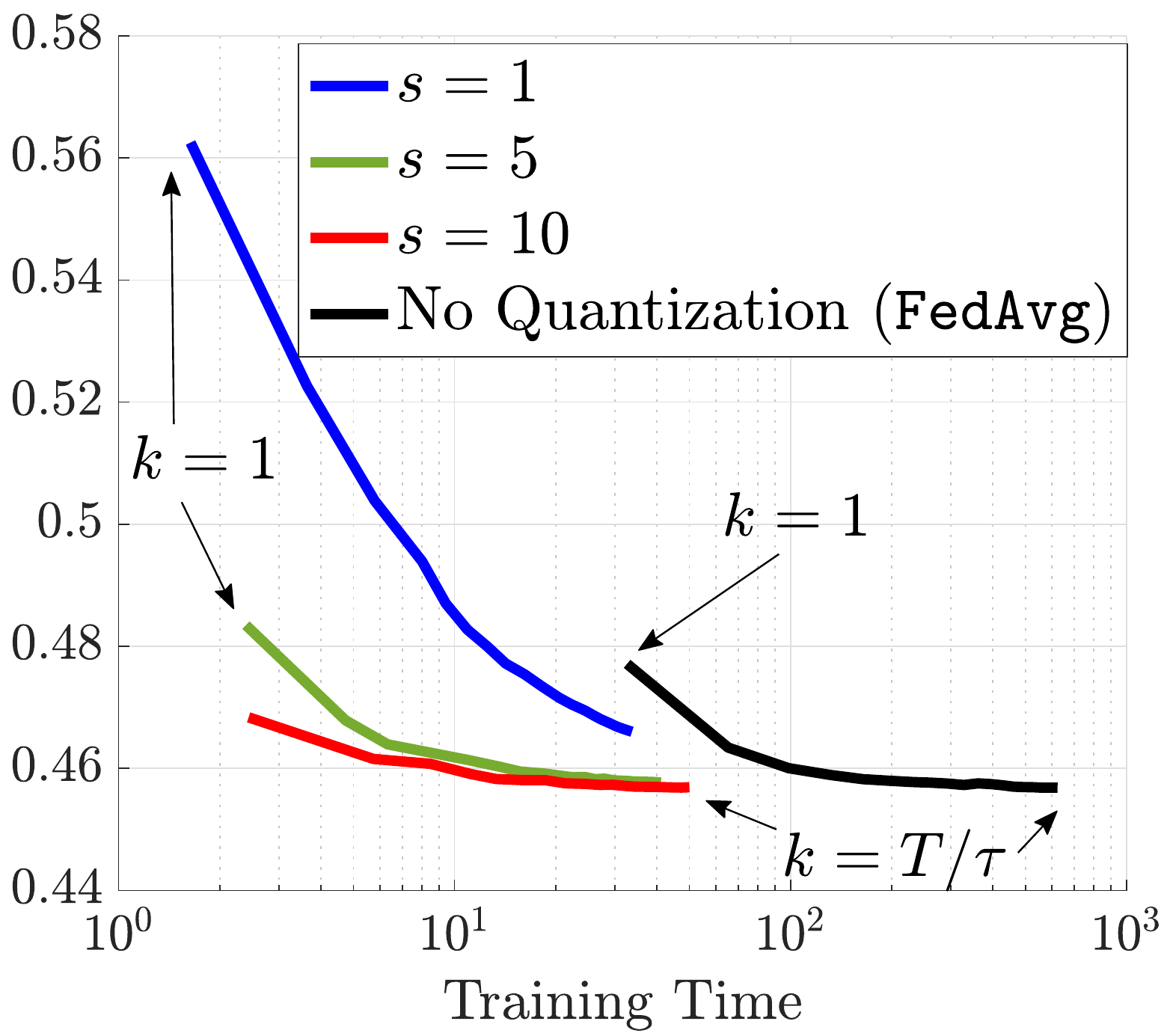}
    \includegraphics[width=0.245\linewidth]{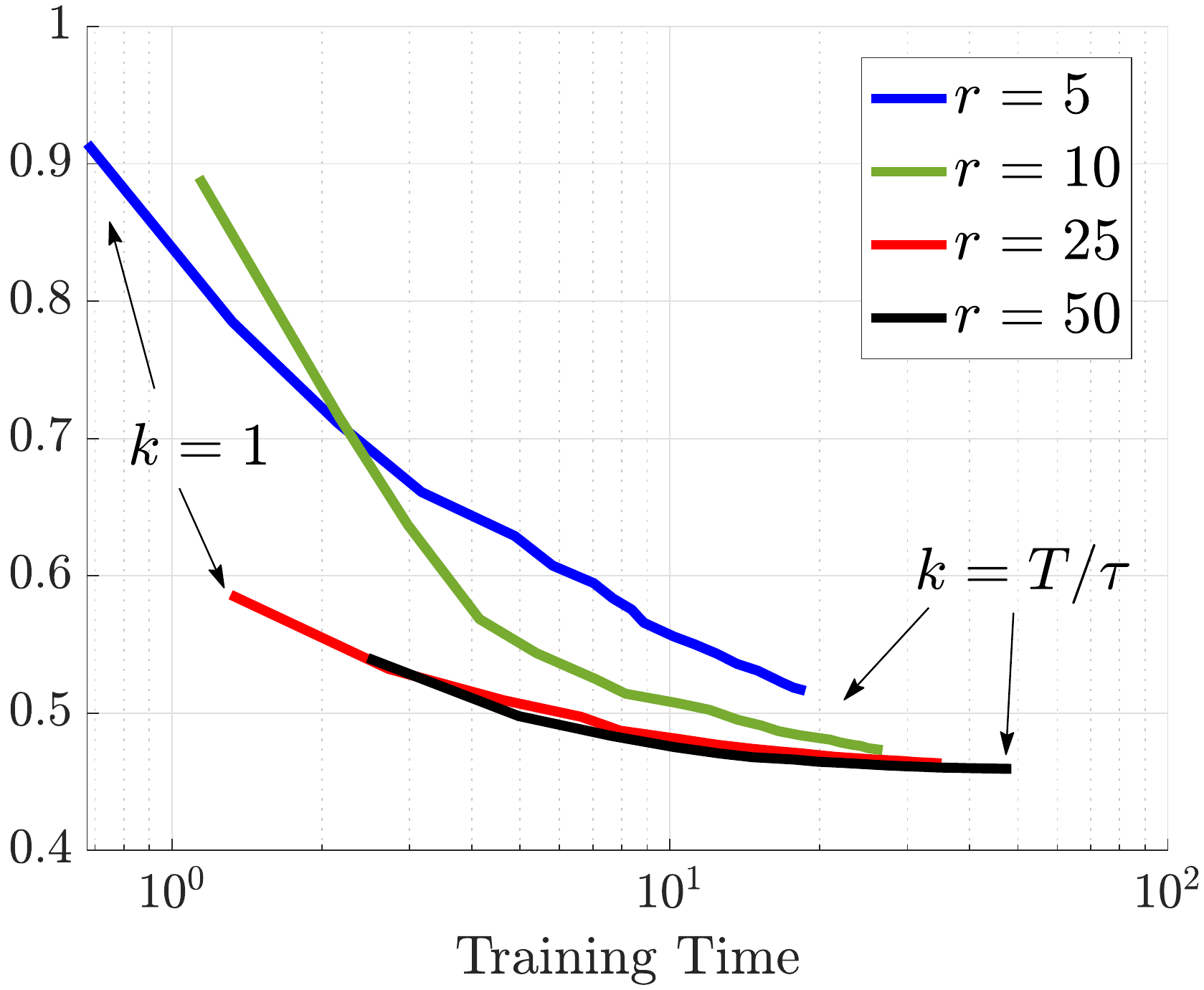}
    \includegraphics[width=0.245\linewidth]{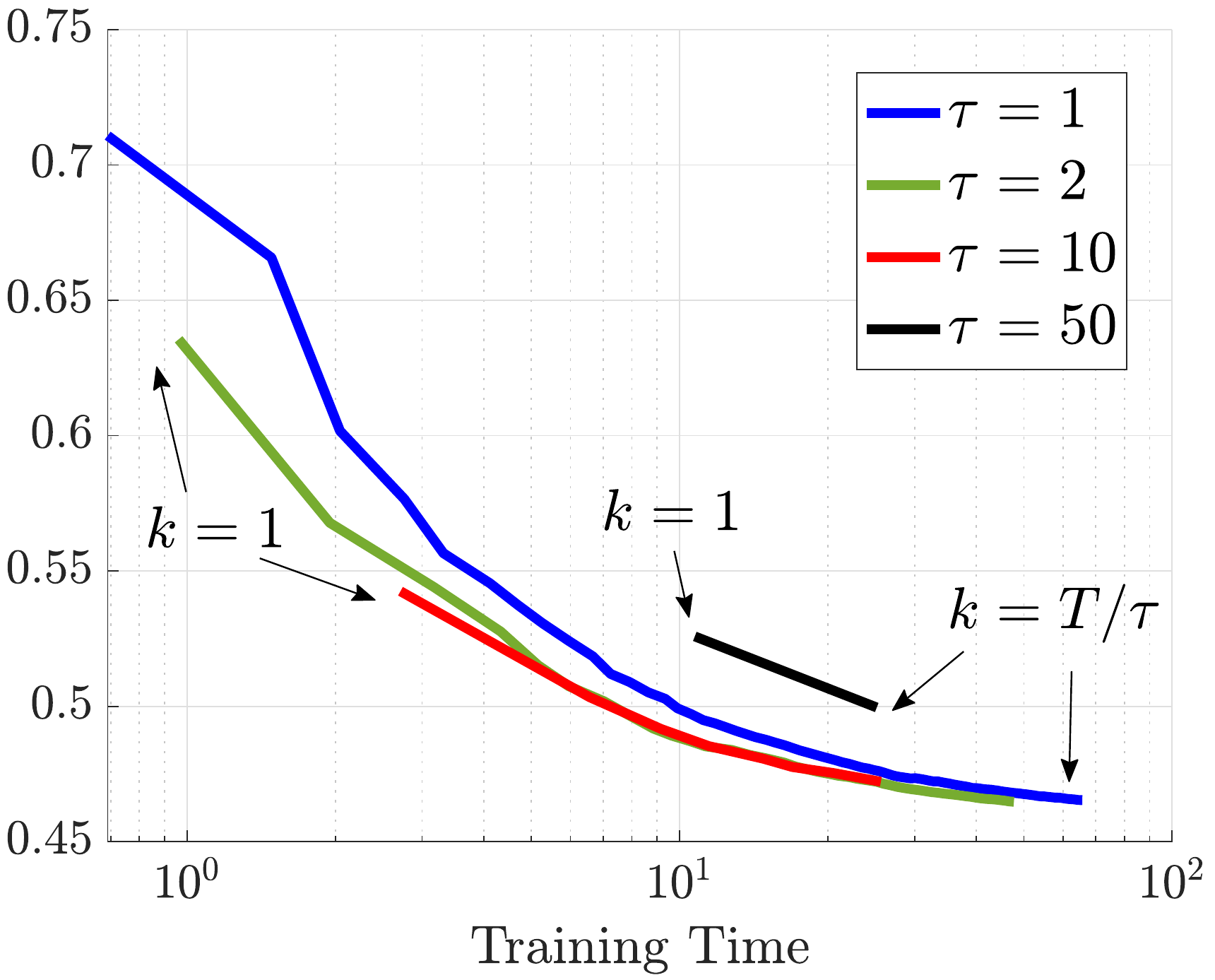}
    \includegraphics[width=0.245\linewidth]{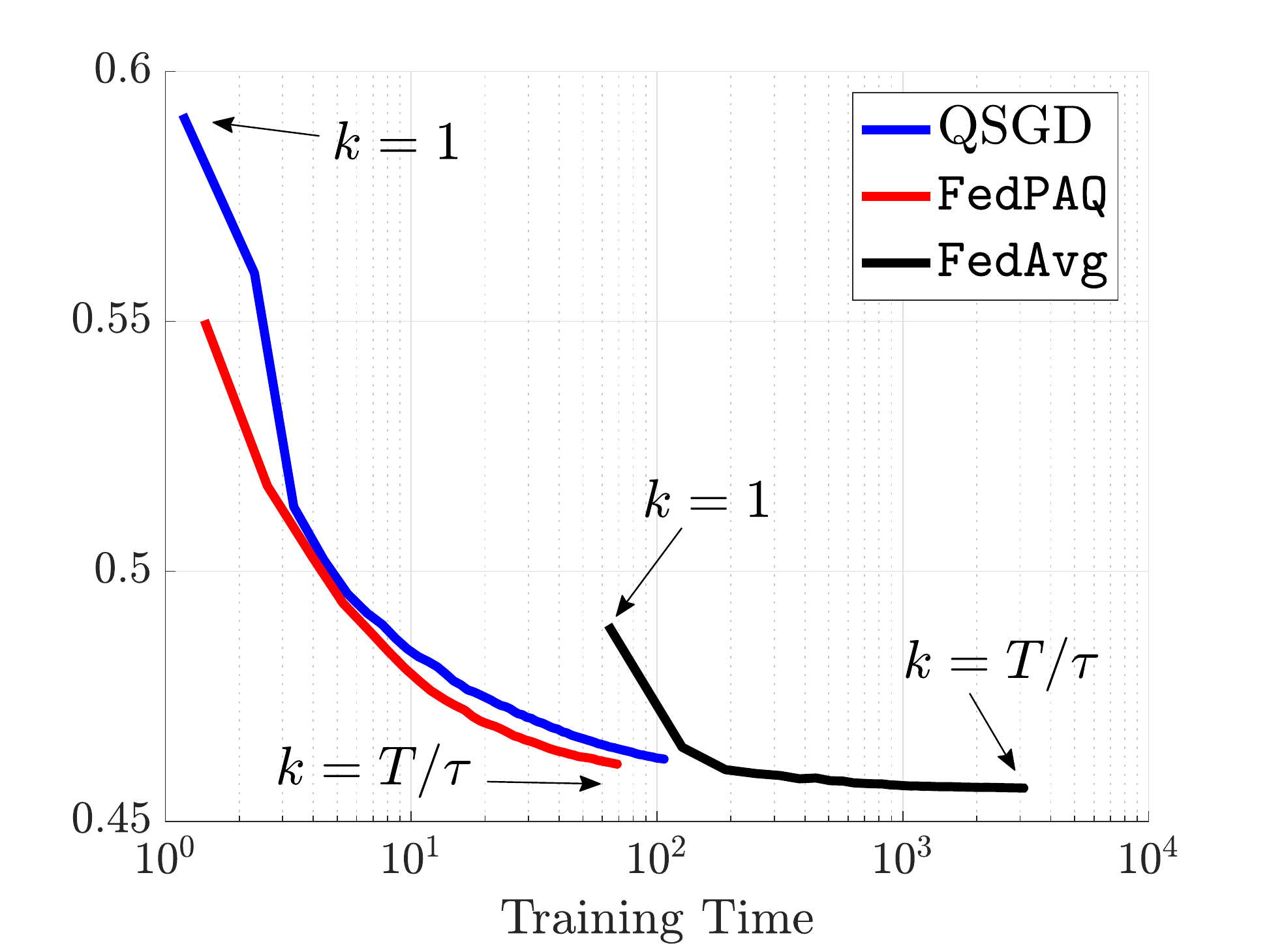}
    %\caption{Logistic Regression - MNIST}%\vspace{-1mm}
    \includegraphics[width=0.245\linewidth]{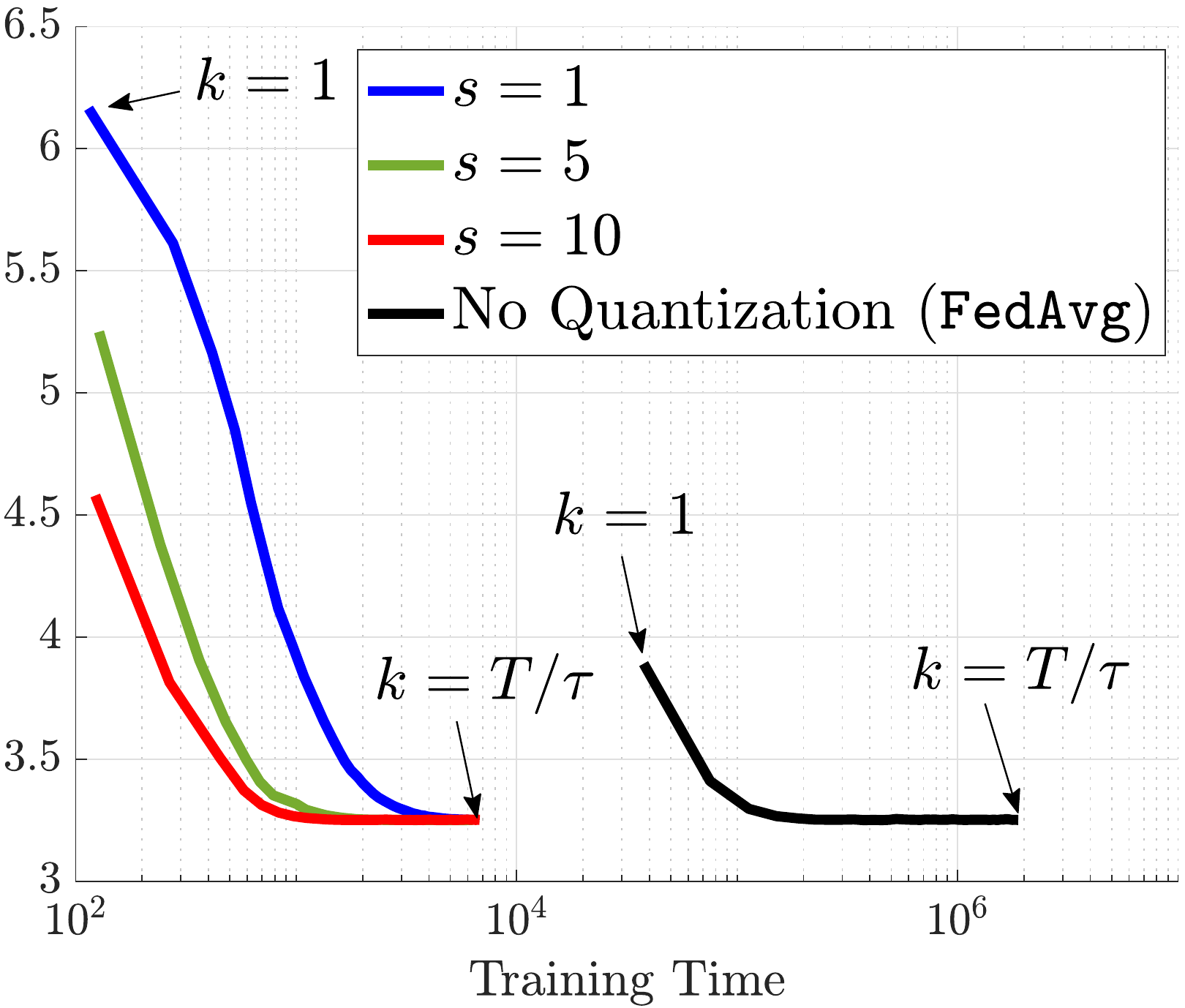}
    \includegraphics[width=0.245\linewidth]{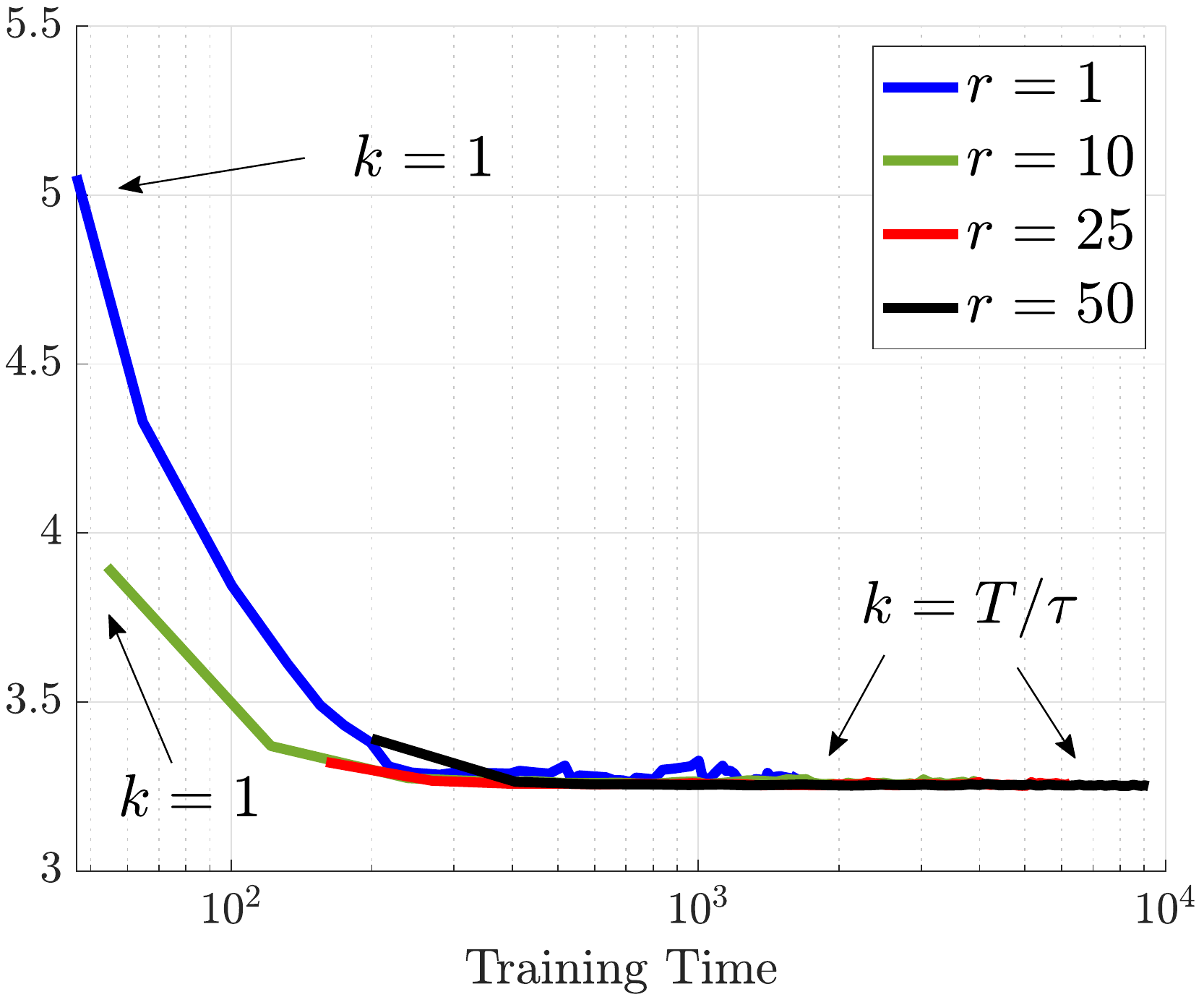}
    \includegraphics[width=0.245\linewidth]{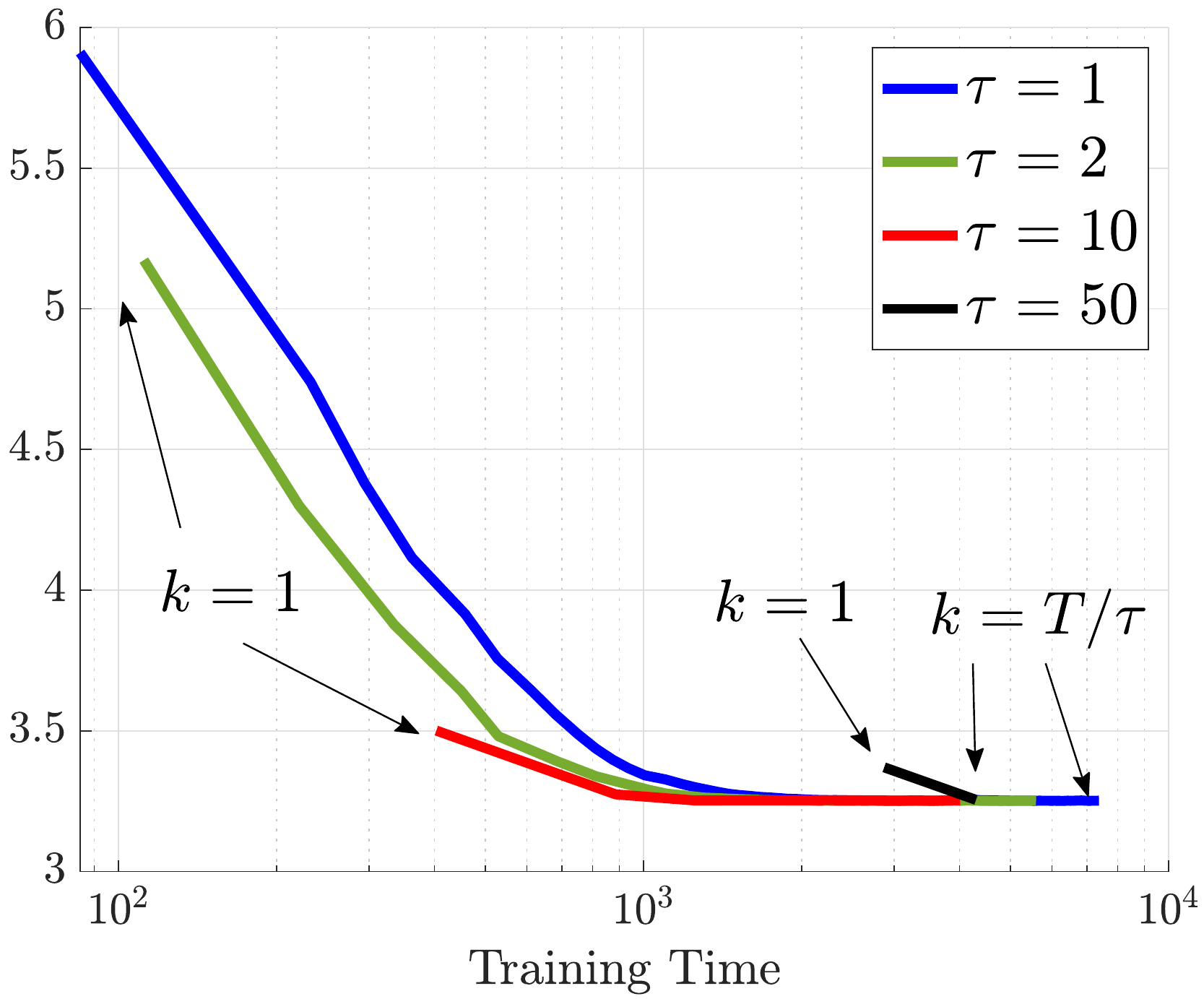}
    \includegraphics[width=0.245\linewidth]{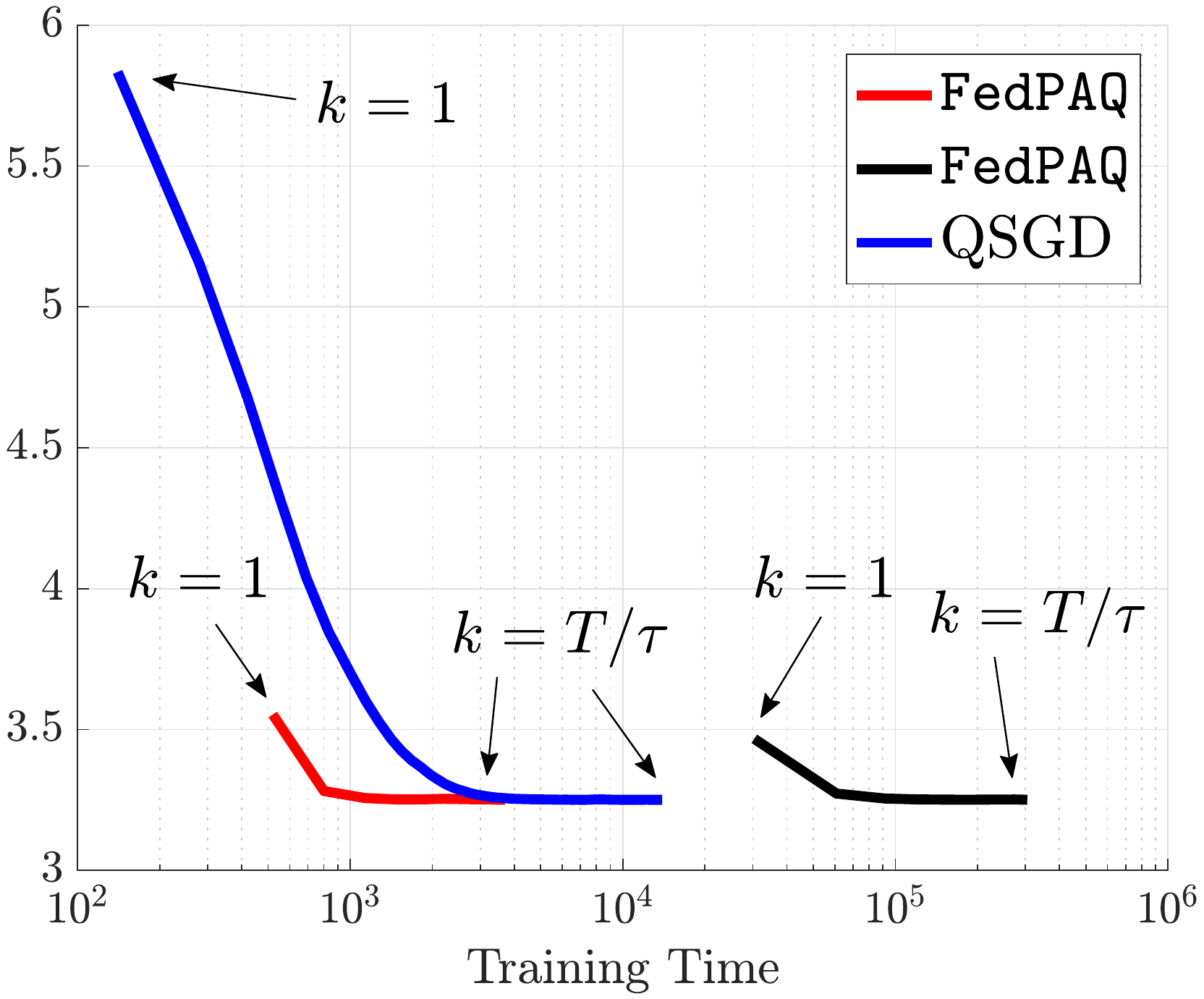}
    \vspace{-5mm}
    \caption{Training Loss vs. Training Time: Logistic Regression on MNIST (top). Neural Network on CIFAR-10 (bottom).}\vspace{-3mm}
    \label{fig:MNIST-and-cifar-92Kpar}
\end{figure*}

\noindent \textbf{Communication time.} We fix a bandwidth $\texttt{BW}$ and define the communication time in each round as the total number of uploaded bits divided by $\texttt{BW}$. Total number of bits in each round is $r \cdot |Q(p,s)|$, where $|Q(p,s)|$ denotes the number of bits required to encode a quantized vector of dimension $p$ according to a specific quantizer with $s$ levels. In our simulations, we use the low-precision quantizer described in Example \ref{ex:lp} and  assume it takes $pF$ bits to represent an unquantized vector of length $p$, where $F$ is typically $32$ bits.

\noindent \textbf{Computation time.} We consider the well-known shifted-exponential model for gradient computation time \citep{lee2017speeding}. In particular, we assume that for any node, computing the gradients in a period with $\tau$ iterations and using batchsize $B$ takes a deterministic shift $\tau \cdot B \cdot \texttt{shift}$ plus a random exponential time with mean value $\tau \cdot B \cdot \texttt{scale}^{-1}$, where $\texttt{shift}$ and $\texttt{scale}$ are respectively shift and scale parameters of the shifted-exponential distribution. Total computation time of each round is then the largest local computation time among the $r$ contributing nodes.
We also define a communication-computation ratio
\begin{align}
    \frac{C_{\texttt{comm}}}{C_{\texttt{comp}}}
    =
    \frac{pF/\texttt{BW}}{\texttt{shift}+1/\texttt{scale}} \nonumber
\end{align}
as the communication time for a length-$p$-vector over the average computation time for one gradient vector. This ratio captures the relative cost of communication and computation, and since communication is a major bottleneck, we have ${C_{\texttt{comm}}}/{C_{\texttt{comp}}} \gg 1$. In all of our experiments, we use batchsize $B=10$ and finely tune the stepsize's coefficient. 

\vspace{-2mm}
\subsection{Logistic Regression on MNIST}
\vspace{-2mm}

In Figure \ref{fig:MNIST-and-cifar-92Kpar}, the top four plots demonstrate the training time for a regularized logistic regression problem over MNIST dataset (`0' and `8' digits) for $T=100$ iterations. The network has $n=50$ nodes each loaded with $200$ samples. We set $C_{\texttt{comm}}/C_{\texttt{comp}} = 100/1$ to capture the communication bottleneck. Among the three parameters quantization levels $s$, number of active nodes in each round $r$, and period length $\tau$, we fix two and vary the third one. First plot demonstrates the relative training loss for different quantization levels $s \in \{1,5,10\}$ and the case with no quantization which corresponds to the \texttt{FedAvg} method \citep{mcmahan2016communication}. The other two parameters are fixed to $(\tau,r) = (5,25)$. Each curve shows the training time versus the achieved training loss for the aggregated model at the server for each round $k=1,\cdots,T/\tau$. In the second plot, $(s,\tau)=(1,5)$ are fixed. The third plot demonstrates the effect of period length $\tau$ in the communication-computation tradeoff. As demonstrated, after $T/\tau$ rounds, smaller choices for $\tau$ (e.g. $\tau=1,2$) result in slower convergence while the larger ones (e.g. $\tau=50$) run faster though providing less accurate models. Here $\tau=10$ is the optimal choice. The last plot compares the training time of \texttt{FedPAQ} with two other benchmarks \texttt{FedAvg} and QSGD. For both \texttt{FedPAQ} and  \texttt{FedAvg}, we set $\tau=2$ while \texttt{FedPAQ} and QSGD use quantization with $s=1$ level. All three methods use $r=n=50$ nodes in each round.

\vspace{-2mm}
\subsection{Neural Network training over CIFAR-10} \label{sec:numerical-cifar10}
\vspace{-2mm}

We conduct another set of numerical experiments to evaluate the performance of \texttt{FedPAQ} on non-convex and smooth objectives. Here we train a neural network with four hidden layers consisting of $n=50$ nodes and more thatn $92$K parameters, where we use $10$K samples from CIFAR-10 dataset with $10$ labels. Since models are much larger than the previous setup, we increase the communication-computation ratio to $C_{\texttt{comm}}/C_{\texttt{comp}} = 1000/1$ to better capture the communication bottleneck for large models. The bottom four plots in Figure~\ref{fig:MNIST-and-cifar-92Kpar} demonstrate the training loss over time for $T=100$ iterations. In the first plot, $(\tau,r)=(2,25)$ are fixed and we vary the quantization levels. The second plot shows the effect of $r$ while $(s,\tau)=(1,2)$. The communication-computation tradeoff in terms of period length $\tau$ is demonstrated in the third plot, where picking $\tau=10$ turns out to attain the fastest convergence. Lastly, we compare \texttt{FedPAQ} with other benchmarks in the forth plot. Here, we set $(s,r,\tau)=(1,20,10)$ in \texttt{FedPAQ}, $(r,\tau)=(20,10)$ in \texttt{FedAvg} and $(s,r,\tau)=(1,50,1)$ for QSGD.% where \texttt{FedPAQ} results in the smallest training time.

%% file: 6-conclusion.tex
\vspace{-3mm}
In this paper, we addressed some of  the communication and scalability challenges of federated learning and proposed \texttt{FedPAQ}, a communication-efficient federated learning method with provable performance guarantees. \texttt{FedPAQ} is based on three modules: (1) periodic averaging in which each edge node performs local iterative updates; (2) partial node participation which captures the random availability of the edge nodes; and (3) quantization in which each model is quantized before being uploaded to the server. We provided rigorous analysis for our proposed method for two general classes of strongly-convex and non-convex losses. We further provided numerical results evaluating the performance of \texttt{FedPAQ}, and discussing the trade-off between communication and computation.  %Our analysis lets a few customary yet mild assumptions to hold, which also yields theoretical challenges.

%% file: 5-proofs.tex
We first introduce some additional notations which will be used throughput the proofs.

\noindent\textbf{Additional notations.} %Let us first establish the following notations which will be used throughout the proofs. 
For each period $k=0,1,\cdots,K-1$ and iteration $t=0,1,\cdots,\tau-1$ we denote
%%%%%%
\begin{align}
    \bx_{k+1} 
    & \coloneqq
    \bx_{k} + \frac{1}{r} \sum_{i \in \ccalS_k} Q \left(\bx^{(i)}_{k, \tau} - \bx_{k} \right), \\
    \widehat{\bx}_{k+1} 
    & \coloneqq
    \bx_{k} + \frac{1}{n} \sum_{i \in [n]} Q \left(\bx^{(i)}_{k,\tau} - \bx_{k} \right), \\
    \xbar_{k,t}
    & \coloneqq
    \frac{1}{n} \sum_{i \in [n]} \bx^{(i)}_{k,t}. \label{eq:notation}
\end{align}
%%%%%%

We begin the proof of Theorem \ref{thm:1} by noting a few key observations. Based on the above notations and the assumptions we made earlier, the optimality gap of the parameter server's model at period $k$, i.e. $\bE \norm{\bx_{k+1} - \bx^*}^2$, can be decomposed as stated in the following lemma.
%%%%%%
\begin{lemma} \label{lemma:1}
Consider any period $k = 0,\cdots,K-1$ and the sequences $\{\bx_{k+1}, \widehat{\bx}_{k+1}, \xbar_{k,\tau}\}$ generated by the \texttt{FedPAQ} method in Algorithm \ref{alg:update}. If Assumption \ref{assump-Q} holds, then
%%%%%%
\begin{align}
    \bE \norm{\bx_{k+1} - \bx^*}^2 
    &=
    \bE \norm{\bx_{k+1} - \widehat{\bx}_{k+1}}^2
    +
    \bE \norm{\widehat{\bx}_{k+1} - \xbar_{k, \tau}}^2
    +
    \bE \norm{\xbar_{k,\tau} - \bx^*}^2,
    \label{eq:11}
\end{align}
%%%%%%
where the expectation is with respect to all  sources of randomness.
\end{lemma}

\begin{proof}
See Section \ref{subsec:lemma1-proof}.
\end{proof}

In the following three lemmas, we characterize each of the terms in the right-hand side (RHS) of \eqref{eq:11}.

\begin{lemma} \label{lemma:2}
Consider the sequence of local updates in the  \texttt{FedPAQ} method in Algorithm \ref{alg:update} and let Assumptions \ref{assump-smooth}, \ref{assump-gr} and \ref{assump-convex} hold. The optimality gap for the average model at the end of period $k$, i.e. $\xbar_{k,\tau}$, relates to that of the initial model of the $k$-th period $\bx_{k}$ as follows:
%%%%%%
\begin{align}
    \bE \norm{\xbar_{k,\tau} - \bx^*}^2
    &\leq
    \left(1 + n \eta_k^2 \right) \left(1 - \mu \eta_k \right)^{\tau} \bE \norm{ \bx_{k} - \bx^*}^2 \\
    & \quad +
    \tau (\tau-1)^2 L^2 \frac{\sigma^2}{n} e \eta_k^2 
    +
    \tau^2 \frac{\sigma^2}{n} \eta_k^2 \\
    & \quad +
    \tau^2 (\tau-1) L^2 \sigma^2 e \eta_k^4,
\end{align}
%%%%%%
for the stepsize $\eta_k \leq \min \{ \nicefrac{\mu}{L^2}, \nicefrac{1}{L \tau}\}$.
\end{lemma}

\begin{proof}
See Section \ref{subsec:lemma2-proof}.
\end{proof}

\begin{lemma} \label{lemma:3}
For the proposed \texttt{FedPAQ} method in Algorithm \ref{alg:update} with stepsize $\eta_k \leq \min \{ \nicefrac{\mu}{L^2}, \nicefrac{1}{L \tau}\}$ and under Assumptions \ref{assump-Q}, \ref{assump-smooth}, \ref{assump-gr} and \ref{assump-convex}, we have
%%%%%%
\begin{align}
    \bE \norm{\widehat{\bx}_{k+1} - \xbar_{k, \tau}}^2
    &\leq
    2 \frac{q}{n} \tau^2 L^2 \eta_k^2  \bE \norm{\bx_k - \bx^*}^2
    +
    2q \tau^2 \frac{\sigma^2}{n} \eta_k^2
    +
    2 q (\tau-1)\tau^2  L^2 \frac{\sigma^2}{n} e \eta_k^4,
\end{align}
%%%%%%
where $\widehat{\bx}_{k+1}$ and $\xbar_{k, \tau}$ are defined in \eqref{eq:notation}.
\end{lemma}

\begin{proof}
See Section \ref{subsec:lemma3-proof}.
\end{proof}

\begin{lemma}\label{lemma:4}
For the proposed \texttt{FedPAQ} method in Algorithm \ref{alg:update} with stepsize $\eta_k \leq \min \{ \nicefrac{\mu}{L^2}, \nicefrac{1}{L \tau}\}$ and under Assumptions \ref{assump-Q}--\ref{assump-convex}, we have
%%%%%%
\begin{align}
    \bE \norm{\bx_{k+1} - \widehat{\bx}_{k+1}}^2
    &\leq
    \frac{n-r}{r(n-1)}  8(1+q) \Bigg\{
    \tau^2 L^2 \eta_k^2  \bE \norm{\bx_k - \bx^*}^2 
    +
    \tau^2 \sigma^2 \eta_k^2
    +
    (\tau-1)\tau^2  L^2 \sigma^2 e \eta_k^4 \Bigg\},
\end{align}
%%%%%%
where $r$ denotes the number of nodes contributing in each period of the \texttt{FedPAQ} method.
\end{lemma}

\begin{proof}
See Section \ref{subsec:lemma4-proof}.
\end{proof}

Now that we have established the main building modules for proving Theorem \ref{thm:1}, let us proceed with the proof by putting together the results in Lemmas \ref{lemma:1}--\ref{lemma:4}. That is, 
%%%%%%
\begin{align}
    \bE \norm{\bx_{k+1} - \bx^*}^2 
    &\leq
    \bE \norm{ \bx_{k} - \bx^*}^2
    \left( \left(1 + n \eta_k^2 \right) \left(1 - \mu \eta_k \right)^{\tau}
    +
    2  L^2 \tau^2 \eta_k^2 \left(\frac{q}{n}  +  \frac{n-r}{r(n-1)} 4(1+q)  \right)
    \right) \\
    & \quad
    +
    \left(1 + 2q +  8  (1+q)\frac{n(n-r)}{r(n-1)} \right)  \frac{\sigma^2}{n} \tau^2 \eta_k^2 \\
    & \quad +
    L^2 \frac{\sigma^2}{n} e  \tau (\tau-1)^2 \eta_k^2\\
    & \quad +
    \left(n + 2q + 8(1+q)\frac{n(n-r)}{r(n-1)}   
    \right)  L^2 \frac{\sigma^2}{n} e (\tau-1)\tau^2  \eta_k^4 \label{eq:12}
\end{align}
%%%%%%
Let us set the following notations:
%%%%%%
\begin{align}
    \delta_k 
    & \coloneqq 
    \bE \norm{\bx_{k} - \bx^*}^2, \\
    C_0
    & \coloneqq
    \left(1 + n \eta_k^2 \right) \left(1 - \mu \eta_k \right)^{\tau}
    +
    2  L^2 \tau^2 \eta_k^2 \left(\frac{q}{n}  +  \frac{n-r}{r(n-1)} 4(1+q)  \right), \\
    C_1
    & \coloneqq 
    \frac{16}{\mu^2} \left(1 + 2q +  8  (1+q)\frac{n(n-r)}{r(n-1)} \right)  \frac{\sigma^2}{n}, \\
    C_2
    & \coloneqq 
    \frac{16}{\mu^2} L^2 \frac{\sigma^2}{n} e,  \\
    C_3
    & \coloneqq
    \frac{256}{\mu^4} \left(n + 2q + 8(1+q)\frac{n(n-r)}{r(n-1)}   
    \right)  L^2 \frac{\sigma^2}{n} e.
\end{align}
%%%%%%
Consider $C_0$, the coefficient of $\bE \norm{ \bx_{k} - \bx^*}^2$ in \eqref{eq:12}. One can show that if the condition in \eqref{eq:k0} in Theorem \ref{thm:1} is satisfied, then we have $C_0 \leq 1 - \frac{1}{2} \mu \tau \eta_k$ (See Section \ref{sec:stepsize-convex}). Therefore, for each period $k \geq k_0$ we have
%%%%%%
\begin{align}
    \delta_{k+1}
    \leq
    \left( 1 - \frac{1}{2} \mu \tau \eta_k \right) \delta_k
    +
    \frac{\mu^2}{16} C_1 \tau^2 \eta_k^2
    +
    \frac{\mu^2}{16} C_2 \tau (\tau-1)^2 \eta_k^2
    +
    \frac{\mu^4}{256} C_3 (\tau-1)\tau^2 \eta_k^4. \label{eq:decay}
\end{align}
%%%%%%
Now, we substitute the stepsize $\eta_k = \nicefrac{4 \mu^{-1}}{k \tau + 1}$ in \eqref{eq:decay} which yields
%%%%%%
\begin{align}
    \delta_{k+1}
    \leq
    \left( 1 - \frac{2}{k + 1/\tau} \right) \delta_k
    +
    C_1 \frac{1}{(k + 1/\tau)^2} 
    +
    C_2 \frac{(\tau - 1)^2}{\tau} \frac{1}{(k + 1/\tau)^2}
    +
    C_3 \frac{\tau - 1}{\tau^2} \frac{1}{(k + 1/\tau)^4} . 
\end{align}
%%%%%%
In Lemma \ref{lemma:delta-k}, we show the convergence analysis of such sequence. In particular, we take $k_1 = 1/\tau$, $a=C_1 + C_2 {(\tau - 1)^2}/{\tau}$ and $b= C_3 ({\tau - 1})/{\tau^2}$ in Lemma \ref{lemma:delta-k} and conclude for any  $k \geq k_0$ that
%%%%%%
\begin{align}
    \delta_{k}
    \leq
    \frac{(k_0 + 1/\tau)^2}{(k + 1/\tau)^2}  \delta_{k_0}
    +
    C_1 \frac{1}{k + 1/\tau} 
    +
    C_2 \frac{(\tau - 1)^2}{\tau} \frac{1}{k + 1/\tau}
    +
    C_3 \frac{\tau - 1}{\tau^2} \frac{1}{(k + 1/\tau)^2} . \label{eq:decay2}
\end{align}
%%%%%%
Finally, rearranging the terms in \eqref{eq:decay2} yields the desired result in Theorem \ref{thm:1}, that is
%%%%%%
\begin{align}
    \bE \norm{ \bx_{k} - \bx^*}^2
    &\leq
    \frac{(k_0 \tau + 1)^2}{(k \tau + 1)^2} \bE \norm{ \bx_{k_0} - \bx^*}^2 
    +
    C_1 \frac{\tau}{k \tau + 1}
    +
    C_2 \frac{(\tau - 1)^2}{k \tau + 1}
    +
    C_3 \frac{\tau - 1}{(k \tau + 1)^2} .
\end{align}
%%%%%%

\subsection{Proof of Lemma \ref{lemma:1}} \label{subsec:lemma1-proof}

Let $\ccalF_{k,t}$ denote the history of all sources of randomness by the $t$-th iteration in period $k$. The following expectation arguments are conditional on the history $\ccalF_{k,\tau}$ which we remove in our notations for simplicity. Since the random subset of nodes $\ccalS_k$ is uniformly picked from the set of all the nodes $[n]$, we can write
%%%%%%
\begin{align}
    \bE_{\ccalS_k} \bx_{k+1} 
    &=
    \bx_{k} + \bE_{\ccalS_k} \frac{1}{r} \sum_{i \in \ccalS_k} Q \left(\bx^{(i)}_{k, \tau} - \bx_{k} \right) \\
    &=
    \bx_{k} + \sum_{\substack{\ccalS \subseteq [n] \\ |\ccalS|=r}} \Pr{\ccalS_k = \ccalS} \frac{1}{r} \sum_{i \in \ccalS_k} Q \left(\bx^{(i)}_{k, \tau} - \bx_{k} \right)\\
    &=
    \bx_{k} + \frac{1}{{n \choose r}} \frac{1}{r} {n-1 \choose r-1} \sum_{i \in [n]} Q \left(\bx^{(i)}_{k,\tau} - \bx_{k} \right)\\
    &=
    \bx_{k} + \frac{1}{n} \sum_{i \in [n]} Q \left(\bx^{(i)}_{k,\tau} - \bx_{k} \right) \\
    &=
    \widehat{\bx}_{k+1}. \label{eq:S}
\end{align}
%%%%%%
Moreover, the quantizer $Q(\cdot)$ is unbiased according to Assumption \ref{assump-Q}, which yields
%%%%%%
\begin{align}
    \bE_Q \, \widehat{\bx}_{k+1} 
    &=
    \bx_{k} + \frac{1}{n} \sum_{i \in [n]} \bE_Q \, Q \left(\bx^{(i)}_{k,\tau} - \bx_{k} \right) \\
    &=
    \frac{1}{n} \sum_{i \in [n]} \bx^{(i)}_{k,\tau}\\
    &=
    \xbar_{k,\tau}. \label{eq:Q}
\end{align}
%%%%%%
Finally, since the two randomnesses induced by the quantization and random sampling are independent, together with \eqref{eq:S} and \eqref{eq:Q} we can conclude that:
%%%%%%
\begin{align}
    \bE \norm{\bx_{k+1} - \bx^*}^2 
    &= 
    \bE \norm{\bx_{k+1} - \widehat{\bx}_{k+1} + \widehat{\bx}_{k+1} - \xbar_{k, \tau} + \xbar_{k,\tau} - \bx^*}^2 \\
    &=
    \bE \norm{\bx_{k+1} - \widehat{\bx}_{k+1}}^2
    +
    \bE \norm{\widehat{\bx}_{k+1} - \xbar_{k, \tau}}^2
    +
    \bE \norm{\xbar_{k,\tau} - \bx^*}^2.
\end{align}
%%%%%%

\subsection{Proof of Lemma \ref{lemma:2}}\label{subsec:lemma2-proof}

According to update rule  in Algorithm \ref{alg:update}, local model at node $i$ for each iteration $t=0,\cdots,\tau-1$ of period $k=0,\cdots,K-1$ can be written as follows:
%%%%%%
\begin{align}
    \bx^{(i)}_{k,t+1} 
    &=
    \bx^{(i)}_{k, t} - \eta_k \tNab f_i \left(\bx^{(i)}_{k, t}\right),
\end{align}
%%%%%%
where all the nodes start the period with the initial model $\bx^{(i)}_{k,0} = \bx_{k}$. In parallel, let us define another sequence of updates as follows:
%%%%%%
\begin{align}
    \beta_{k,t+1} 
    &=
    \beta_{k,t} - \eta_k \gr f \left(\beta_{k,t}\right), \label{eq:beta}
\end{align}
%%%%%%
also starting with $\beta_{k,0} = \bx_{k}$. The auxiliary sequence $\{\beta_{k,t}\}$ represents Gradient Descent updates over the global loss function $f$ while $\bx^{(i)}_{k,t}$ captures the sequence of SGD updates on each local node. However, both sequences are initialized with $\bx_{k}$ at the beginning of each period $k$. To evaluate the deviation $\norm{\xbar_{k,\tau} - \bx^*}^2$, we link the two sequences. In particular, let us define the following notations for each $k=0,\cdots,K-1$ and $t=0,\cdots,\tau-1$:
%%%%%%
\begin{align}
    \bbe_{k, t}
    &=
    \frac{1}{n} \sum_{i \in [n]} \tNab f_i \left(\bx^{(i)}_{k, t}\right) 
    -
    \gr f \left(\beta_{k,t}\right). 
\end{align}
%%%%%%
One can easily observe that $\bE \bbe_{k, 0} = 0$ as $\bx^{(i)}_{k, 0}=\beta_{k,0}=\bx_{k}$ and $\tNab f_i$ is unbiased for $\gr f$. However, $\bE \bbe_{k, t} \neq 0$ for $t \geq 1$. In other words, $\frac{1}{n} \sum_{i \in [n]} \tNab f_i (\bx^{(i)}_{k, t})$ is not unbiased for $\gr f(\beta_{k,t})$. We also define $\bbe_{k} = \bbe_{k, 0} + \cdots + \bbe_{k, \tau-1}$ and $\bbg_k = \gr f(\beta_{k,0}) + \cdots + \gr f(\beta_{k,\tau-1})$. Now, the average model obtained at the end of period $k$ can be written as
%%%%%%
\begin{align}
    \xbar_{k,\tau} 
    &=
    \frac{1}{n} \sum_{i \in [n]} \bx^{(i)}_{k,\tau}\\
    &=
    \bx_{k} - \eta_k  \left( \frac{1}{n} \sum_{i \in [n]} \tNab f_i \left( \bx^{(i)}_{k,0} \right) +   \cdots 
    +
    \frac{1}{n} \sum_{i \in [n]} \tNab f_i \left( \bx^{(i)}_{k,\tau-1} \right)\right) \\
    &=
    \bx_{k} - \eta_k  \left( \bbg_k + \bbe_k \right).
\end{align}
%%%%%%
Therefore, the optimality gap for the averaged model can be written as
%%%%%%
\begin{align}
    \bE \norm{\xbar_{k,\tau} - \bx^*}^2
    &=
    \bE \norm{\bx_{k} - \eta_k \bbg_k - \bx^*}^2
    -2 \eta_k \bE \left\langle \bx_{k}- \eta_k \bbg_k - \bx^* , \bbe_k \right\rangle + \eta_k^2 \bE \norm{\bbe_k}^2 \\
    &\leq
    \bE \norm{\bx_{k} - \eta_k \bbg_k - \bx^*}^2 \\
    & \quad + 
    n \eta_k^2 \bE \norm{\bx_{k} - \eta_k \bbg_k - \bx^*}^2 + \frac{1}{n}\norm{\bE \bbe_k}^2 \\
    & \quad + 
    \eta_k^2 \bE \norm{\bbe_k}^2 \\
    &=
    \left(1 + n \eta_k^2 \right) \bE \norm{\bx_{k} - \eta_k \bbg_k - \bx^*}^2 + \frac{1}{n}\norm{\bE \bbe_k}^2 
    + 
    \eta_k^2 \bE \norm{\bbe_k}^2, \label{eq:5}
\end{align}
%%%%%%
where we used the inequality $-2 \langle \bba,\bbb \rangle \leq \alpha \norm{\bba}^2 + \alpha^{-1}\norm{\bbb}^2$ for any two vectors $\bba,\bbb$ and scalar $\alpha > 0$. In the following, we bound each of the three terms in the RHS of \eqref{eq:5}. First, consider the term $\norm{\bx_{k} - \eta_k \bbg_k - \bx^*}^2$ and recall the auxiliary sequence $\{\beta_{k,t}\}$ defined in \eqref{eq:beta}. For every $t$ and $k$ we have
%%%%%%
\begin{align}
    \norm{ \beta_{k,t + 1} - \bx^*}^2 
    &=
    \norm{ \beta_{k,t} - \eta_k \gr f(\beta_{k,t}) -  \bx^*}^2 \\
    &=
    \norm{ \beta_{k,t} - \bx^*}^2 - 2 \eta_k \left \langle \beta_{k,t} - \bx^*, \gr f(\beta_{k,t}) \right \rangle + \eta_k^2 \norm{\gr f(\beta_{k,t})}^2 \\
    & \leq
    \left(1 - 2 \mu \eta_k + L^2 \eta_k^2 \right) \norm{ \beta_{k,t} - \bx^*}^2 \\
    & \leq
    (1 - \mu \eta_k) \norm{ \beta_{k,t} - \bx^*}^2. \label{eq:beta}
\end{align}
%%%%%%
In the above derivations, we used the facts that $f$ is $\mu$-strongly convex and its gradient is $L$-Lipschitz (Assumptions \ref{assump-smooth} and \ref{assump-convex}). The stepsize is also picked such that $\eta_k \leq \nicefrac{\mu}{L^2}$. Now, conditioned on the history $\ccalF_{k,0}$ and using \eqref{eq:beta} we have
%%%%%%
\begin{align}
    \norm{\bx_{k} - \eta_k \bbg_k - \bx^*}^2 
    &=
    \norm{ \beta_{k,\tau} - \bx^*}^2 \\
    &\leq
    \left(1 - \mu \eta_k \right)^{\tau} \norm{ \beta_{k,0} - \bx^*}^2 \\
    &=
    \left(1 - \mu \eta_k \right)^{\tau} \norm{ \bx_{k} - \bx^*}^2.
\end{align}
%%%%%%
Secondly, consider the term $\norm{\bE \bbe_k}^2$ in \eqref{eq:5}. By definition, we have $\bE \bbe_{k} = \bE \bbe_{k,1} + \cdots + \bE \bbe_{k,\tau-1}$ and hence $\norm{\bE \bbe_{k}}^2 \leq (\tau-1)\norm{\bE \bbe_{k,1}}^2  + \cdots + (\tau-1)\norm{\bE \bbe_{k,\tau-1}}^2$. The first term $\norm{\bE \bbe_{k,1}}^2$ can be bounded using Assumptions \ref{assump-smooth} and \ref{assump-gr} as follows:
%%%%%%
\begin{align}
    \norm{\bE \bbe_{k,1}}^2 
    &=
    \norm{\frac{1}{n} \sum_{i \in [n]} \bE \tNab f_i \left(\bx^{(i)}_{k,1}\right) 
    -
    \gr f \left(\beta_{k,1}\right)}^2 \\
    &=
    \norm{\frac{1}{n} \sum_{i \in [n]} \bE \gr f \left(\bx^{(i)}_{k,1}\right) 
    -
    \gr f \left(\beta_{k,1}\right)}^2 \\
    &\leq
    \frac{1}{n} \sum_{i \in [n]} \bE \norm{ \gr f \left(\bx^{(i)}_{k,1}\right) 
    -
    \gr f \left(\beta_{k,1}\right)}^2 \\
    &\leq
    \frac{1}{n} L^2 \sum_{i \in [n]} \bE \norm{\bx^{(i)}_{k,1} - \beta_{k,1}}^2 \\
    &=
    \frac{1}{n} L^2 \sum_{i \in [n]} \bE \norm{ \left( \bx^{(i)}_{k, 0} - \eta_k \tNab f_i \left(\bx^{(i)}_{k, 0}\right)  \right) -  \left( \beta_{k, 0} - \eta_k \gr f \left(\beta_{k, 0}\right)  \right)}^2 \\
    &=
    \frac{1}{n} L^2 \eta_k^2 \sum_{i \in [n]} \bE \norm{ \tNab f_i\left(\bx_k\right) - \gr f \left(\bx_k\right)}^2 \\
    &\leq
    L^2 \sigma^2 \eta_k^2.
\end{align}
%%%%%%
In general, for each $t=1\cdots,\tau-1$ we can write
%%%%%%
\begin{align}
    \norm{\bE \bbe_{k,t}}^2 
    &=
    \norm{\frac{1}{n} \sum_{i \in [n]} \bE \tNab f_i \left(\bx^{(i)}_{k,t}\right) 
    -
    \gr f \left(\beta_{k,t}\right)}^2 \\
    &=
    \norm{\frac{1}{n} \sum_{i \in [n]} \bE \gr f \left(\bx^{(i)}_{k,t}\right) 
    -
    \gr f \left(\beta_{k,t}\right)}^2 \\
    &\leq
    \frac{1}{n} \sum_{i \in [n]} \bE \norm{ \gr f \left(\bx^{(i)}_{k,t}\right) 
    -
    \gr f \left(\beta_{k,t}\right)}^2 \\
    &\leq
    \frac{1}{n} L^2 \sum_{i \in [n]} \bE \norm{\bx^{(i)}_{k,t} - \beta_{k,t}}^2. 
\end{align}
%%%%%%
Let us denote $a_{k,t} \coloneqq \frac{1}{n}\sum_{i \in [n]} \bE \norm{\bx^{(i)}_{k,t} - \beta_{k,t}}^2$. In the following, we will derive a recursive bound on $a_t$. That is,
%%%%%%
\begin{align}
    a_{k,t} 
    &=
    \frac{1}{n}\sum_{i \in [n]} \bE \norm{\bx^{(i)}_{k,t} - \beta_{k,t}}^2 \\
    &=
    \frac{1}{n}\sum_{i \in [n]} \bE \Bigg\|\left( \bx^{(i)}_{k, 0} - \eta_k \tNab f_i \left(\bx^{(i)}_{k, 0}\right) - \cdots - \eta_k \tNab f_i \left(\bx^{(i)}_{k, t-1}\right) \right) \\
    &\quad \quad \quad \quad \quad \quad \quad-
    \left( \beta_{k, 0} - \eta_k \gr f \left(\beta_{k, 0}\right) - \cdots -  \eta_k \gr f \left(\beta_{k, t-1}\right) \right)\Bigg\|^2 \\
    &=
    \frac{1}{n} \eta_k^2 \sum_{i \in [n]} \bE \Bigg\|\tNab f_i \left(\bx^{(i)}_{k, 0}\right) - \gr f \left(\beta_{k, 0}\right) + \cdots + \tNab f_i \left(\bx^{(i)}_{k, t-1}\right) - \gr f \left(\beta_{k, t-1}\right)\Bigg\|^2 \\
    &\leq
    \eta_k^2 \sigma^2
    +
    \frac{1}{n} \eta_k^2 \sum_{i \in [n]} \bE \Bigg\|\tNab f_i \left(\bx^{(i)}_{k, 1}\right) - \gr f \left(\beta_{k, 1}\right) + \cdots + \tNab f_i \left(\bx^{(i)}_{k, t-1}\right) - \gr f \left(\beta_{k, t-1}\right)\Bigg\|^2 \\
    &\leq
    \eta_k^2 \sigma^2
    +
    \frac{1}{n} \eta_k^2 \sum_{i \in [n]} \bE \Bigg\|\tNab f_i \left(\bx^{(i)}_{k, 1}\right) - \gr f \left(\bx^{(i)}_{k, 1}\right)
    +
    \gr f \left(\bx^{(i)}_{k, 1}\right) - \gr f \left(\beta_{k, 1}\right) \\
    &\quad\quad\quad\quad\quad\quad +
    \cdots +
    \tNab f_i \left(\bx^{(i)}_{k, t-1}\right) - \gr f \left(\bx^{(i)}_{k, t-1}\right)
    +
    \gr f \left(\bx^{(i)}_{k, t-1}\right) - \gr f \left(\beta_{k, t-1}\right)\Bigg\|^2 \\
    &\leq
    t \eta_k^2 \sigma^2
    +
    \frac{1}{n} \eta_k^2 \sum_{i \in [n]} \bE \Bigg\|\gr f \left(\bx^{(i)}_{k, 1}\right) - \gr f \left(\beta_{k, 1}\right) 
    + \cdots
    \gr f \left(\bx^{(i)}_{k, t-1}\right) - \gr f \left(\beta_{k, t-1}\right)\Bigg\|^2 \\
    &\leq
    t \eta_k^2 \sigma^2 
    +
    (t-1) L^2 \eta_k^2  \frac{1}{n} \sum_{i \in [n]} \bE \Bigg\|\bx^{(i)}_{k, 1} - \beta_{k, 1}\Bigg\|^2 
    + \cdots
    (t-1) L^2 \eta_k^2 \frac{1}{n} \sum_{i \in [n]} \bE \Bigg\|
    \bx^{(i)}_{k, t-1} - \beta_{k, t-1}\Bigg\|^2 \\
    &=
    t \eta_k^2 \sigma^2 
    +
    (t-1) L^2 \eta_k^2 \left( a_{k,1} + \cdots + a_{k,t-1} \right) \\
    &\leq
    \tau \eta_k^2 \sigma^2 
    +
    \tau L^2 \eta_k^2 \left( a_{k,1} + \cdots + a_{k,t-1} \right).
\end{align}
%%%%%%
Therefore, for the sequence $\{a_{k,1},\cdots,a_{k,\tau-1}\}$ we have shown that 
%%%%%%
\begin{align}
    a_{k,t} 
    &\leq
    \tau \eta_k^2 \sigma^2 
    +
    \tau L^2 \eta_k^2 \left( a_{k,1} + \cdots + a_{k,t-1} \right),
\end{align}
%%%%%%
where $a_{k,1} \leq \sigma^2 \eta_k^2$. We can show by induction, that such sequence satisfies the following inequality:
%%%%%%
\begin{align}
    a_{k,t}
    &\leq
    \tau \eta_k^2 \sigma^2 \left( 1 + \tau L^2 \eta_k^2 \right)^{t-1}. \label{eq:recursive}
\end{align}
%%%%%%
See Section \ref{subsec:proof-recursive} for the detailed proof. Therefore, we have
%%%%%%
\begin{align}
    \norm{\bE \bbe_{k}}^2 
    &\leq 
    (\tau-1)\norm{\bE \bbe_{k,1}}^2  + \cdots + (\tau-1)\norm{\bE \bbe_{k,\tau-1}}^2 \\
    &\leq
    (\tau-1)L^2 \left( a_1 + \cdots + a_{\tau-1} \right)\\
    &\leq
    \tau (\tau-1)^2 L^2 \sigma^2 \eta_k^2 \left( 1 + \tau L^2 \eta_k^2 \right)^{\tau}.
\end{align}
%%%%%%
Now, we use the inequality $1+x \leq e^x$ and conclude that
%%%%%%
\begin{align}
    \norm{\bE \bbe_{k}}^2
    \leq 
    \tau (\tau-1)^2 L^2 \sigma^2 \eta_k^2 e^{\tau^2 L^2 \eta_k^2}.
\end{align}
%%%%%%
Therefore, if $\tau^2 L^2 \eta_k^2 \leq 1$, we have
%%%%%%
\begin{align}
    \norm{\bE \bbe_{k}}^2 
    &\leq 
    \tau (\tau-1)^2 L^2 \sigma^2 e \eta_k^2.
\end{align}
%%%%%%
Finally, we bound the third term in \eqref{eq:5}, that is $\bE \norm{\bbe_k}^2$. Using the definition, we know that $\bE \norm{\bbe_k}^2 \leq \tau \bE \norm{\bbe_{k,0}}^2 + \cdots + \tau \bE \norm{\bbe_{k,\tau-1}}^2$. Firstly, note that
%%%%%%
\begin{align}
    \bE \norm{\bbe_{k,0}}^2 
    &=
    \bE \norm{\frac{1}{n} \sum_{i \in [n]}\tNab f_i \left(\bx^{(i)}_{k,0}\right) 
    -
    \gr f \left(\beta_{k,0}\right)}^2 \\
    &=
    \bE \norm{\frac{1}{n} \sum_{i \in [n]}\tNab f_i \left(\bx_{k}\right) 
    -
    \gr f \left(\bx_{k}\right)}^2 \\
    &\leq
    \frac{\sigma^2}{n}.
\end{align}
%%%%%%
For each $t=1,\cdots,\tau-1$ we have
%%%%%%
\begin{align}
    \bE \norm{\bbe_{k,t}}^2 
    &=
    \bE \norm{\frac{1}{n} \sum_{i \in [n]}\tNab f_i \left(\bx^{(i)}_{k,t}\right) 
    -
    \gr f \left(\beta_{k,t}\right)}^2 \\
    &=
    \bE \norm{\frac{1}{n} \sum_{i \in [n]}\tNab f_i \left(\bx^{(i)}_{k,t}\right) 
    -
    \gr f \left(\bx^{(i)}_{k,t}\right)
    +
    \frac{1}{n} \sum_{i \in [n]}\gr f \left(\bx^{(i)}_{k,t}\right) 
    -
    \gr f \left(\beta_{k,t}\right)}^2 \\
    &\leq
    \frac{\sigma^2}{n}
    +
    L^2 \frac{1}{n}\sum_{i \in [n]} \bE \norm{\bx^{(i)}_{k,t} - \beta_{k,t}}^2 \\
    &=
    \frac{\sigma^2}{n}
    +
    L^2 a_{k,t}.
\end{align}
%%%%%%
Summing over $t=0,1,\cdots,\tau-1$ results in the following
%%%%%%
\begin{align}
    \bE \norm{\bbe_k}^2 
    &\leq
    \tau \bE \norm{\bbe_{k,0}}^2 + \cdots + \tau \bE \norm{\bbe_{k,\tau-1}}^2 \\
    &\leq
    \tau^2 \frac{\sigma^2}{n} 
    +
    \tau L^2 \left( a_1 + \cdots + a_{\tau-1} \right)\\
    &\leq
    \tau^2 \frac{\sigma^2}{n}
    +
    \tau^2 (\tau-1) L^2 \sigma^2 \eta_k^2 \left( 1 + \tau L^2 \eta_k^2 \right)^{\tau} \\
    &\leq
    \tau^2 \frac{\sigma^2}{n}
    +
    \tau^2 (\tau-1) L^2 \sigma^2 e \eta_k^2. \label{eq:e_k}
\end{align}
%%%%%%
Now, we can put everything together and conclude Lemma \ref{lemma:2}, as follows
%%%%%%
\begin{align}
    \bE \norm{\xbar_{k,\tau} - \bx^*}^2
    &=
    \left(1 + n \eta_k^2 \right) \bE \norm{\bx_{k} - \eta_k \bbg_k - \bx^*}^2 + \frac{1}{m}\norm{\bE \bbe_k}^2 
    + 
    \eta_k^2 \bE \norm{\bbe_k}^2 \\
    &\leq
    \left(1 + n \eta_k^2 \right) \left(1 - \mu \eta_k \right)^{\tau} \bE \norm{ \bx_{k} - \bx^*}^2 \\
    & \quad +
    \tau (\tau-1)^2 L^2 \frac{\sigma^2}{n} e \eta_k^2 
    +
    \tau^2 \frac{\sigma^2}{n} \eta_k^2 \\
    & \quad +
    \tau^2 (\tau-1) L^2 \sigma^2 e \eta_k^4.
\end{align}
%%%%%%

\subsection{Proof of Lemma \ref{lemma:3}}\label{subsec:lemma3-proof}

According to the notations defined on \eqref{eq:notation}, we can write 

%%%%%%
\begin{align}
    \bE \norm{\widehat{\bx}_{k+1} - \xbar_{k, \tau}}^2
    &=
    \bE \norm{ \bx_{k} + \frac{1}{n} \sum_{i \in [n]} Q \left(\bx^{(i)}_{k,\tau} - \bx_{k} \right) - \frac{1}{n} \sum_{i \in [n]} \bx^{(i)}_{k,\tau}}^2 \\
    &=
    \bE \norm{\frac{1}{n} \sum_{i \in [n]} Q \left(\bx^{(i)}_{k,\tau} - \bx_{k} \right) - \left(\bx^{(i)}_{k,\tau} - \bx_{k} \right)}^2 \\
    &=
    \frac{1}{n^2} \sum_{i \in [n]}\bE \norm{Q \left(\bx^{(i)}_{k,\tau} - \bx_{k} \right) - \left(\bx^{(i)}_{k,\tau} - \bx_{k} \right)}^2\\
    &\leq
    q \frac{1}{n^2} \sum_{i \in [n]} \bE \norm{\bx^{(i)}_{k,\tau} - \bx_{k}}^2, \label{eq:8}
\end{align}
%%%%%%
where, we used Assumption \ref{assump-Q}. In particular, the last equality above follows from the fact that the random quatizer is unbiased and the quantizations are carried out independently in each iteration and each worker. Moreover, the last inequality in \eqref{eq:8} simply relates the variance of the quantization to its argument. 
Next, we bound $\bE \norm{\bx^{(i)}_{k,\tau} - \bx_{k}}^2$ for each worker $i \in [n]$. From the update rule in Algorithm \ref{alg:update} we have
%%%%%%
\begin{align}
    \bx^{(i)}_{k,\tau} 
    &=
    \bx_{k} - \eta_k  \left(  \tNab f_i \left( \bx^{(i)}_{k,0} \right) +   \cdots 
    +
    \tNab f_i \left( \bx^{(i)}_{k,\tau-1} \right)\right) \\
    &=
    \bx_{k} - \eta_k  \left( \bbg_k + \bbe^{(i)}_k \right),
\end{align}
%%%%%%
where we denote
%%%%%%
\begin{align}
    \bbe^{(i)}_k
    \coloneqq 
    \tNab f_i \left( \bx^{(i)}_{k,0} \right) 
    -
    \gr f \left( \beta_{k,0} \right)
    +
    \cdots 
    +
    \tNab f_i \left( \bx^{(i)}_{k,\tau-1} \right)
    -
    \gr f \left( \beta_{k,\tau-1} \right),
\end{align}
%%%%%%
and $\bbg_k = \gr f(\beta_{k,0}) + \cdots + \gr f(\beta_{k,\tau-1})$ as defined before. Using these notations we have
%%%%%%
\begin{align}
    \bE \norm{\bx^{(i)}_{k,\tau} - \bx_{k}}^2
    & =
    \eta_k^2 \bE \norm{\bbg_k + \bbe^{(i)}_k}^2\\
    &\leq
    2 \eta_k^2 \norm{\bbg_k}^2
    +
    2 \eta_k^2 \bE \norm{\bbe^{(i)}_k}^2. \label{eq:x-t}
\end{align}
%%%%%%
Let us first bound the first term in \eqref{eq:x-t}, i.e. $\norm{\bbg_k}^2$. That is,
%%%%%%
\begin{align}
    \norm{\bbg_k}^2
    &\leq
    \tau \norm{\gr f(\beta_{k,0})}^2
    +
    \cdots+
    \tau \norm{\gr f(\beta_{k,\tau-1})}^2\\
    &\overset{(a)}{\leq}
    \tau L^2 \left( \norm{\bx_k - \bx^*}^2 + \cdots + \left( 1 - \mu \eta_k \right)^{\tau-1}\norm{\bx_k - \bx^*}^2\right) \\
    &\leq
    \tau^2 L^2 \norm{\bx_k - \bx^*}^2, \label{eq:gk}
\end{align}
%%%%%%
where we used the smoothness of the loss function $f$ (Assumption \ref{assump-smooth}) and the result in \eqref{eq:beta} to derive inequality $(a)$. To bound the second term in \eqref{eq:x-t}, i.e. $\bE \norm{\bbe^{(i)}_k}^2$, we can employ our result in \eqref{eq:e_k} for the special case $n=1$. It yields that for $\eta_k \leq \nicefrac{1}{L \tau}$,
%%%%%%
\begin{align}
    \bE \norm{\bbe^{(i)}_k}^2
    &\leq
    \tau^2 \sigma^2
    +
    \tau^2 (\tau-1) L^2 \sigma^2 e \eta_k^2. \label{eq:ei}
\end{align}
%%%%%%
Plugging \eqref{eq:gk} and \eqref{eq:ei} in \eqref{eq:x-t} implies that
%%%%%%
\begin{align}
    \bE \norm{\bx^{(i)}_{k,\tau} - \bx_{k}}^2
    &\leq
    2  \tau^2 L^2 \eta_k^2  \bE \norm{\bx_k - \bx^*}^2
    +
    2 \tau^2 \sigma^2 \eta_k^2
    +
    2 (\tau-1)\tau^2  L^2 \sigma^2 e \eta_k^4, \label{eq:x-t_bound}
\end{align}
%%%%%%
which together with \eqref{eq:8} concludes Lemma \ref{lemma:3}:
%%%%%%
\begin{align}
    \bE \norm{\widehat{\bx}_{k+1} - \xbar_{k, \tau}}^2
    &\leq
    2 \frac{q}{n} \tau^2 L^2 \eta_k^2  \bE \norm{\bx_k - \bx^*}^2
    +
    2q \tau^2 \frac{\sigma^2}{n} \eta_k^2
    +
    2 q (\tau-1)\tau^2  L^2 \frac{\sigma^2}{n} e \eta_k^4.
\end{align}
%%%%%%

\subsection{Proof of Lemma \ref{lemma:4}}\label{subsec:lemma4-proof}

For each node $i \in [n]$ denote $\bbz^{(i)}_{k,\tau} = Q (\bx^{(i)}_{k,\tau} - \bx_{k})$ and $\zbar_{k, \tau} = \frac{1}{n} \sum_{i \in [n]} \bbz^{(i)}_{k, \tau}$. Then,
%%%%%%
\begin{align}
    \bE_{ \ccalS_{k}} \norm{\bx_{k+1} - \widehat{\bx}_{k+1}}^2
    &=
    \bE_{ \ccalS_{k}} \norm{\frac{1}{r} \sum_{i \in \ccalS_{k}} \bbz^{(i)}_{k, \tau} - \zbar_{k, \tau}}^2 \\
    &=
    \frac{1}{r^2} \bE_{ \ccalS_{k}} \norm{\sum_{i \in [n]} \mathbbm{1}\{i \in \ccalS_{k} \} \left( \bbz^{(i)}_{k, \tau} - \zbar_{k, \tau} \right)}^2 \\
    &=
    \frac{1}{r^2} \Bigg\{ \sum_{i \in [n]} \Pr{i \in \ccalS_{k}} \norm{ \bbz^{(i)}_{k, \tau} - \zbar_{k, \tau} }^2 \\
    &\quad +
    \sum_{i \neq j} \Pr{i,j \in \ccalS_{k}} \left\langle \bbz^{(i)}_{k, \tau} - \zbar_{k, \tau} , \bbz^{(j)}_{k, \tau} - \zbar_{k, \tau} \right\rangle
    \Bigg\} \\
    &=
    \frac{1}{nr} \sum_{i \in [n]} \norm{ \bbz^{(i)}_{k, \tau} - \zbar_{k, \tau} }^2 \\
    &\quad +
    \frac{r-1}{r n (n-1)}\sum_{i \neq j} \left\langle \bbz^{(i)}_{k, \tau} - \zbar_{k, \tau} , \bbz^{(j)}_{k, \tau} - \zbar_{k, \tau} \right\rangle \\
    &=
    \frac{1}{r(n-1)} \left( 1 - \frac{r}{n} \right) \sum_{i \in [n]} \norm{ \bbz^{(i)}_{k, \tau} - \zbar_{k, \tau} }^2, \label{eq:x-xhat}
\end{align}
%%%%%%
where we used the fact that $\norm{ \bbz^{(i)}_{k, \tau} - \zbar_{k, \tau} }^2 + \sum_{i \neq j} \left\langle \bbz^{(i)}_{k, \tau} - \zbar_{k, \tau} , \bbz^{(j)}_{k, \tau} - \zbar_{k, \tau} \right\rangle=0$. Further taking expectation with respect to the quantizer yields
%%%%%%
\begin{align}
    \sum_{i \in [n]} \bE_{Q} \norm{ \bbz^{(i)}_{k, \tau} - \zbar_{k, \tau} }^2 
    &\leq
    2 \sum_{i \in [n]} \bE_{Q} \norm{ \bbz^{(i)}_{k, \tau}}^2 
    +
    2 n \bE_{Q} \norm{ \zbar_{k, \tau} }^2 \\
    &\leq
    4 \sum_{i \in [n]} \bE_{Q} \norm{ \bbz^{(i)}_{k, \tau}}^2 \\
    &=
    4 \sum_{i \in [n]} \bE_{Q} \norm{Q \left(\bx^{(i)}_{k,\tau} - \bx_{k} \right)}^2 \\
    &\leq
    4(1+q) \sum_{i \in [n]} \norm{\bx^{(i)}_{k,\tau} - \bx_{k}}^2. \label{eq:z}
\end{align}
%%%%%%
In the above derivations, we used the fact that under Assumption \ref{assump-Q} and for any $\bx$ we have $\bE \norm{Q(\bx)}^2 \leq (1+q) \norm{\bx}^2$. Therefore, \eqref{eq:z} together with the equality derived in \eqref{eq:x-xhat} yields that
%%%%%%
\begin{align}
    \bE \norm{\bx_{k+1} - \widehat{\bx}_{k+1}}^2
    &\leq
    \frac{1}{r(n-1)} \left( 1 - \frac{r}{n} \right) 4(1+q) \sum_{i \in [n]} \bE \norm{\bx^{(i)}_{k,\tau} - \bx_{k}}^2. \label{eq:x-xhat2}
\end{align}
%%%%%%
Finally, we substitute the bound in \eqref{eq:x-t_bound} into \eqref{eq:x-xhat2} and conclude Lemma \ref{lemma:4} as follows:
%%%%%%
\begin{align}
    \bE \norm{\bx_{k+1} - \widehat{\bx}_{k+1}}^2
    &\leq
    \frac{n-r}{r(n-1)} 8(1+q) \left\{
     \tau^2 L^2 \eta^2  \bE \norm{\bx_k - \bx^*}^2  +
     \tau^2 \sigma^2 \eta^2
    +
     (\tau-1)\tau^2  L^2 \sigma^2 e \eta^4 \right\}.
\end{align}
%%%%%%

\subsection{Proof of Equation \eqref{eq:recursive}} \label{subsec:proof-recursive}
%\textcolor{red}{Ramtin: I suggest we either have a lemma for this equation/property or include the proof inside the text instead of having separate subsection.}
Let us fix the period $k$ and for simplicity of the notations in this proof, let us take $a_t = a_{k,t}$ and $\eta = \eta_k$. We showed that
%%%
$a_t \leq
    \tau \eta^2 \sigma^2 
    +
    \tau L^2 \eta^2 \left( a_1 + \cdots + a_{t-1} \right)$
%%%
for every $t=2,\cdots,\tau-1$ and also $a_1 \leq \eta^2 \sigma^2 $. For $t=1$, \eqref{eq:recursive} holds. Assume that \eqref{eq:recursive} holds also for $\{a_1,\cdots, a_{t-1}\}$. Now, for $a_t$ we have

%%%%%%
\begin{align}
    a_t 
    &\leq
    \tau \eta^2 \sigma^2 
    +
    \tau L^2 \eta^2 \left( a_1 + \cdots + a_{t-1} \right) \\
    &\leq
    \tau \eta^2 \sigma^2
    +
    \tau L^2 \eta^2 \sum_{i=0}^{t-2} \tau \eta^2 \sigma^2  \left( 1 + \tau L^2 \eta^2 \right)^i \\
    &=
    \tau \eta^2 \sigma^2
    +
    \tau \eta^2 \sigma^2 \cdot \tau L^2 \eta^2 \cdot \frac{\left( 1 + \tau L^2 \eta^2 \right)^{t-1} - 1}{\tau L^2 \eta^2}\\
    &=
    \tau \eta^2 \sigma^2  \left( 1 + \tau L^2 \eta^2 \right)^{t-1},
\end{align}
%%%%%%
as desired. Therefore, \eqref{eq:recursive} holds for every $t=1,\cdots,\tau-1$.

%\noindent\makebox[\linewidth]{\rule{\paperwidth}{0.4pt}}

\subsection{Discussion on stepsize $\eta_k$}\label{sec:stepsize-convex}

Here we show that for any $k \geq k_0$ we have $C_0 \leq 1 - \frac{1}{2} \mu \tau \eta_k$, where $k_0$ satisfies the condition in Theorem \ref{thm:1}, that is 
%%%%%%
\begin{align}
    k_0
    \geq
    4 \max \left\{ \frac{L}{\mu}, 4 \left( \frac{B_1}{\mu^2} + 1 \right), \frac{1}{\tau}, \frac{4n}{\mu^2 \tau} \right\}.
\end{align}
%%%%%%
First note that this condition on $k_0$ implies the following conditions on the stepsize $\eta_k = \nicefrac{4 \mu^{-1}}{k\tau+1}$ for $k \geq k_0$:
%%%%%%
\begin{align}
    \eta_k \tau 
    \leq 
    \min \left\{ \frac{1}{L}, \frac{\mu}{4 \left( \mu^2 + B_1 \right)} \right\}, 
    \quad \text{ and } \quad
    \eta_k 
    \leq 
    \min \left\{ \frac{\mu}{L^2}, \frac{\mu}{4n} \right\},
\end{align}
%%%%%%
Now consider the term $(1 - \mu \eta_k )^{\tau}$ in $C_0$. We have
%%%%%%
\begin{align}
    \left(1 - \mu \eta_k \right)^{\tau}
    &=
    \left(1 - \frac{\mu \tau \eta_k}{\tau}  \right)^{\tau} \\
    &\leq
    e^{-\mu \tau \eta_k} \\
    &\leq
    1 - \mu \tau \eta_k + \mu^2 \tau^2 \eta_k^2,
\end{align}
%%%%%%
where the first inequality follows from the assumption $\eta_k \leq \nicefrac{1}{\mu}$ and the second inequality uses the fact that $e^x \leq 1 + x + x^2$ for $x \leq 0$. Therefore, 
%%%%%%
\begin{align}
    C_0
    &\leq
    \left(1 + n \eta_k^2 \right)
    \left( 1 - \mu \tau \eta_k + \mu^2 \tau^2 \eta_k^2 \right)
    +
    B_1 \tau^2 \eta_k^2 \\
    &=
    1 - \mu \tau \eta_k + \tau^2 \eta_k^2 (B_1 + \mu^2) + n\eta_k^2 \left( 1 - \mu \tau \eta_k + \mu^2 \tau^2 \eta_k^2 \right). 
\end{align}
%%%%%%
Note that from the assumption $\eta_k \leq 
\nicefrac{1}{L \tau}$ we have $0 \leq \mu \tau \eta_k \leq \nicefrac{\mu}{L} \leq 1$. This implies that $1 - \mu \tau \eta_k + \mu^2 \tau^2 \eta_k^2 \leq 1$. Hence, 
%%%%%%
\begin{align}
    C_0 
    \leq
    1 - \mu \tau \eta_k + \tau^2 \eta_k^2 (B_1 + \mu^2) + n\eta_k^2. \label{eq:C0}
\end{align}
%%%%%%
Now from the condition $\eta_k \tau \leq \nicefrac{\mu}{4(B_1 + \mu^2)}$ we have 
%%%%%%
\begin{align}
    \tau^2 \eta_k^2 (B_1 + \mu^2)
    \leq
    \frac{1}{4}\mu \tau \eta_k, \label{eq:C0-1}
\end{align}
%%%%%%
and from $\eta_k \leq \nicefrac{\mu}{4n}$ we have 
%%%%%%
\begin{align}
    n\eta_k^2
    \leq
    \frac{1}{4}\mu \tau \eta_k, \label{eq:C0-2}
\end{align}
%%%%%%
sine $\tau \geq 1$. Plugging \eqref{eq:C0-1} and \eqref{eq:C0-2} in \eqref{eq:C0} yields that for any $k \geq k_0$ we have $C_0 \leq 1 - \frac{1}{2}\mu \tau \eta_k$.

\subsection{Skipped lemmas and proofs}

\begin{lemma} \label{lemma:delta-k}
Let a non-negative sequence $\delta_k$ satisfy the following
%%%%%%
\begin{align}
    \delta_{k+1}
    \leq
    \left( 1 - \frac{2}{k+k_1} \right) \delta_k
    +
    \frac{a}{(k + k_1)^2}
    +
    \frac{b}{(k + k_1)^4},
\end{align}
%%%%%%
for every $k \geq k_0$, where $a,b,c,k_1$ are positive reals and $k_0$ is a positive integer. Then for every $k \geq k_0$ we have
%%%%%%
\begin{align}
    \delta_{k}
    \leq
    \frac{(k_0 + k_1)^2}{(k + k_1)^2}\delta_{k_0}
    +
    \frac{a}{k + k_1}
    +
    \frac{b}{(k + k_1)^2}. \label{eq:delta-k}
\end{align}
%%%%%%
\end{lemma}

\begin{proof}
We prove by induction on $k \geq k_0$. The claim in \eqref{eq:delta-k} is trivial for $k = k_0$. Let \eqref{eq:delta-k} hold for $s \geq k_0$, that is
%%%%%%
\begin{align}
    \delta_{s}
    \leq
    \frac{(k_0 + k_1)^2}{(s + k_1)^2}\delta_{k_0}
    +
    \frac{a}{s + k_1}
    +
    \frac{b}{(s + k_1)^2}.
\end{align}
%%%%%%
We can then write
%%%%%%
\begin{align}
    \delta_{s+1}
    &\leq
    \left( 1 - \frac{2}{s+k_1} \right) \delta_s
    +
    \frac{a}{s + k_1}
    +
    \frac{b}{(s + k_1)^2}\\
    &\leq
    \left( 1 - \frac{2}{s+k_1} \right) 
    \left( \frac{(k_0 + k_1)^2}{(s + k_1)^2}\delta_{k_0}
    +
    \frac{a}{s + k_1}
    +
    \frac{b}{(s + k_1)^2} \right)
    +
    \frac{a}{(s + k_1)^2}
    +
    \frac{b}{(s + k_1)^4} \\
    & =
    \frac{s+k_1-2}{(s+k_1)^3} (k_0 + k_1)^2 \delta_{k_0}
    +
    \frac{s + k_1 - 1}{(s + k_1)^2} a
    +
    \frac{(s + k_1 - 1)^2}{(s + k_1)^4} b. \label{eq:s+1}
\end{align}
%%%%%%
Now, take $s' = s + k_1$. We have for $s' \geq 1$ that
%%%%%%
\begin{align}
    \frac{s'-2}{s'^3} 
    \leq
    \frac{1}{(s'+1)^2},
    \quad  \quad
    \frac{s' - 1}{s'^2} 
    \leq
    \frac{1}{s'+1}, 
    \quad \quad
    \frac{(s' - 1)^2}{s'^4} 
    \leq
    \frac{1}{(s'+1)^2}. \label{eq:s'}
\end{align}
%%%%%%
Plugging \eqref{eq:s'} in \eqref{eq:s+1} yields that the claim in \eqref{eq:delta-k} holds for $s+1$ and hence for any $k \geq k_0$.
\end{proof}

%\textcolor{red}{This needs some work. I actually don't think we need such subsection ... Can we include a discussion in a remark environment and bring the math inside the main proof?}

%% file: 6-nonconvex.tex
We begin the proof of Theorem \ref{thm:2} by noting the following property for any smooth loss function.
%%%%%%
\begin{lemma}\label{lemma:noncnvx-1}

Consider the sequences of updates $\{\bx_{k+1}, \widehat{\bx}_{k+1}, \xbar_{k,\tau}\}$ generated by \texttt{FedPAQ} method in Algorithm \ref{alg:update}. If Assumptions \ref{assump-Q} and \ref{assump-smooth} hold, then
%%%%%%
\begin{align}
    \bE f(\bx_{k+1}) 
    &\leq
    \bE f(\xbar_{k,\tau})
    +
    \frac{L}{2} \bE \norm{\widehat{\bx}_{k+1} - \xbar_{k, \tau}}^2
    + 
    \frac{L}{2} \bE \norm{\widehat{\bx}_{k+1} 
    -
    \bx_{k+1}}^2, \label{eq:lemma5}
\end{align}
%%%%%%
for any period $k=0,\cdots,K-1$.
\end{lemma}
%%%%%%
\begin{proof}
See Section \ref{subsec:lemma-noncnvx-1-proof}.
\end{proof}

In the following three lemmas, we bound each of the three terms in the RHS of \eqref{eq:lemma5}. 

%%%%%%
\begin{lemma}\label{lemma:noncnvx-2}

Let Assumptions \ref{assump-smooth} and \ref{assump-gr} hold and consider the sequence of updates in \texttt{FedPAQ} method with stepsize $\eta$. Then, for every period $k=0,\cdots,K-1$ we have
%%%%%%
\begin{align}
    \bE f(\xbar_{k,\tau})
    &\leq
    \bE f (\bx_{k}) 
    -
    \frac{1}{2} \eta \sum_{t=0}^{\tau-1} \bE \norm{\gr f (\xbar_{k, t})}^2 \\
    &\quad -
    \eta \left( \frac{1}{2n} - \frac{1}{2n} L \eta - \frac{1}{n} L^2 \tau(\tau-1) \eta^2 \right) \sum_{t=0}^{\tau-1} \sum_{i \in [n]} \bE \norm{\gr f \left( \bbx^{(i)}_{k,t} \right)}^2 \\
    &\quad +
    \eta^2 \frac{L}{2} \frac{\sigma^2}{n} \tau + \eta^3 \frac{\sigma^2}{n} (n+1) \frac{\tau(\tau-1)}{2} L^2.
\end{align}
%%%%%%
\end{lemma}
%%%%%%
\begin{proof}
See Section \ref{subsec:lemma-noncnvx-2-proof}.
\end{proof}

%%%%%%
\begin{lemma}\label{lemma:noncnvx-3}

If Assumptions \ref{assump-Q} and \ref{assump-gr} hold, then for sequences $\{\widehat{\bx}_{k+1} , \xbar_{k, \tau}\}$ defined in \eqref{eq:notation} we have
%%%%%%
\begin{align}
    \bE \norm{\widehat{\bx}_{k+1} - \xbar_{k, \tau}}^2
    \leq
    q \frac{\sigma^2}{n} \tau \eta^2 
    +
    q \frac{1}{n^2} \tau \eta^2 \sum_{i \in [n]} \sum_{t=0}^{\tau-1} \norm{\gr f \left(\bbx^{(i)}_{k, t} \right) }^2.
\end{align}
%%%%%%
\end{lemma}
%%%%%%
\begin{proof}
See Section \ref{subsec:lemma-noncnvx-3-proof}.
\end{proof}

%%%%%%
\begin{lemma}\label{lemma:noncnvx-4}
Under Assumptions \ref{assump-Q} and \ref{assump-gr}, for the sequence of averages $\{\widehat{\bx}_{k+1}\}$ defined in \eqref{eq:notation} we have
%%%%%%
\begin{align}
    \bE \norm{\widehat{\bx}_{k+1} 
    -
    \bx_{k+1}}^2 
    \leq
    \frac{1}{r(n-1)} \left( 1 - \frac{r}{n} \right) 4(1+q)  
    \left\{ n \sigma^2  \tau \eta^2 
    +
    \tau \eta^2 \sum_{i \in [n]} \sum_{t=0}^{\tau-1} \norm{\gr f \left(\bbx^{(i)}_{k, t} \right) }^2
    \right\}.
\end{align}
%%%%%%
\end{lemma}
%%%%%%
\begin{proof}
See Section \ref{subsec:lemma-noncnvx-4-proof}.
\end{proof}

After establishing the main building modules in the above lemmas, we now proceed to prove the convergence rate in Theorem \ref{thm:2}. In particular, we combine the results in Lemmas \ref{lemma:noncnvx-1}--\ref{lemma:noncnvx-4} to derive the following recursive inequality on the expected function value on the models updated at the parameter servers, i.e. $\{\bx_{k}:k=1,\cdots,K\}$:
%%%%%%
\begin{align}
    &\bE f(\bx_{k+1}) \leq
    \bE f (\bx_{k}) \\
    &\quad-
    \frac{1}{2} \eta \sum_{t=0}^{\tau-1} \bE \norm{\gr f (\xbar_{k, t})}^2 \\
    &\quad 
    -
    \eta \frac{1}{2n} \left(1 - L \left(  1 + \frac{1}{n} q \tau
    +
    4 \frac{n-r}{r(n-1)} (1+q)
    \tau \right) \eta - 2 L^2 \tau(\tau-1) \eta^2 \right) \sum_{t=0}^{\tau-1} \sum_{i\in[n]} \bE \norm{\gr f \left(\bbx^{(i)}_{k, t} \right)}^2 \\
    &\quad +
    \eta^2 \frac{L}{2} (1+q) \tau
    \left(
    \frac{\sigma^2}{m}
    +
    4 \frac{\sigma^2}{r} \frac{n-r}{n-1}
    \right)
    + \eta^3 \frac{\sigma^2}{m} (m+1) \frac{\tau(\tau-1)}{2} L^2.
\end{align}
%%%%%%
For sufficiently small $\eta$, such that
%%%%%%
\begin{align} \label{eq:19}
    %1 - L \left(  1 + \frac{1}{m} q \tau\right) \eta - 2 L^2 \tau(\tau-1) \eta^2
    1 - L \eta - L \left( \frac{1}{n} q 
    +
    4 \frac{n-r}{r(n-1)} (1+q)
    \right) \tau \eta - 2 L^2 \tau(\tau-1) \eta^2 
    \geq 0,
\end{align}
%%%%%%
we have
%%%%%%
\begin{align}
    \bE f(\bx_{k+1}) 
    &\leq
    \bE f (\bx_{k}) 
    -
    \frac{1}{2} \eta \sum_{t=0}^{\tau-1} \bE \norm{\gr f (\xbar_{k, t})}^2 \\
    &\quad +
    \eta^2 \frac{L}{2} (1+q) \tau \left(
    \frac{\sigma^2}{n}
    +
    4 \frac{\sigma^2}{r} \frac{n-r}{n-1}
    \right)
    +
    \eta^3 \frac{\sigma^2}{n} (n+1) \frac{\tau(\tau-1)}{2} L^2. \label{eq:noncnvx-recursive}
\end{align}
%%%%%%
In Section \ref{sec:stepsize} we show that if the stepsize is picked as $\eta = \nicefrac{1}{L\sqrt{T}}$ and the $T$ ans $\tau$ satisfy the condition \eqref{eq:thm2-cond} in Theorem \ref{thm:2}, then \eqref{eq:19} also holds. Now summing \eqref{eq:noncnvx-recursive} over $k=0,\cdots,K-1$ and rearranging the terms yield that
%%%%%%
\begin{align}
    & \frac{1}{2} \eta \sum_{k=0}^{K-1} \sum_{t=0}^{\tau-1} \bE \norm{\gr f (\xbar_{k,t})}^2 \\
    &\quad \leq
    f (\bx_{0}) - f^* 
    +
    K \eta^2 \frac{L}{2} (1+q) \tau \left(
    \frac{\sigma^2}{n}
    +
    4 \frac{\sigma^2}{r} \frac{n-r}{n-1}
    \right)
    + 
    K \eta^3 \frac{\sigma^2}{n} (n+1) \frac{\tau(\tau-1)}{2} L^2,
\end{align}
%%%%%%
or
%%%%%%
\begin{align}
    & \frac{1}{K\tau} \sum_{k=0}^{K-1} \sum_{t=0}^{\tau-1} \bE \norm{\gr f (\xbar_{k, t})}^2 \\
    &\quad\leq
    \frac{2 (f (\bx_{0}) - f^*)}{\eta K\tau}
    +
    \eta L (1+q) \left(
    \frac{\sigma^2}{n}
    +
    4 \frac{\sigma^2}{r} \frac{n-r}{n-1}
    \right) + 
    \eta^2 \frac{\sigma^2}{n} (n+1) (\tau-1) L^2.
\end{align}
%%%%%%
Picking the stepsize $\eta=\nicefrac{1}{L\sqrt{T}}=\nicefrac{1}{L\sqrt{K\tau}}$ results in the following convergence rate:
%%%%%%
\begin{align}
    & \frac{1}{T} \sum_{k=0}^{K-1} \sum_{t=0}^{\tau-1} \bE \norm{\gr f (\xbar_{k,t})}^2 \\
    &\leq
    \frac{2L(f (\bx_{0}) - f^*)}{\sqrt{T}}
    +
    (1+q) \left(
    \frac{\sigma^2}{n}
    +
    \frac{\sigma^2}{r} \frac{n-r}{n-1}
    \right) \frac{1}{\sqrt{T}} + 
    \frac{\sigma^2}{n} (n+1) \frac{\tau-1}{T},
\end{align}
%%%%%%
which completes the proof of Theorem \ref{thm:2}.

\subsection{Discussion on stepsize $\eta$} \label{sec:stepsize}

Here, we consider the constraint on the stepsize derived in \eqref{eq:19} and show that if $\eta$ is picked according to Theorem \ref{thm:2}, then it also satisfies \eqref{eq:19}. First, let the stepsize satisfy $1 - L \eta \geq 0.1$. Now, if the following holds
%%%%%%
\begin{align} \label{eq:cond1}
    %1 - L \left(  1 + \frac{1}{m} q \tau\right) \eta - 2 L^2 \tau(\tau-1) \eta^2
    L \left( \frac{1}{n} q 
    +
    4 \frac{n-r}{r(n-1)} (1+q)
    \right) \tau \eta
    +
    2 L^2 (\tau \eta)^2 
    \leq 0.1, 
\end{align}
%%%%%%
the condition in \eqref{eq:19} also holds. It is straightforward to see when \eqref{eq:cond1} holds. To do so, consider the following quadratic inequality in terms of $y=\eta \tau$:
%%%%%%
\begin{align} 
    2 L^2 y^2 
    +
    L B_2 y
    - 0.1
    \leq 0, \label{eq:quadratic}
\end{align}
%%%%%%
where
%%%%%%
\begin{align} 
    B_2 \coloneqq
    \frac{1}{n} q 
    +
    4 \frac{n-r}{r(n-1)} (1+q).
\end{align}
%%%%%%
We can solve the quadratic form in \eqref{eq:quadratic} for $y = \eta \tau$ which yields
%%%%%%
\begin{align} \label{eq:quadratic-sol}
    \eta \tau
    \leq
    \frac{\sqrt{ B_2^2 + 0.8 } -B_2}{4L}.
\end{align}
%%%%%%
This implies that if the parameter $\tau$ and the stepsize $\eta$ satisfy \eqref{eq:quadratic-sol} and $\eta \leq \nicefrac{0.9}{L}$, then the condition \eqref{eq:19} is satisfied. In particular, for our pick of $\eta = \nicefrac{1}{L\sqrt{T}}$, the condition $\eta \leq \nicefrac{0.9}{L}$ holds if $T \geq 2$; and the constraint in \eqref{eq:quadratic-sol} is equivalent to having
%%%%%%
\begin{align} 
    \tau
    &\leq
    \frac{\sqrt{ B_2^2 + 0.8 } -B_2}{8} \sqrt{T}.
\end{align}
%%%%%%

%Therefore, in Theorem \ref{thm:2}, $\tau$ can take on values at most  $\ccalO(\sqrt{T})$.

\subsection{Proof of Lemma \ref{lemma:noncnvx-1}}\label{subsec:lemma-noncnvx-1-proof}
Recall that for any $L$-smooth function $f$ and variables $\bx,\bby$ we have
%%%%%%
\begin{align}
    f(\bx) 
    &\leq
    f(\bby)
    +
    \left\langle \gr f (\bby) , \bx - \bby \right\rangle
    +
    \frac{L}{2} \norm{\bx - \bby}^2.
\end{align}
%%%%%%
Therefore, we can write
%%%%%%
\begin{align}
    f(\bx_{k+1}) 
    &=
    f(\widehat{\bx}_{k+1} 
    +
    \bx_{k+1}
    -
    \widehat{\bx}_{k+1}) \\ 
    &\leq
    f(\widehat{\bx}_{k+1})
    +
    \left\langle \gr f (\widehat{\bx}_{k+1}) , \bx_{k+1}
    -
    \widehat{\bx}_{k+1} \right\rangle
    +
    \frac{L}{2} \norm{\bx_{k+1}
    -
    \widehat{\bx}_{k+1}}^2. \label{eq:theta}
\end{align}
%%%%%%
We take expectation of both sides of \eqref{eq:theta} and since $\widehat{\bx}_{k+1}$ is unbiased for $\bx_{k+1}$, that is $\bE_{\ccalS_k} \bx_{k+1} =  \widehat{\bx}_{k+1}$ (See \eqref{eq:S}), it yields that 
%%%%%%
\begin{align}
    \bE f(\bx_{k+1}) 
    &\leq
    \bE f(\widehat{\bx}_{k+1})
    + 
    \frac{L}{2} \bE \norm{\widehat{\bx}_{k+1} 
    -
    \bx_{k+1}}^2.\label{eq:theta-thetahat}
\end{align}
%%%%%%
Moreover, $\widehat{\bx}_{k+1}$ is also unbiased for $\xbar_{k,\tau}$, i.e. $\bE_Q \widehat{\bx}_{k+1} = \xbar_{k,\tau}$ (See \eqref{eq:Q}), and since $f$ is $L$-smooth, we can write
%%%%%%
\begin{align}
    \bE f(\widehat{\bx}_{k+1}) 
    &\leq
    \bE f(\xbar_{k,\tau})
    + 
    \frac{L}{2} \bE \norm{\widehat{\bx}_{k+1} 
    -
    \xbar_{k,\tau}}^2, 
\end{align}
%%%%%%
which together with \eqref{eq:theta-thetahat} concludes the lemma.

\subsection{Proof of Lemma \ref{lemma:noncnvx-2}}\label{subsec:lemma-noncnvx-2-proof}
According to the update rule in Algorithm \ref{alg:update}, for every $t=0,\cdots,\tau-1$ the average model is 
%%%%%%
\begin{align}
    \xbar_{k,t + 1}
    =
    \xbar_{k, t} - \eta \frac{1}{n} \sum_{i \in [n]} \tNab f_i \left(\bbx^{(i)}_{k, t} \right).
\end{align}
%%%%%%
Since $f$ is $L$-smooth, we can write
%%%%%%
\begin{align}
    f(\xbar_{k,t + 1})
    &\leq
    f(\xbar_{k, t})
    -
    \eta \left\langle \gr f (\xbar_{k, t}) , \frac{1}{n} \sum_{i \in [n]} \tNab f_i \left(\bbx^{(i)}_{k, t} \right) \right\rangle 
    +
    \eta^2 \frac{L}{2} \norm{ \frac{1}{n} \sum_{i \in [n]} \tNab f_i \left(\bbx^{(i)}_{k, t} \right)}^2. \label{eq:16}
\end{align}
%%%%%%
The inner product term above can be written in expectation as follows:
%%%%%%
\begin{align}
    2 \bE \left\langle \gr f (\xbar_{k, t}) , \frac{1}{n} \sum_{i \in [n]} \tNab f_i \left(\bbx^{(i)}_{k, t} \right) \right\rangle 
    &=
    \frac{1}{n} \sum_{i \in [n]}
    2 \bE \left\langle \gr f (\xbar_{k, t}) ,  \gr f \left(\bbx^{(i)}_{k, t} \right) \right\rangle \\
    &=
    \bE \norm{\gr f (\xbar_{k, t})}^2
    +
    \frac{1}{n} \sum_{i \in [n]}
    \bE \norm{\gr f \left(\bbx^{(i)}_{k, t} \right)}^2 \\
    & \quad -
    \frac{1}{n} \sum_{i \in [n]}
    \bE \norm{\gr f (\xbar_{k, t}) - \gr f \left(\bbx^{(i)}_{k, t} \right)}^2, \label{eq:innerp}
\end{align}
%%%%%%
where we used the identity $2 \langle \bba , \bbb \rangle = \norm{\bba}^2 + \norm{\bbb}^2 - \norm{\bba - \bbb}^2$ for any two vectors $\bba , \bbb$. In the following, we bound each of the three terms in the RHS of \eqref{eq:innerp}. Starting with the third term, we use the smoothness assumption to write
%%%%%%
\begin{align}
    \norm{\gr f (\xbar_{k,t}) - \gr f \left(\bbx^{(i)}_{k, t} \right)}^2 
    \leq
    L^2 \norm{\xbar_{k,t} - \bbx^{(i)}_{k, t}}^2.
\end{align}
%%%%%%
Moreover, local models $\bbx^{(i)}_{k, t}$ and average model $\xbar_{k,t}$  are respectively
%%%%%%
\begin{align}
    \bbx^{(i)}_{k, t}
    =
    \bx_{k} - \eta \left(  \tNab f_i(\bx_{k}) +  \tNab f_i\left(\bbx^{(i)}_{k, 1} \right) + \cdots +  \tNab f_i\left(\bbx^{(i)}_{k, t-1} \right) \right), \label{eq:xi-theta}
\end{align}
%%%%%%
and
%%%%%%
\begin{align}
    \xbar_{k,t} 
    =
    \bx_{k\tau} 
    -
    \eta \left( \frac{1}{n} \sum_{j \in [n]} \tNab f_j(\bx_{k}) 
    + 
    \frac{1}{n} \sum_{j \in [n]} \tNab f_j\left(\bbx^{(j)}_{k,1}\right) 
    + \cdots +
    \frac{1}{n} \sum_{j \in [n]} \tNab f_j\left(\bbx^{(j)}_{k, t - 1}\right) \right).
\end{align}
%%%%%%
Therefore, the expected deviation of each local model form the average model can be written as
%%%%%%
\begin{align}
    & \bE \norm{\xbar_{k,t} - \bbx^{(i)}_{k, t}}^2 \\
    &\leq
    2 \eta^2 \bE \norm{\frac{1}{n} \sum_{j \in [n]} \tNab f_j(\bx_{k}) 
    + 
    \frac{1}{n} \sum_{j \in [n]} \tNab f_j\left(\bbx^{(j)}_{k,1}\right) 
    + \cdots +
    \frac{1}{n} \sum_{j \in [n]} \tNab f_j\left(\bbx^{(j)}_{k, t - 1}\right)}^2 \\
    & \quad +
    2 \eta^2 \bE \norm{ \tNab f_i(\bx_{k}) +  \tNab f_i\left(\bbx^{(i)}_{k, 1} \right) + \cdots +  \tNab f_i\left(\bbx^{(i)}_{k, t-1} \right)}^2 \\
    &\leq
    2 \eta^2 \left( t \frac{\sigma^2}{n} + \norm{\frac{1}{n} \sum_{j \in [n]} \gr f(\bx_{k}) 
    + 
    \frac{1}{n} \sum_{j \in [n]} \gr f\left(\bbx^{(j)}_{k,1}\right) 
    + \cdots +
    \frac{1}{n} \sum_{j \in [n]} \gr f\left(\bbx^{(j)}_{k, t - 1}\right)}^2\right) \\
    & \quad +
    2 \eta^2 \left( t \sigma^2 + \norm{ \gr f(\bx_{k}) +  \gr f\left(\bbx^{(i)}_{k, 1} \right) + \cdots +  \gr f\left(\bbx^{(i)}_{k, t-1} \right)}^2\right) \\
    &\leq
    2 \eta^2 t \frac{\sigma^2}{n}
    +
    2 \eta^2 t 
    \left(\frac{1}{n} \sum_{j \in [n]} \norm{\gr f(\bx_{k})}^2 + \frac{1}{n} \sum_{j \in [n]} \norm{\gr f \left(\bbx^{(j)}_{k,1} \right)}^2 + \cdots + \frac{1}{n} \sum_{j \in [n]} \norm{\gr f \left( \bbx^{(j)}_{k,t - 1} \right)}^2\right) \\
    &\quad +
    2 \eta^2 t \sigma^2 + 2 \eta^2 t 
    \left(
    \norm{  \gr f(\bx_{k})}^2 +  \norm{ \gr f \left(\bbx^{(i)}_{k, 1}\right)}^2 + \cdots +  \norm{ \gr f \left(\bbx^{(i)}_{k, t - 1} \right)}^2
    \right). \label{eq:xi-xbar}
\end{align}
%%%%%%
Summing \eqref{eq:xi-xbar} over all the workers $i \in [n]$ yields
%%%%%%
\begin{align}
    &\sum_{i \in [n]} \bE \norm{\xbar_{k,t} - \bbx^{(i)}_{k, t}}^2 \\
    &\leq
    2 \eta^2 t \sigma^2 + 
    2 \eta^2 t 
    \left( \sum_{j \in [n]} \norm{\gr f(\bx_{k})}^2 +  \sum_{j \in [n]} \norm{\gr f\left(\bbx^{(j)}_{k,1}\right)}^2 + \cdots + \sum_{j \in [n]} \norm{\gr f\left(\bbx^{(j)}_{k,t - 1}\right)}^2\right) \\
    &\quad +
    2 \eta^2 t \sigma^2 n
    +
    2 \eta^2 t 
    \left(\sum_{i \in [n]} \norm{  \gr f(\bx_{k})}^2 +  \sum_{i \in [n]} \norm{ \gr f\left(\bbx^{(i)}_{k, 1}\right)}^2 + \cdots +  \sum_{i \in [n]} \norm{ \gr f\left(\bbx^{(i)}_{k, t - 1}\right)}^2\right) \\
    &=
    2 \eta^2 t \sigma^2 (n+1)
    +
    4 \eta^2 t 
    \left(\sum_{j \in [n]} \norm{\gr f(\bx_{k})}^2 + \sum_{j \in [n]} \norm{\gr f\left(\bbx^{(j)}_{k,1}\right)}^2 + \cdots +  \sum_{j \in [n]} \norm{\gr f\left(\bbx^{(j)}_{k,t - 1}\right)}^2\right). \label{eq:sum-xi-xbar}
\end{align}
%%%%%%
Finally, summing \eqref{eq:sum-xi-xbar} over $t=0,\cdots,\tau-1$ results in the following:
%%%%%%
\begin{align}
    & \sum_{t=0}^{\tau-1} \sum_{i \in [n]} \bE \norm{\xbar_{k,t} - \bbx^{(i)}_{k, t}}^2 \\
    &\leq
    2 \eta^2 \sigma^2 (n+1) \sum_{t=0}^{\tau-1} t
    +
    4 \eta^2 \sum_{t=0}^{\tau-1}  t 
    \left(\sum_{j \in [n]} \norm{\gr f(\bx_{k})}^2 + \sum_{j \in [n]} \norm{\gr f\left(\bbx^{(j)}_{k,1}\right)}^2 + \cdots +  \sum_{j \in [n]} \norm{\gr f\left(\bbx^{(j)}_{k,t - 1}\right)}^2\right) \\
    &\leq
    \eta^2 \sigma^2 (n+1) \tau(\tau-1)
    +
    2 \eta^2 \tau(\tau-1) \sum_{t=0}^{\tau-2} \sum_{i \in [n]} \norm{\gr f\left(\bbx^{(i)}_{k,t}\right)}^2. \label{eq:sum-sum-xi-xbar}
\end{align}
%%%%%%
Next, we bound the third term in \eqref{eq:16}. Using Assumption \ref{assump-gr} we have
%%%%%%
\begin{align}
    \bE \norm{ \frac{1}{n} \sum_{i \in [n]} \tNab f_i \left(\bbx^{(i)}_{k, t} \right)}^2 
    &=
    \bE \norm{ \frac{1}{n} \sum_{i \in [n]} \gr f \left(\bbx^{(i)}_{k, t} \right)}^2
    +
    \bE \norm{ \frac{1}{n} \sum_{i \in [n]} \tNab f_i \left(\bbx^{(i)}_{k, t} \right) - \gr f \left(\bbx^{(i)}_{k, t} \right)}^2 \\
    &\leq
    \frac{1}{n} \sum_{i \in [n]} \bE \norm{\gr f\left(\bbx^{(i)}_{k, t} \right)}^2
    +
    \frac{\sigma^2}{n}. \label{eq:st-gr}
\end{align}
%%%%%%
Summing \eqref{eq:st-gr} over iterations $t=0,\cdots,\tau-1$ yields
%%%%%%
\begin{align}
    \sum_{t=0}^{\tau-1} \eta^2 \frac{L}{2} \bE \norm{\frac{1}{n} \sum_{i \in [n]} \tNab f_i \left(\bbx^{(i)}_{k, t} \right)}^2
    &\leq
    \eta^2 \frac{L}{2} \frac{1}{n}\sum_{t=0}^{\tau-1}  \sum_{i \in [n]} \bE \norm{\gr f \left(\bbx^{(i)}_{k, t} \right)}^2
    +
    \eta^2 \frac{L}{2} \frac{\sigma^2}{n} \tau. \label{eq:sum-st-gr}
\end{align}
%%%%%%
Now we can sum \eqref{eq:16} for $t=0,\cdots,\tau-1$ and use the results in \eqref{eq:sum-sum-xi-xbar} and \eqref{eq:sum-st-gr} to conclude:
%%%%%%
\begin{align}
    \bE f(\xbar_{k,\tau})
    &\leq
    \bE f (\bx_{k}) 
    -
    \frac{1}{2} \eta \sum_{t=0}^{\tau-1} \bE \norm{\gr f (\xbar_{k, t})}^2 \\
    & \quad -
    \frac{1}{2n} \eta \sum_{t=0}^{\tau-1} \sum_{i\in [n]} \bE \norm{\gr f\left(\bbx^{(i)}_{k, t}\right)}^2 \\
    & \quad +
    \frac{1}{2n} \eta \sum_{t=0}^{\tau-1} \sum_{i\in [n]} \bE \norm{\gr f (\xbar_{k,t}) - \gr f \left(\bbx^{(i)}_{k, t}\right)}^2 \\
    & \quad +
    \sum_{t=0}^{\tau-1} \eta^2 \frac{L}{2} \bE \norm{\frac{1}{n} \sum_{i\in [n]} \tNab f_i \left(\bbx^{(i)}_{k, t}\right)}^2 \\
    &\leq
    \bE f (\bx_{k}) 
    -
    \frac{1}{2} \eta \sum_{t=0}^{\tau-1} \bE \norm{\gr f (\xbar_{k, t})}^2 \\
    &\quad -
    \eta \left( \frac{1}{2n} - \frac{1}{2n} L \eta - \frac{1}{n} L^2 \tau(\tau-1) \eta^2 \right) \sum_{t=0}^{\tau-1} \sum_{i \in [n]} \bE \norm{\gr f\left(\bbx^{(i)}_{k,t}\right)}^2 \\
    &\quad +
    \eta^2 \frac{L}{2} \frac{\sigma^2}{n} \tau + \eta^3 \frac{\sigma^2}{n} (n+1) \frac{\tau(\tau-1)}{2} L^2. \label{eq:17}
\end{align}
%%%%%%

\subsection{Proof of Lemma \ref{lemma:noncnvx-3}}\label{subsec:lemma-noncnvx-3-proof}
According to definitions in \eqref{eq:notation} and using Assumption \ref{assump-Q} we have
%%%%%%
\begin{align}
    \bE \norm{\widehat{\bx}_{k+1} - \xbar_{k, \tau}}^2
    &\leq
    \frac{1}{n^2} \sum_{i \in [n]} q \, \bE \norm{\bx^{(i)}_{k, \tau} - \bx_{k} }^2. \label{eq:lemma7}
\end{align}
%%%%%%
Using the model update in \eqref{eq:xi-theta} and Assumption \ref{assump-gr}, we can write
%%%%%%
\begin{align}
    \bE \norm{\bx^{(i)}_{k,\tau} - \bx_{k} }^2 
    &=
    \eta^2 \bE \norm{\tNab f_i(\bx_{k}) + \tNab f_i\left(\bbx^{(i)}_{k, 1}\right) + \cdots + \tNab f_i\left(\bbx^{(i)}_{k,\tau - 1}\right)}^2 \\
    &=
    \eta^2 \bE \norm{\tNab f_i(\bx_{k}) - \gr f(\bx_{k}) + \cdots + \tNab f_i\left(\bbx^{(i)}_{k,\tau - 1}\right) - \gr f\left(\bbx^{(i)}_{k,\tau - 1}\right) }^2 \\
    &\quad +
    \eta^2  \norm{\gr f(\bx_{k\tau}) + \cdots + \gr f\left(\bbx^{(i)}_{k,\tau - 1}\right) }^2 \\
    &\leq
    \eta^2 \sigma^2 \tau
    +
    \eta^2 \tau \sum_{t=0}^{\tau-1} \norm{\gr f\left(\bbx^{(i)}_{k,t}\right) }^2. \label{eq:xi-xk}
\end{align}
%%%%%%
Summing \eqref{eq:xi-xk} over all workers $i \in [n]$ and using \eqref{eq:lemma7} yields
%%%%%%
\begin{align}
    \bE \norm{\widehat{\bx}_{k+1} - \xbar_{k, \tau}}^2
    &\leq
    q \frac{\sigma^2}{n} \tau \eta^2 
    +
    q \frac{1}{n^2} \tau \eta^2 \sum_{i \in [n]} \sum_{t=0}^{\tau-1} \norm{\gr f \left(\bbx^{(i)}_{k, t}\right) }^2, \label{eq:18}
\end{align}
%%%%%%
as desired in Lemma \ref{lemma:noncnvx-3}.

\subsection{Proof of Lemma \ref{lemma:noncnvx-4}}\label{subsec:lemma-noncnvx-4-proof}

The steps to prove the bound in \eqref{eq:x-xhat2} for strongly convex losses in Lemma \ref{lemma:4} can also be applied for non-convex losses. That is, we can use \eqref{eq:x-xhat2} and together with \eqref{eq:xi-xk} conclude the following:
%%%%%%
\begin{align}
    \bE \norm{\widehat{\bx}_{k+1} 
    -
    \bx_{k+1}}^2 
    &\leq
    \frac{1}{r(n-1)} \left( 1 - \frac{r}{n} \right) 4(1+q) \sum_{i \in [n]} \norm{\bx^{(i)}_{k,\tau} - \bx_{k}}^2 \\
    &\leq
    \frac{1}{r(n-1)} \left( 1 - \frac{r}{n} \right) 4(1+q)  
    \left\{ n \sigma^2  \tau \eta^2 
    +
    \tau \eta^2 \sum_{i \in [n]} \sum_{t=0}^{\tau-1} \norm{\gr f \left(\bbx^{(i)}_{k, t}\right) }^2
    \right\}.
\end{align}
%%%%%%

%% file: 8-additional-numerical.tex
To further illustrate the practical performance of the proposed \texttt{FedPAQ} method, in this section we provide more numerical results using different and more complicated datasets and model parameters. The network settings, communication and computation time models remain the same as those in Section \ref{sec:numerical}. The following figures demonstrate the training time corresponding to the following scenarios:
\begin{itemize}
    \item Figure \ref{fig:cifar10-248Kpar}: Training time of a neural network with four hidden layers and more than $248$K parameters over $10$K samples of the CIFAR-10 dataset with $10$ labels.
    \item Figure \ref{fig:cifar100}: Training time of a neural network with one hidden layer over $10$K samples of the CIFAR-100 dataset with $100$ labels.
    \item Figure \ref{fig:FMNIST}: Training time of a neural network with one hidden layer over $10$K samples of the Fashion-MNIST dataset with $10$ labels.
\end{itemize}

Similar to Section \ref{sec:numerical-cifar10}, in all of the above scenarios, the data samples are uniformly distributed among $n=50$ nodes. We also keep the communication-computation ratio and the batchsize to be $C_{\texttt{comm}}/C_{\texttt{comp}} = 1000/1$ and $B=10$ respectively, and finely tune the stepsize for every training.

\begin{figure*}[h!]
\centering
    \includegraphics[width=0.245\linewidth]{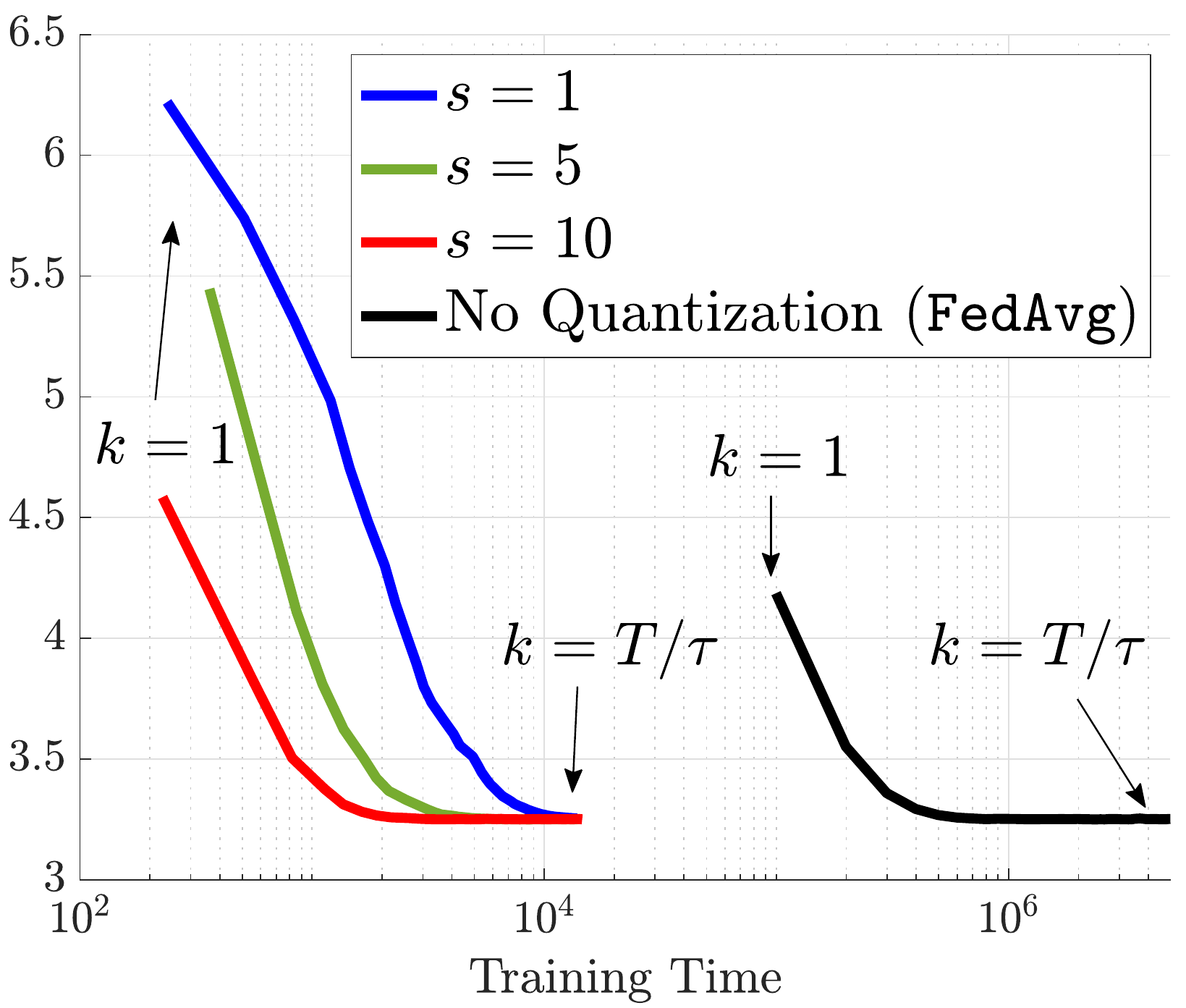}
    \includegraphics[width=0.245\linewidth]{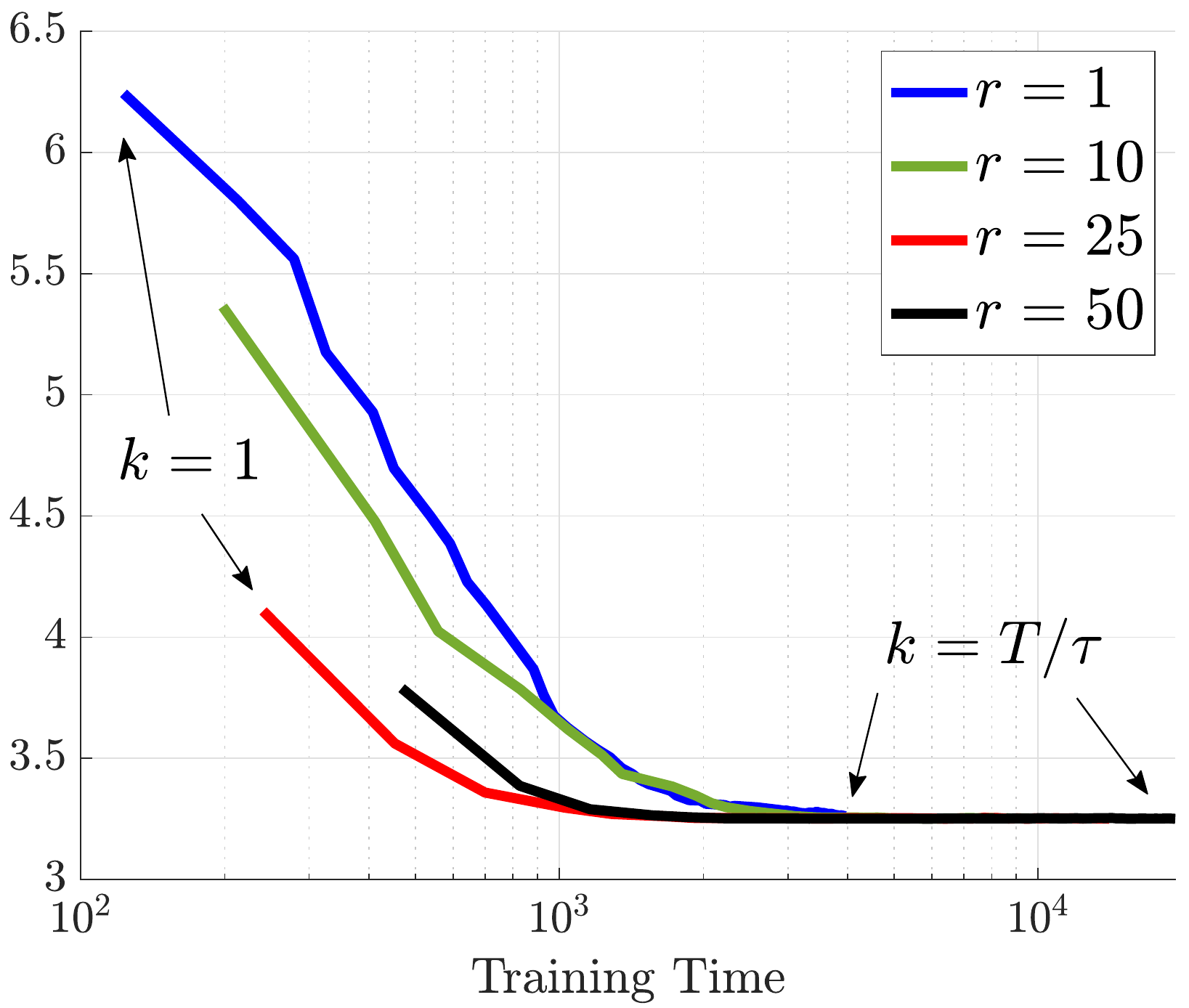}
    \includegraphics[width=0.245\linewidth]{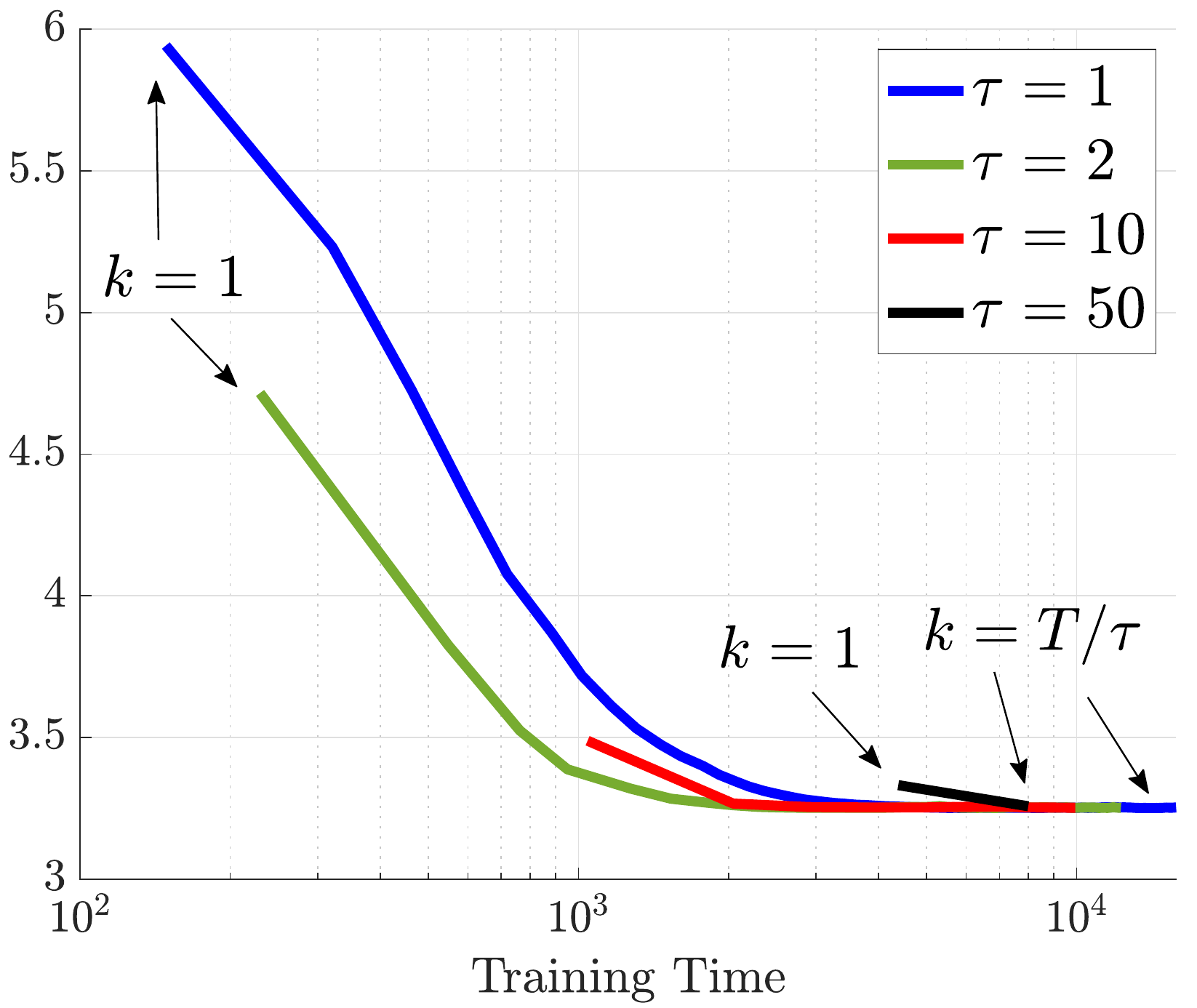}
    \includegraphics[width=0.245\linewidth]{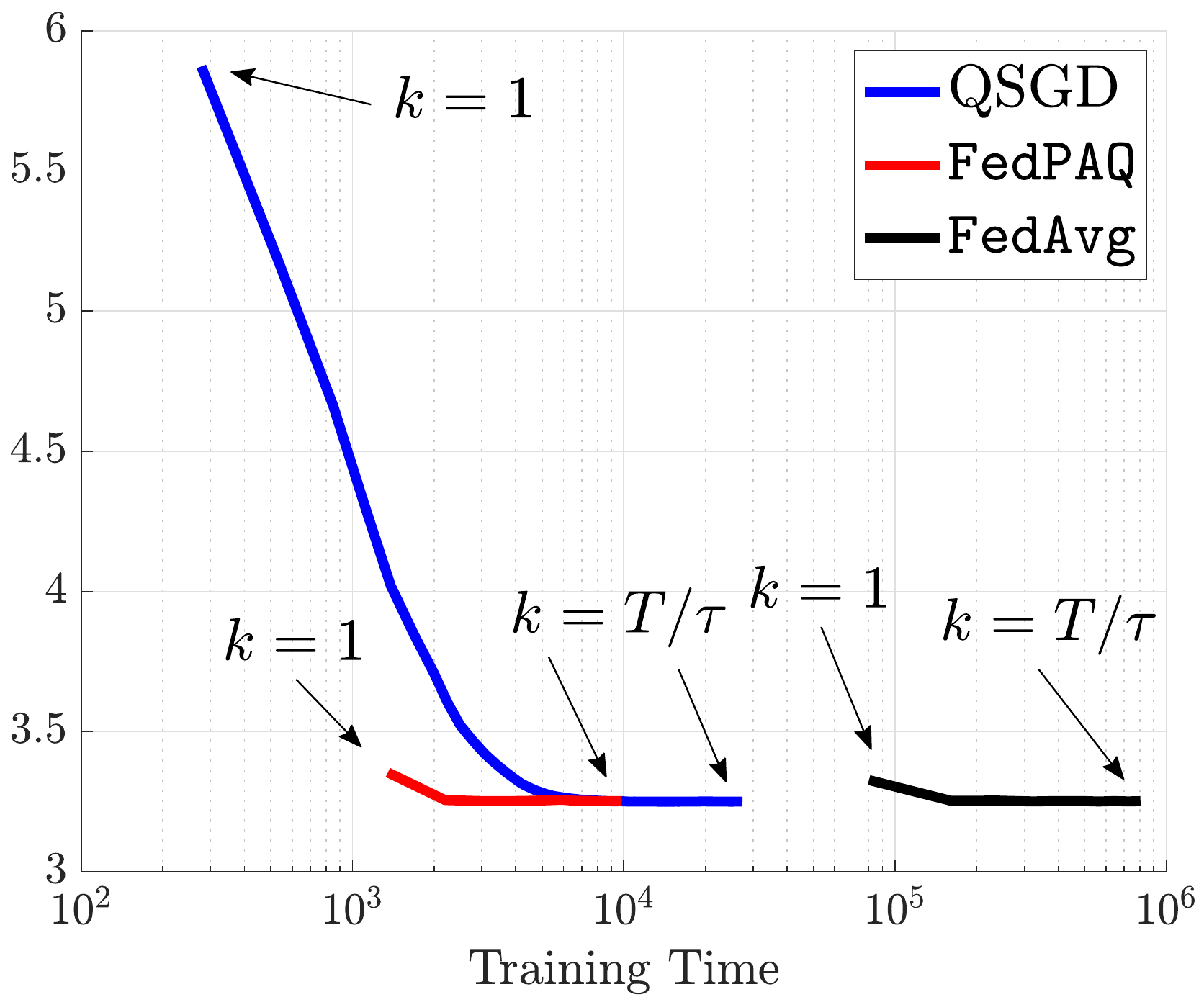}
    %\vspace{-5mm}
    \caption{Training Loss vs. Training Time: Neural Network on CIFAR-10 dataset with $248$K parameters.}
    \label{fig:cifar10-248Kpar}
\end{figure*}

\begin{figure*}[h!]
\centering
    \includegraphics[width=0.245\linewidth]{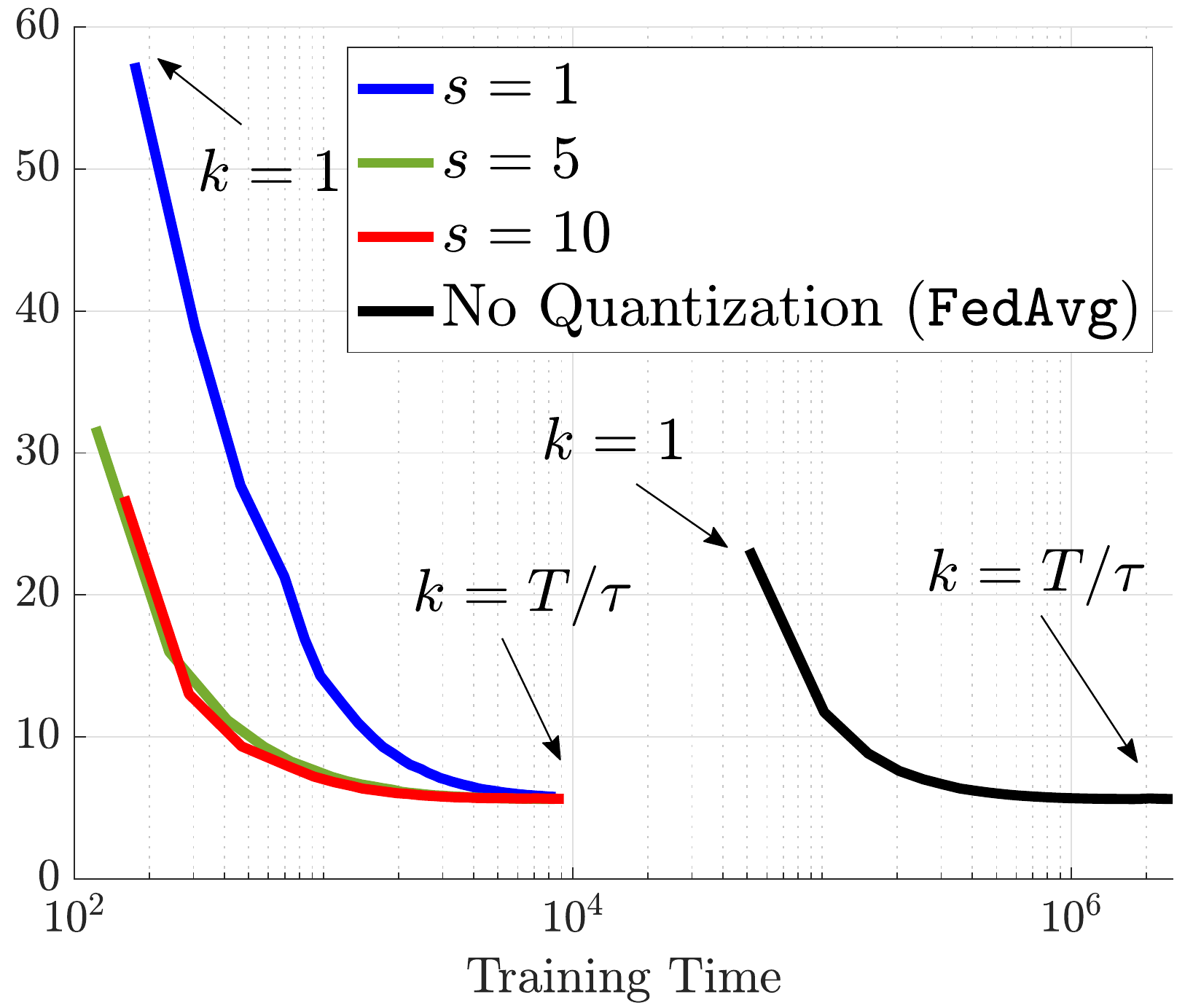}
    \includegraphics[width=0.245\linewidth]{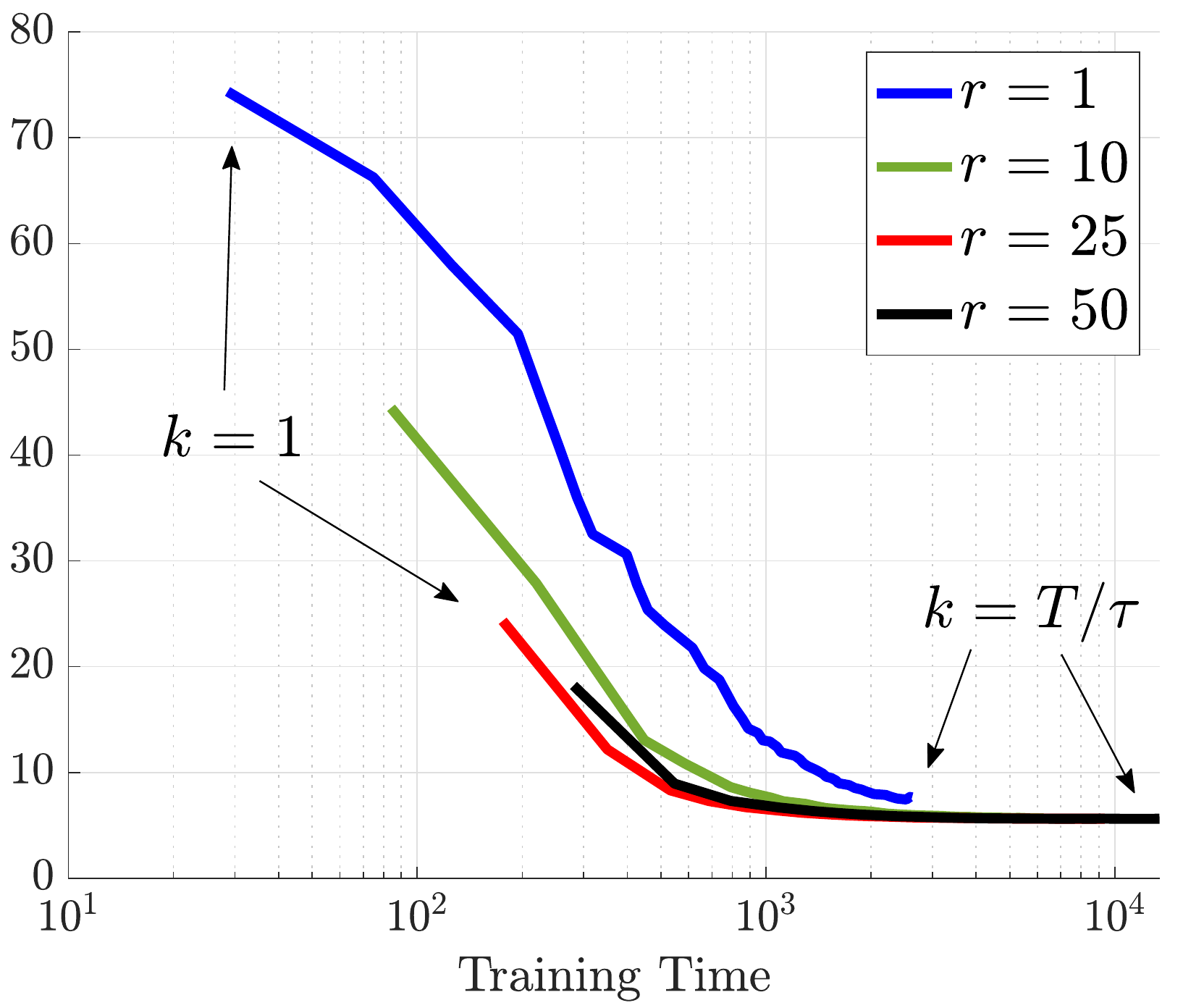}
    \includegraphics[width=0.245\linewidth]{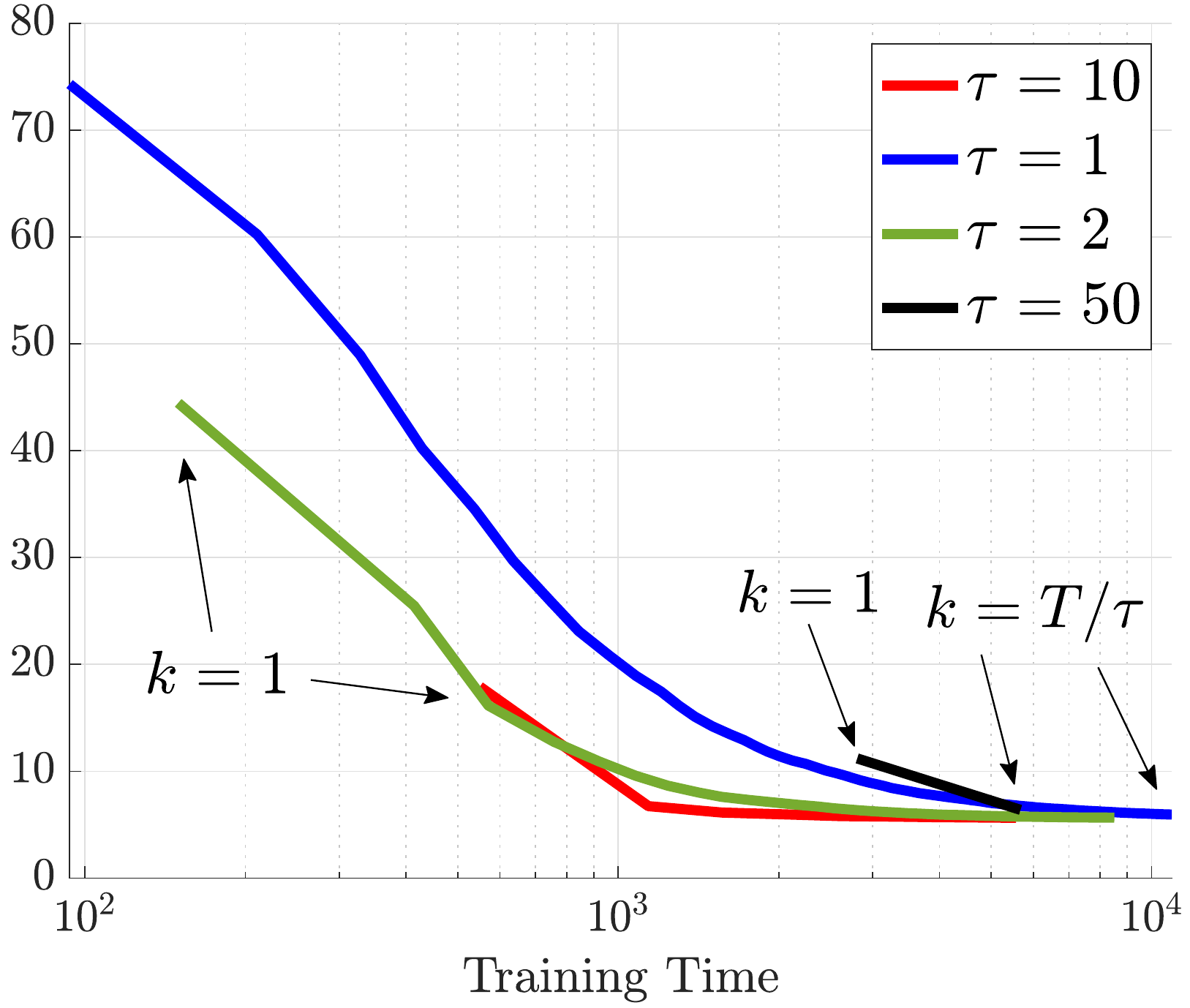}
    \includegraphics[width=0.245\linewidth]{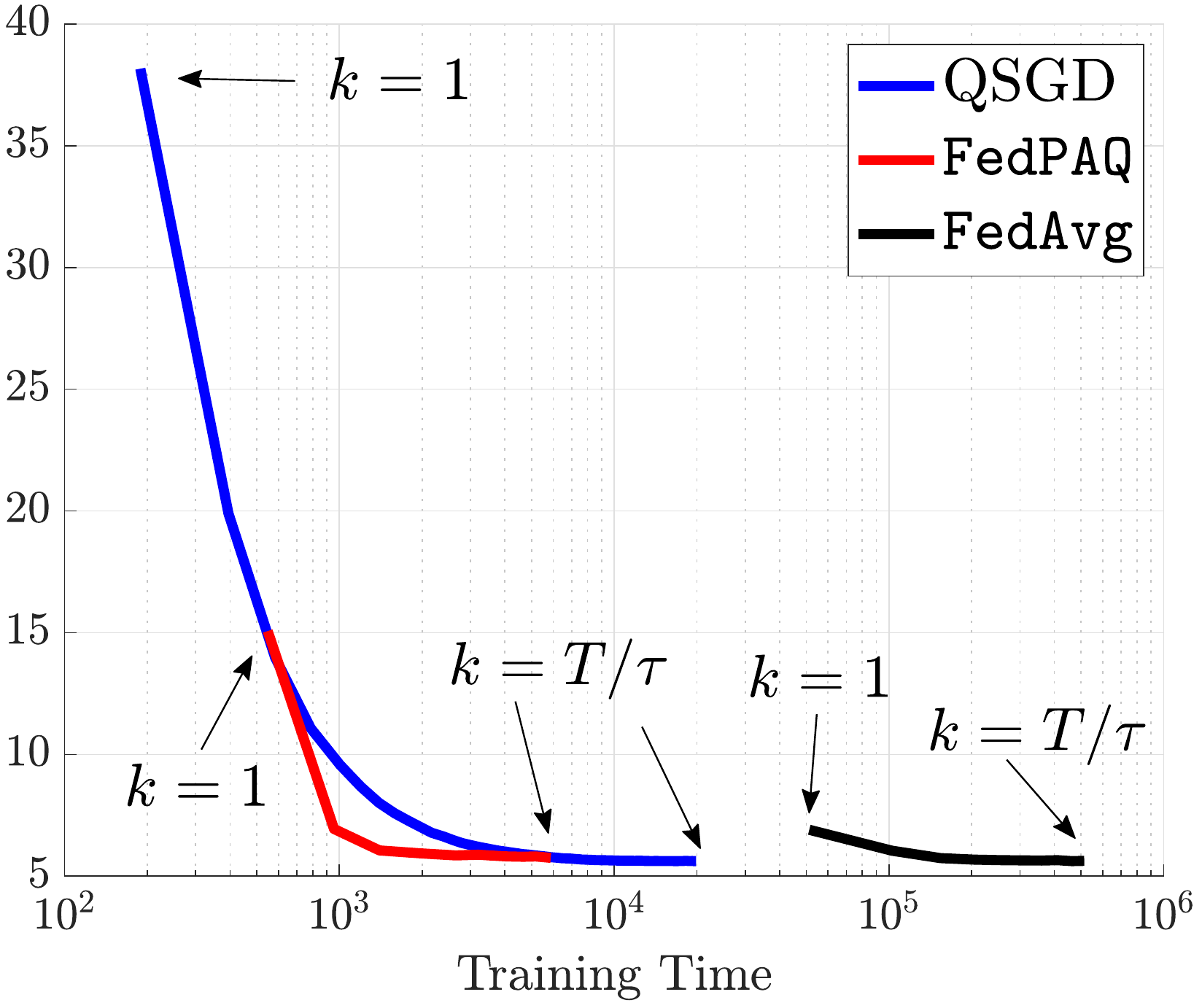}
    %\vspace{-5mm}
    \caption{Training Loss vs. Training Time: Neural Network on CIFAR-100 dataset.}
    \label{fig:cifar100}
\end{figure*}

\begin{figure*}[h!]
\centering
    \includegraphics[width=0.245\linewidth]{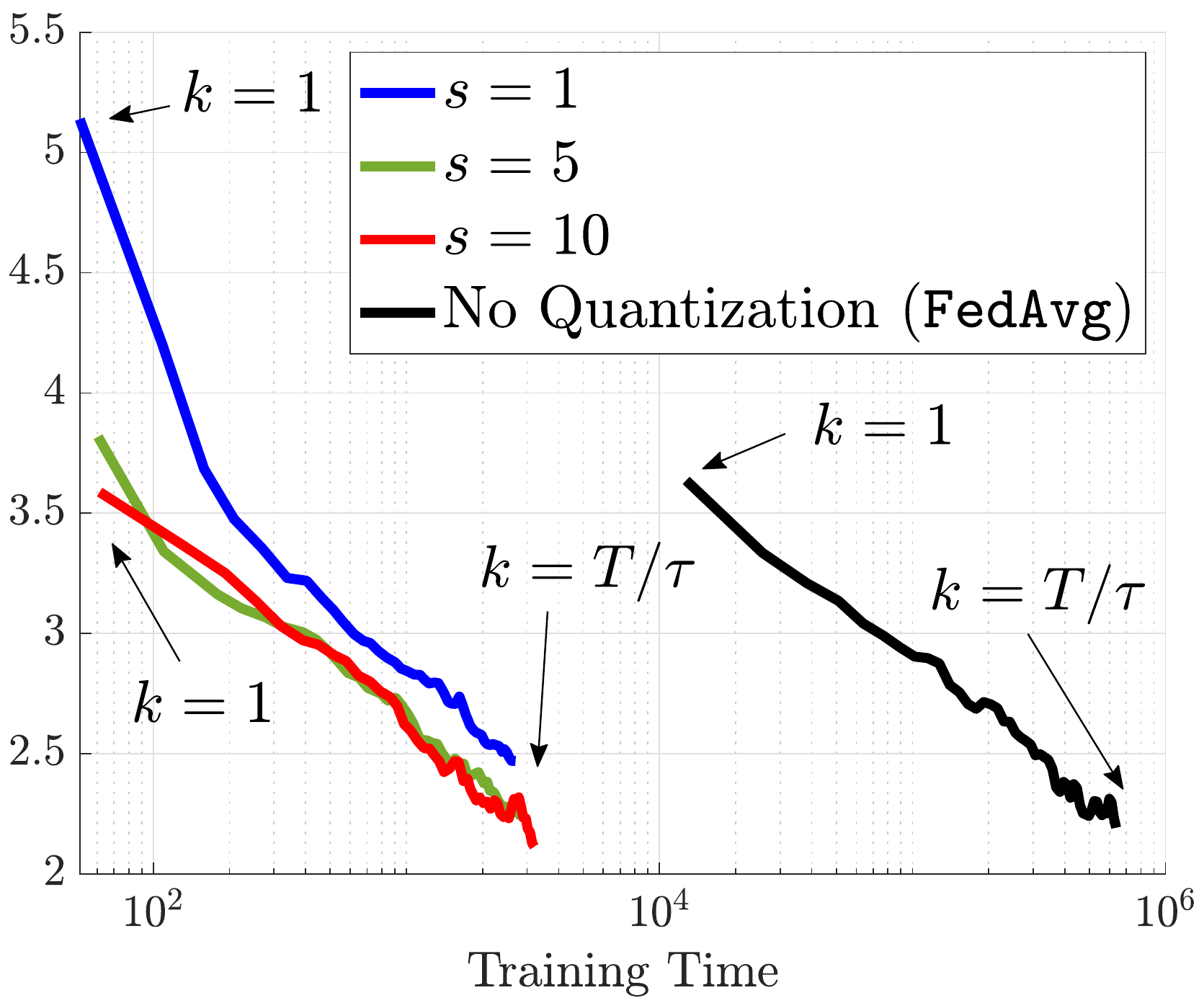}
    \includegraphics[width=0.245\linewidth]{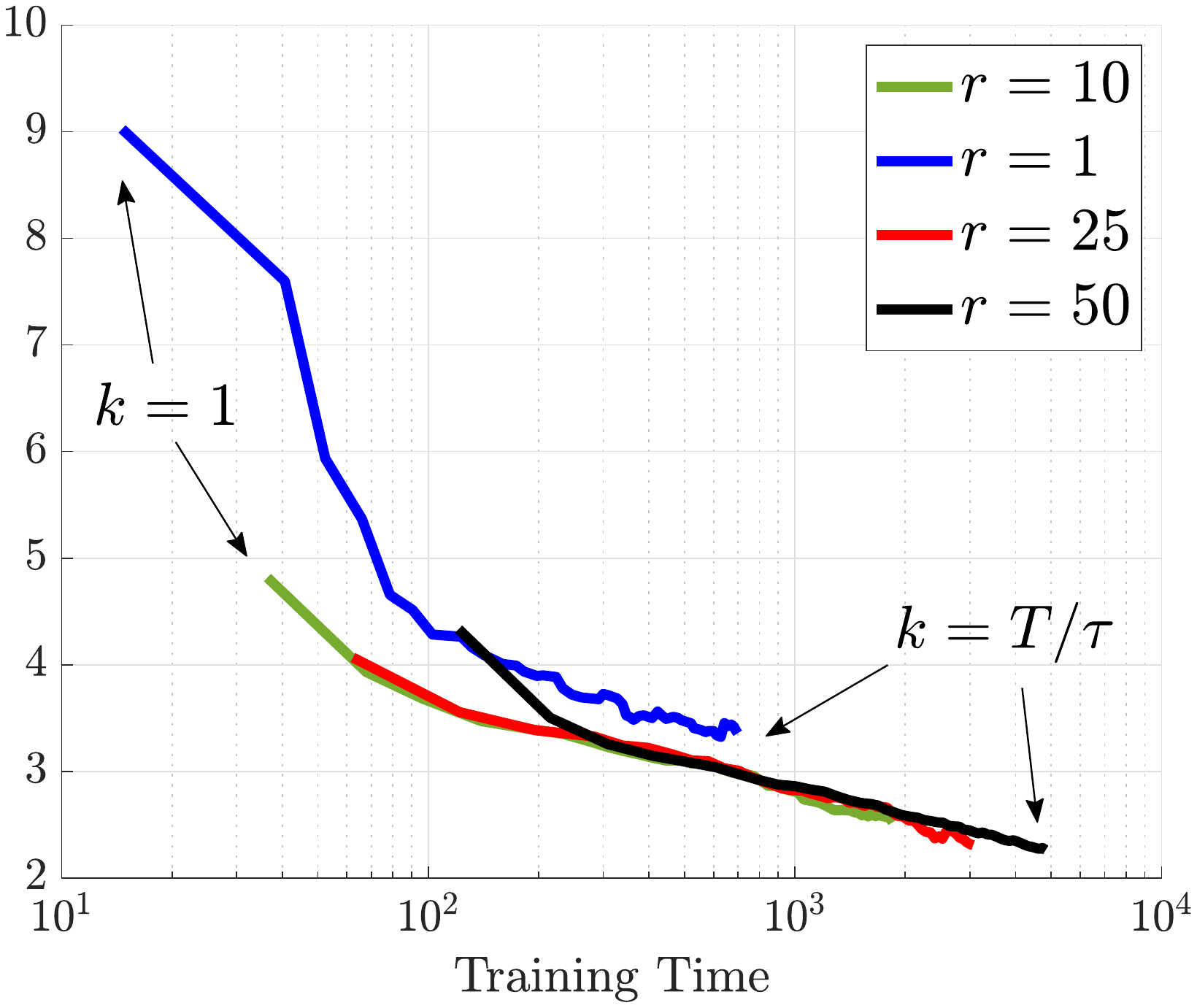}
    \includegraphics[width=0.245\linewidth]{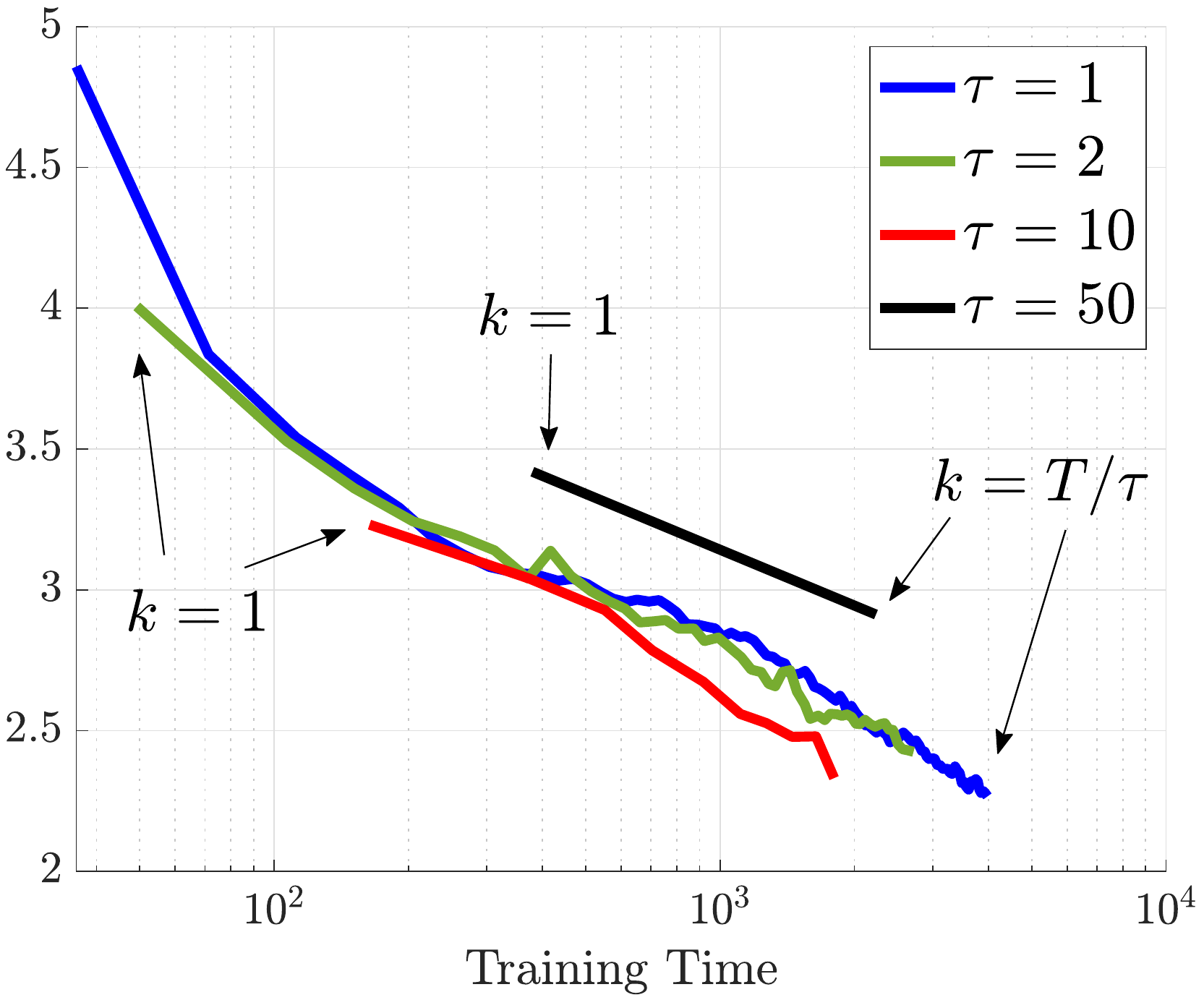}
    \includegraphics[width=0.245\linewidth]{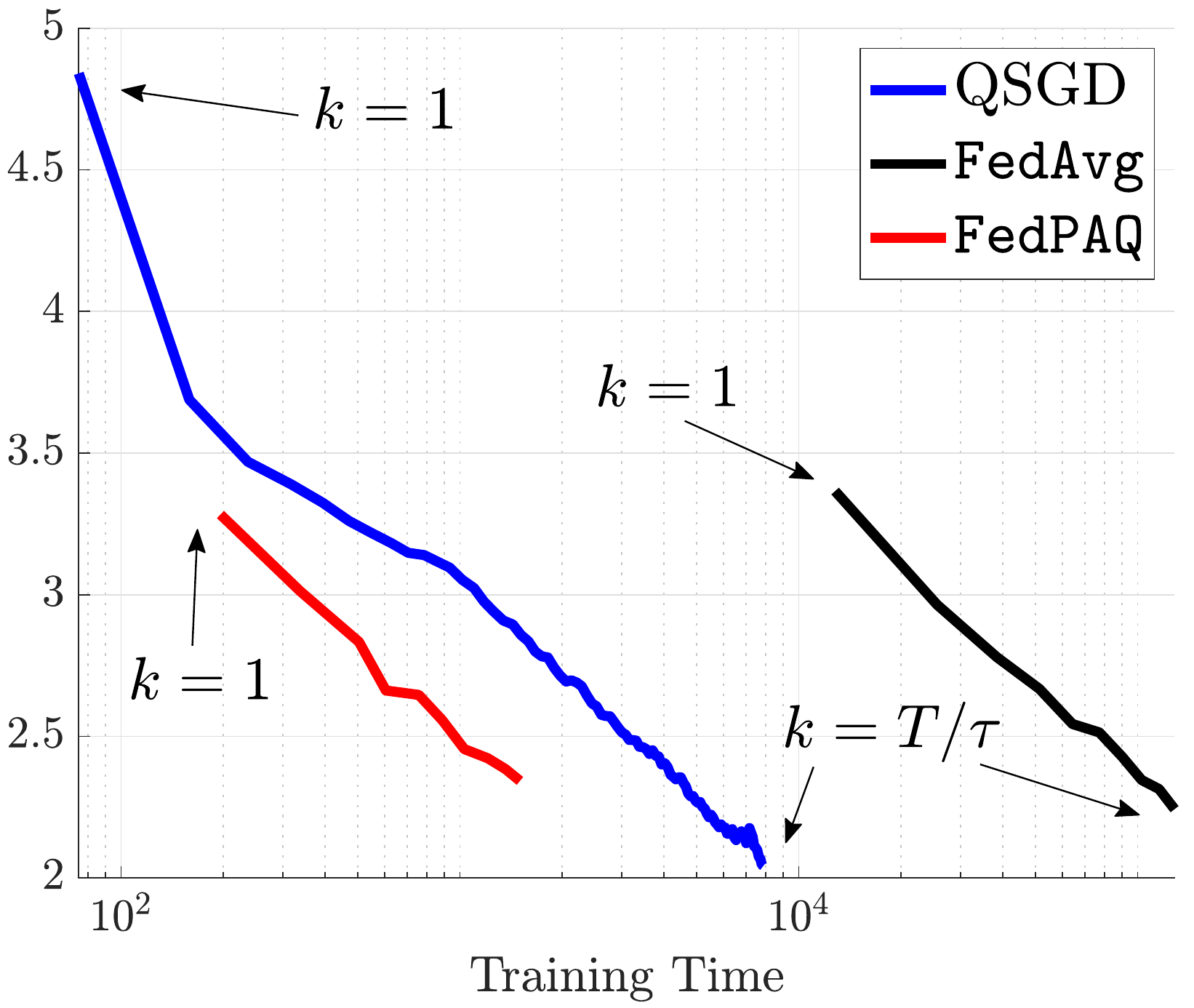}
    %\vspace{-5mm}
    \caption{Training Loss vs. Training Time: Neural Network on Fashion-MNIST dataset.}
    \label{fig:FMNIST}
\end{figure*}